\documentclass[11pt]{article}
\usepackage[utf8]{inputenc}
\usepackage[left=2.5cm,right=2.5cm,vmargin=2.5cm]{geometry}
\usepackage{hyperref}

\title{\textbf{Joint Community Detection and Rotational Synchronization via Semidefinite Programming}}
\author{Yifeng Fan\thanks{Department of Electrical and Computer Engineering,  University of Illinois at Urbana-Champaign, Champaign, IL.}
\and Yuehaw Khoo\thanks{Department of Statistics, University of Chicago, Chicago, IL.}
\and Zhizhen Zhao\footnotemark[1]}

\usepackage{graphicx}
\usepackage{amssymb}
\usepackage{amsmath}
\usepackage{amsfonts}
\usepackage{subfig}
\usepackage{multirow}
\usepackage{amsthm}
\usepackage{booktabs}
\usepackage{array}
\usepackage{bm}
\usepackage{subfig}
\usepackage{float}

\newtheorem{theorem}{Theorem}[section]
\newtheorem{lemma}[theorem]{Lemma}
\newtheorem{remark}{Remark}

\newcommand{\T}{\operatorname{Tr}}
\newcommand{\SO}{\mathrm{SO}}

\DeclareMathOperator*{\argmax}{arg\,max}

\DeclareMathOperator{\diag}{diag}

\begin{document}

\maketitle

\begin{abstract}
In the presence of heterogeneous data, where randomly rotated objects fall into multiple underlying categories, it is challenging to simultaneously classify them into clusters and synchronize them based on pairwise relations. This gives rise to the joint problem of community detection and synchronization. We propose a series of semidefinite relaxations, and prove their exact recovery when extending the celebrated stochastic block model to this new setting where both rotations and cluster identities are to be determined. Numerical experiments demonstrate the efficacy of our proposed algorithms and confirm our theoretical result which indicates a sharp phase transition for exact recovery.
\end{abstract}


\section{Introduction}
\label{sec:intro}

In the presence of heterogeneous data, where randomly rotated objects fall into multiple underlying categories, it is often challenging to simultaneously synchronize the objects and classify them into clusters. This gives rise to the joint problem of clustering and synchronization considered in this work, which is an emerging area that connects community detection~(clustering)~\cite{abbe2017community, fortunato2010community, newman2003structure,abbe2015exact, hajek2016achievinga, hajek2016achievingb} and synchronization~\cite{singer2011angular,boumal2016nonconvex, gao2019multi} as two fundamental problems in data science. Recently, several works discussed simultaneous classification and mapping (alignment)~\cite{bajaj2018smac,lederman2019representation} and proposed different models and algorithms.  In~\cite{bajaj2018smac}, the authors addressed simultaneous permutation group synchronization and clustering via a spectral method with rounding and used the consistency of the mapping for clustering. In~\cite{lederman2019representation}, the authors noticed that both rotational alignment and classification are problems over compact groups and proposed a harmonic analysis and semidefinite programming~(SDP) based approach for solving alignment and classification simultaneously.

A motivating example of the joint problem is the 2D class averaging step for cryo-electron microscopy~(cryo-EM) image denoising~\cite{frank2006,singer2011viewing,zhao2014rotationally}~(see Figure~\ref{fig:cryo_intro}). Given a collection of single particle images with unknown viewing directions~(views), it first clusters them into different groups such that images within each cluster have similar views and are almost identical up to in-plane rotational alignments, then it denoises by rotationally aligning and averaging images for each cluster. Usually, the raw images obtained in practice are extremely noisy, which poses a great challenge on the robustness of any algorithms for 2D class averaging. 


\begin{figure}[t!]
    \centering
    \vspace{-0.1cm}
    \includegraphics[width = 0.7\textwidth]{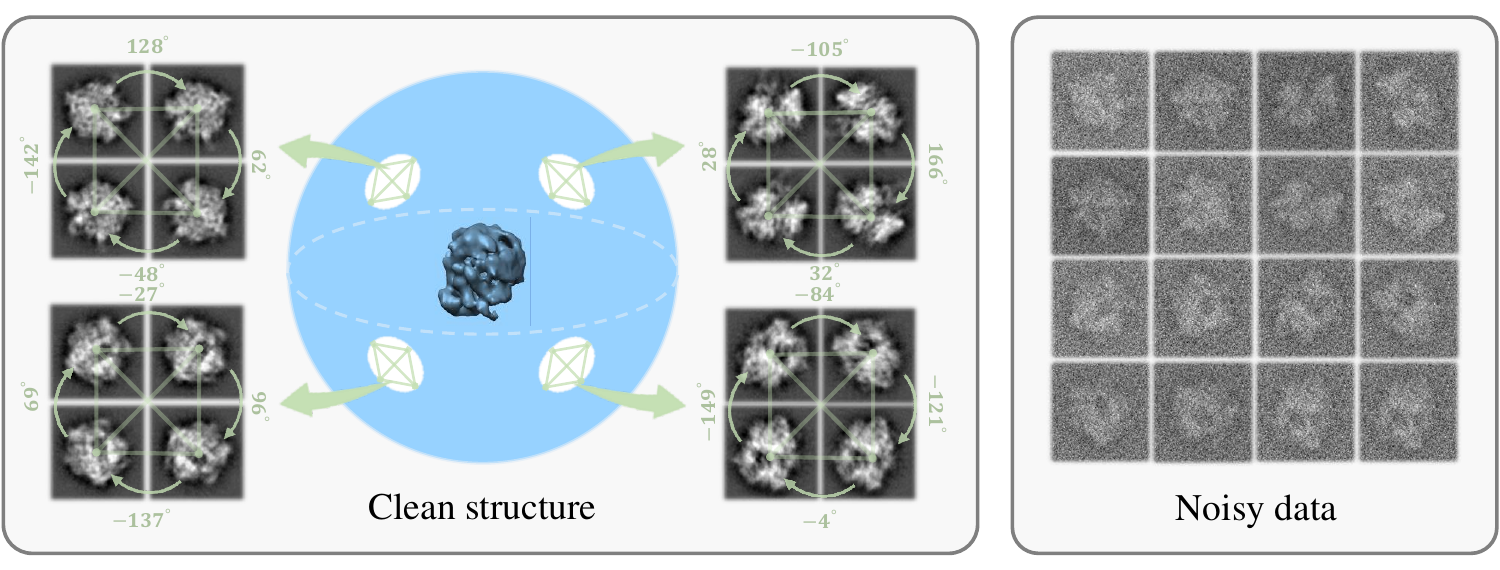}
    \vspace{-0.15cm}
    \caption{\textit{2D class averaging step in Cryo-EM image denoising}. \textit{Left}: Images of a single particle are taken from different viewing directions. Images with similar views are almost identical up to some in-plane rotational alignments which are listed. \textit{Right}: A collection of raw images in practice, which are extremely noisy.}
    \label{fig:cryo_intro}
\end{figure}

In this work, we study the \textit{joint community detection and rotational synchronization} problem under a specific probabilistic model, which extends the celebrated stochastic block model (SBM)~\cite{decelle2011asymptotic, doreian2005generalized, dyer1989solution, fienberg1985statistical, holland1983stochastic, karrer2011stochastic, massoulie2014community, mcsherry2001spectral, mossel2012stochastic, mossel2018proof} to incorporate both the community structure and pairwise rotations. In particular, we inherit the $G(n, p,q)$-SBM setting~\cite{li2018convex, hajek2016achievinga, abbe2015exact, hajek2016achievingb} as shown in Figure~\ref{fig:intro_model}. That is, given a network of size $n$ with $K$ underlying disjoint communities, a random graph $G = (V,E)$ with node set $V$ and edge set $E$ is generated such that each pair of vertices are independently connected with probability $p$ (resp.~$q$) if they belong to the same cluster (resp.~different clusters). In addition, each node $i$ is associated with an unknown rotation $\bm{R}_i \in \SO(d)$. For each edge $(i,j) \in E$, the pairwise alignment $\bm{R}_{ij}$ is observed in a way that, when $i$ and $j$ belong to the same cluster we have $\bm{R}_{ij} = \bm{R}_i\bm{R}_j^\top$ that is the clean measurement, otherwise $\bm{R}_{ij}$ is uniformly drawn from $\SO(d)$ that carries no information but noise. Notably, such noise model on $\bm{R}_{ij}$ is similar to the one considered in~\cite{singer2011viewing,fan2019representation,fan2019multi}, and it captures the key observations in cryo-EM images such that the rotational alignment for images with similar views can be estimated accurately, while for images with distant views, they cannot be well-aligned and thus the estimation is random and purely noisy.

 
Given the model above, our goal is to recover the community structures and the rotations $\{\bm{R}_i\}_{i = 1}^n$. A naive two-stage approach is (1) applying existing methods for community detection, then (2) performing synchronization for each identified community separately. 
However, such an approach is sub-optimal since it does not take any advantage of the \emph{consistency} of alignments within each cluster and the \emph{inconsistency} of alignments across clusters. In other words, the identification of communities can benefit from exploiting such (in)consistency. For instance, given three nodes $i,j$ and $k$ within the same cluster, their rotational alignments should satisfy the \emph{cycle-consistency} as $\bm{R}_{ij}\bm{R}_{jk}\bm{R}_{ki} = \bm{I}_d$. In fact, the notion of cycle consistency has been used in synchronization problems~\cite{singer2011angular, nguyen2011optimization, singer2012vector, fan2019multi, fan2019unsupervised, shi2020message}.


\subsection{Our contributions}
To incorporate the consistency of alignments, we formulate a series of optimization problems in different settings which aim to simultaneously recover the community structures and individual rotations. To efficiently solve these non-convex programs, we apply semidefinite relaxation and each resulting \emph{semidefinite programming}~(SDP) is tailored to accommodate different cluster structures with known or unknown cluster sizes. For the case of two clusters, we provide dual certificates and derive sharp exact recovery conditions of the corresponding SDPs. In particular, we establish two novel concentration inequalities that our analysis relies on: (1) Given a series of i.i.d. random matrices $\{\bm{R}_i\}_{i=1}^n$ uniformly drawn from $\SO(d)$, we derive a sharp tail bound for the Frobenius norm $\|\sum_{i = 1}^n\bm{R}_i\|_\mathrm{F}$ at $d = 2$.
(2) Given an $m \times n$ block matrix $\bm{S} \in \mathbb{R}^{md \times nd}$ with i.i.d. blocks uniformly drawn from $\SO(d)$, we show the operator norm $\|\bm{S}\| \sim \sqrt{n} + \sqrt{m}$, which can be regarded as an extension of the operator norm bound of random matrices with i.i.d. entries~\cite{bandeira2016sharp, vershynin2018high,tao2012topics}. Our result improves the sharpness of the non-commutative Khintchine inequality by removing the multiplicative  $\sqrt{\log (n d)}$ factor. Similar phenomenon is also observed in~\cite{bandeira2021spectral} for bounding the operator norm of a different class of random block matrices, called random lifts of matrices.

\begin{figure}[t!]
    \centering
    \vspace{-0.2cm}
    \includegraphics[width = 0.63\textwidth]{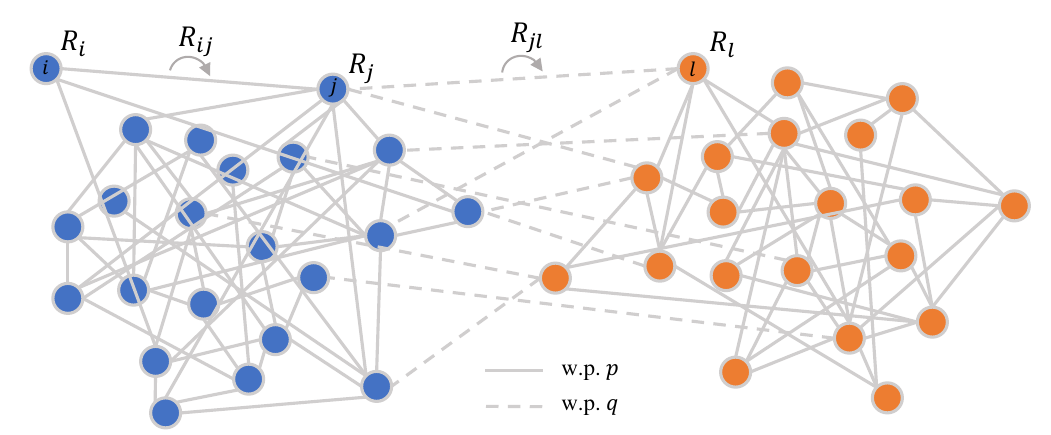}
    \vspace{-0.15cm}
    \caption{We show a network with two communities in blue and orange respectively. Each vertex $i$ is associated with a rotation $\bm{R}_i \in \SO(d)$. Vertices within the same cluster (resp.~across clusters) are independently connected with probability $p$ (resp.~$q$) that are shown in solid lines (resp.~dash lines). The pairwise alignment $\bm{R}_{ij}$ is observed on each edge $(i,j)$.}
    \label{fig:intro_model}
\end{figure}

\subsection{Organization}
The rest of this paper is organized as follows. 
In Section~\ref{sec:model}, we formally define the model and formulate non-convex optimization problems for recovery. In Section~\ref{sec:method}, we propose SDP relaxations and establish the conditions for exact recovery in the case of two communities. Numerical results are given in Section~\ref{sec:exp} to evaluate the proposed SDPs and confirm our theorems. We conclude with some future directions in Section~\ref{sec:discussion_summary}. For clarity of presentation, most of the proofs are deferred to appendix and supplementary material. 

\subsection{Notations} 
\label{sec:notation}
Throughout this paper we use the following notations: given a matrix $\bm{X}$, its transpose is denoted by $\bm{X}^\top$. An $m \times n$ matrix of all zeros is denoted by $\bm{0}_{m \times n}$ (or $\bm{0}$ for brevity). An $n \times n$ identity matrix is denoted by $\bm{I}_n$. $\bm{1}_n \in \mathbb{R}^{n}$ represents a vector of length $n$ with all ones. 
For a real symmetric matrix $\bm{X}$, its maximum and minimum eigenvalues are denoted by $\lambda_{\text{max}}(\bm{X})$ and $\lambda_{\text{min}}(\bm{X})$ respectively. 
$\|\bm{X}\|$ and $\|\bm{X}\|_\mathrm{F}$ denote the operator norm and the Frobenius norm of $\bm{X}$ respectively. For two matrices $\bm{X}$ and $\bm{Y}$, their inner product is written as $\langle\bm{X}, \bm{Y}\rangle = \textrm{Tr}(\bm{X}^\top\bm{Y})$. For two non-negative functions $f(n)$ and $g(n)$, $f(n) = O(g(n))$ (resp.~$f(n) = \Omega(g(n))$) denotes if there exists a constant $C > 0$ such that $f(n) \leq Cg(n)$ (resp.~$f(n) \geq Cg(n)$) for all sufficiently large $n$, $f(n) = o(g(n))$ denotes for every constant $C > 0$ it holds $f(n) \leq Cg(n)$ for all sufficiently large $n$. Given a matrix $\bm{X}$ of size $md \times nd$, we usually treat it as an $m \times n$ block matrix such that $\bm{X}_{ij} \in \mathbb{R}^{d \times d}$ denotes its $(i,j)$-th block, and the $i$-th \textit{block row} (resp.~$j$-th \textit{block column}) of $\bm{X}$ is referred to as the sub-matrix that contains $\bm{X}_{ij}$ for all $j = 1, \ldots, n$ (resp.~$i = 1, \ldots, m$). 

\section{Our probabilistic model and problem formulation}
\label{sec:model}
Formally, we assume that $n$ agents in a network fall in $K$ underlying communities. Each agent $i$ has a cluster membership $\kappa(i) \in \{ 1, \dots, K\}$, and is associated with an unknown rotation $\bm{R}_i \in \SO(d)$. Let $C_k$ denotes the set of agents in the $k$-th cluster and $m_k = |C_k|$ is the cluster size. A random graph $G = (V, E)$ with node set $V$ and edge set $E$ is generated such that each pair of nodes $(i,j)$ is connected with probability $p$ if agents $i$ and $j$ are in the same cluster, i.e. $\kappa(i) = \kappa(j)$, and with probability $q$ if they are in different clusters, i.e. $\kappa(i) \neq \kappa(j)$. In addition, the rotational alignment $\bm{R}_{ij}$ is observed on each edge $(i,j) \in E$ in a way that if $\kappa(i) = \kappa(j)$, we obtain $\bm{R}_{ij} = \bm{R}_i\bm{R}_j^\top$ which equals to the truth alignment, while if $\kappa(i) \neq \kappa(j)$, we have $\bm{R}_{ij} \sim \text{Unif}(\mathrm{SO}(d))$ that is uniformly drawn from $\SO(d)$. 

Given the model, our observation on the graph $G$ can be represented by a symmetric block \textit{observation matrix} $\bm{A} \in \mathbb{R}^{nd \times nd}$, whose $(i,j)$-th block $\bm{A}_{ij} \in \mathbb{R}^{d \times d}$ for any $i \leq j$ satisfies 
\begin{equation}
  \bm{A}_{ij} = 
    \begin{cases}
    \bm{R}_i\bm{R}_j^\top,  &\; \text{with probability } p \text{ and when } \kappa(i) =  \kappa(j), \\
    \bm{R}_{ij} \sim \text{Unif}(\mathrm{SO}(d)), &\; \text{with probability } q \text{ and when } \kappa(i) \neq \kappa(j),\\
    \bm{0}, &\; \text{otherwise}. 
    \end{cases} 
    \label{eq:clean_observation}
\end{equation}
We set $\bm{A}_{ij} = \bm{A}_{ji}$, and the diagonal blocks $\bm{A}_{ii} = \bm{0}, i = 1,\ldots, n$.
As a result, $\bm{A}$ can be viewed as an extension of the adjacency matrix to incorporate not only the edge connections but also the pairwise rotational alignments. Besides, we denote the matrix $\bm{M}^* \in \mathbb{R}^{nd \times nd}$ as 
\begin{equation}
    \bm{M}_{ij}^* = 
    \begin{cases}
    \bm{R}_i\bm{R}_j^\top,  &  \kappa(i) = \kappa(j), \\
    \bm{0}, &\; \text{otherwise}.  
    \end{cases}
    \label{eq:ground_truth_def}
\end{equation}
As we can see, $\bm{M}^*$ is the \textit{clean observation matrix} such that $\bm{M}^*\!=\!\bm{A}$ when $p\!=\!1$ and $q\!=\!0$. 
Specifically, for the case of two clusters, assuming that the first $m_1$ nodes belong to $C_1$ and the remaining $m_2$ nodes belong to $C_2$, then $\bm{M}^*$ can be written as
\begin{equation}
\label{eq:M_ground_truth}
\bm{M}^* = \begin{pmatrix}
\widetilde{\bm{M}}_1^* & \bm{0} \\
\bm{0} & \widetilde{\bm{M}}_2^*
\end{pmatrix}
\end{equation}
where $\widetilde{\bm{M}}_1^* \in \mathbb{R}^{m_1 d \times m_1 d}$ and $\widetilde{\bm{M}}_2^* \in \mathbb{R}^{m_2 d \times m_2 d}$ represent $C_1$ and $C_2$ respectively.


\subsection{Problem formulation} 
Given the observation matrix $\bm{A}$, the maximum likelihood estimator for our probabilistic model in~\eqref{eq:clean_observation} is simply the set of rotations $\{ \bm{\mathcal{R}}_i\}_{i = 1}^n$, and the partition of nodes $\{\mathcal{C}_k\}_{k = 1}^K$ that satisfy as many equations $\bm{A}_{ij} = \bm{\mathcal{R}}_i \bm{\mathcal{R}}_j^\top $ as possible for $i$ and $j$ identified to be in the same community.
Therefore, we introduce the consistency score 
\begin{equation}
  CS(\bm{\mathcal{R}}_1, \dots, \bm{\mathcal{R}}_n; \; \mathcal{C}_1, \dots \mathcal{C}_K) =  \#\{(i,j) \in E \; | \; \bm{A}_{ij} = \bm{\mathcal{R}}_i\bm{\mathcal{R}}_j^\top \text{ and } \mathcal{K}(i) = \mathcal{K}(j)\},
    \label{eq:cs}
\end{equation}
where $\mathcal{K}(i)$ denotes the identified cluster for $i$. 
Then, maximizing~\eqref{eq:cs} is equivalent to maximizing the log-likelihood. We remark that \eqref{eq:cs} can be viewed as an extension of \cite[Eq.~(4)]{singer2011angular} which focuses on the angular synchronization problem with the random corruption model. 

The maximum likelihood estimation suffers from a major drawback that \eqref{eq:cs} is non-convex and thus computationally intractable to be exactly solved, especially when $n$ is large. In this paper, we consider a different approach that leads to an efficient solution by SDP, which can be solved in polynomial time complexity. 

Our approach starts from the following program which also measures the consistency of $\bm{A}_{ij}$ with $\bm{\mathcal{R}}_i \bm{\mathcal{R}}_j^\top$ for $i$ and $j$ identified to be in the same cluster:
\begin{equation}
    \max_{\bm{\mathcal{R}}_1, \ldots, \bm{\mathcal{R}}_n \in \SO(d), \; \mathcal{C}_1, \ldots, \mathcal{C}_K} \sum_{i, j \in \mathcal{C}_k, \forall k} \left\langle \bm{A}_{ij}, \; \bm{\mathcal{R}}_{i} \bm{\mathcal{R}}_j^\top \right\rangle. 
    \label{eq:clustersync}
\end{equation}
As a result, \eqref{eq:clustersync} attempts to identify $\{\mathcal{C}_k\}_{k = 1}^K$ and $\{\bm{\mathcal{R}}_i\}_{i = 1}^n$ simultaneously such that the identification relies on both the edge connections and the consistency of alignments.  Furthermore, by introducing $\bm{M} \in \mathbb{R}^{nd \times nd}$ with its $(i,j)$-th block
\begin{equation}
    \bm{M}_{ij} = 
    \begin{cases}
    \bm{\mathcal{R}}_i \bm{\mathcal{R}}_j^\top,  &\; \mathcal{K}(i) =\mathcal{K}(j), \\
    \bm{0}, &\; \text{otherwise},
    \end{cases}
    \label{eq:Mform}
\end{equation} 
then we can recast \eqref{eq:clustersync} as the following program, where the objective is linear in $\bm{M}$:
\begin{equation}
    \begin{aligned}
        \max_{\bm{M}} &\; \left\langle \bm{A}, \bm{M}\right\rangle \quad 
        \mathrm{s.t.} \quad \text{$\bm{M}$ satisfies the form in \eqref{eq:Mform}}.
    \end{aligned}
    \label{eq:clustersync_matrix}
\end{equation}
If the cluster sizes $\{m_k\}_{k = 1}^K$ are known to us, additional constraints should be added as
\begin{equation}
    \begin{aligned}
        \max_{\bm{M}} \; \left\langle \bm{A}, \bm{M}\right\rangle \quad 
        \mathrm{s.t.} \quad \text{$\bm{M}$ satisfies the form in \eqref{eq:Mform},  $|\hat{C}_k| = m_k$ for $k = 1, \ldots, K$.}
    \end{aligned}
    \label{eq:clustersync_matrix_known}
\end{equation}
Since directly solving either \eqref{eq:clustersync_matrix} or \eqref{eq:clustersync_matrix_known} is still non-convex, in the following section we introduce semidefinite relaxations to efficiently solve them.

\section{Semidefinite relaxation}
\label{sec:method}
For simplicity, in Sections~\ref{sec:SDP_equal}--\ref{sec:rounding}, we consider the case of two clusters with the following three scenarios: (1) two cluster sizes $m_1\!=\!m_2$ are equal and known; (2) $m_1$ and $m_2$ are unequal and known; (3) $m_1$ and $m_2$ are unknown.
For each case, we develop its tailored SDP and provide corresponding performance guarantee. In Section~\ref{sec:kcluster}, we extend our approach to general cases with an arbitrary number of underlying clusters.

\subsection{Two equal-sized clusters with known cluster sizes}
\label{sec:SDP_equal}
In this case, let $m = n/2$ be the cluster size, then the ground truth $\bm{M}^*$ in \eqref{eq:M_ground_truth} satisfies the following convex constraints:
\begin{enumerate}
    \vspace{0.1cm}
    \item $\bm{M}^*$ is positive semi-definite, i.e. $\bm{M}^* \succeq 0$.
    \vspace{0.1cm}
    \item Each diagonal block of $\bm{M}^*$ is an identity matrix, i.e. $\bm{M}^*_{ii} = \bm{I}_d$, $i = 1, \ldots, n$.
    \vspace{0.1cm}
    \item The $i$-th block row of $\bm{M}^*$ satisfies $\sum_{j = 1}^n \|\bm{M}^*_{ij}\|_\mathrm{F} \leq m\sqrt{d}$,  $i = 1,\ldots, n,\;  m = n/2.$
\end{enumerate}
Notably, the equality in the last constraint holds as $\sum_{j = 1}^n \|\bm{M}^*_{ij}\|_\mathrm{F}\! = \!m\sqrt{d}$. However, the norm equality constraint is non-convex and therefore a convex relaxation by replacing ``$=$'' with ``$\leq$'' is necessary. Based on these, a semidefinite relaxation of \eqref{eq:clustersync_matrix_known} can be stated as 
\begin{equation}
    \begin{aligned}
    \bm{M}_{\text{SDP}} &= \argmax_{\bm{M}} \quad \left\langle \bm{A}, \bm{M} \right\rangle\\
    \mathrm{s.t.} &\quad \bm{M} \succeq 0, \quad \bm{M}_{ii} = \bm{I}_d, \quad \sum_{j= 1}^n \|\bm{M}_{ij}\|_\mathrm{F} \leq m\sqrt{d}, \quad i = 1,\ldots, n.
    \end{aligned}
    \label{eq:SDP_equal}
\end{equation}
Then the resulting semidefinite program~(SDP) can be solved with polynomial complexity~\cite{vandenberghe1996semidefinite}. We remark that in the context of SBM, SDPs similar to \eqref{eq:SDP_equal} have been studied in~\cite{abbe2015exact, hajek2016achievinga} on two equal-sized clusters, where $\bm{M}_{ij} \in \mathbb{R}$ is instead a scalar and a convex equality constraint $\sum_{i = 1}^n \bm{M}_{ij}\!=\!m$ is imposed. In contrast, we have to use the inequality for convexity.

To characterize the performance of \eqref{eq:SDP_equal}, we identify the conditions on the model parameters $(p,q,n)$ such that \eqref{eq:SDP_equal} achieves exact recovery, i.e.,~$\bm{M}_{\text{SDP}} \!=\! \bm{M}^*$ as follows.
\begin{theorem}
\label{thm:2}
Let $p = \alpha \log n/n$ and $q = \beta \log n/n$. For sufficiently large $n$, \eqref{eq:SDP_equal} achieves exact recovery with probability $1 - n^{-\Omega(1)}$ if 
\begin{equation}
    \alpha - \sqrt{2\ell_d\beta} \log\left(\frac{e\alpha}{\sqrt{2\ell_d\beta}}\right) > 2, \quad \text{where} \quad  \ell_d = 
    \begin{cases}
    1, &\quad \text{when $d = 2$},\\
    2, &\quad \text{when $d > 2$}.
    \end{cases}
    \label{eq:2_cluster_equal_d_2}
\end{equation}
\label{the:2_cluster_equal_d_2}
\end{theorem}

Here, the difference in \eqref{eq:2_cluster_equal_d_2} between the case $d\!=\!2$ and $d\!>\!2$ lies in the different concentration inequalities of random rotation matrices used in the proof (see Section~\ref{sec:sketch_proof_equal_main}). Also, we conjecture the result for $d > 2$ can be improved such that  $\ell_d$ can be less than $2$ (not necessarily an integer). When $d = 2$, empirically \eqref{eq:2_cluster_equal_d_2} is shown to be tight in a way that it sharply characterizes the phase transition boundary for exact recovery (see Figure~\ref{fig:two_equal_unequal_exp}). 


\subsubsection{The sketch of proof}
\label{sec:sketch_proof_equal_main}
We outline the key steps for proving Theorem~\ref{the:2_cluster_equal_d_2} based on showing the dual certificate of the ground truth $\bm{M}_{\text{SDP}} = \bm{M}^*$ for \eqref{eq:SDP_equal}. The proofs of all the lemmas are deferred to Section~\ref{sec:proof_two_equal} in the supplementary material.
\vspace{0.05cm}
\paragraph{Step 1: Derive KKT conditions, uniqueness and optimality of $\bm{M}^*$} The Lagrangian of the optimization problem \eqref{eq:SDP_equal} is,
\begin{equation*}
    L(\bm{M}, \bm{\Lambda}, \bm{Z}, \bm{\mu}) = -\left\langle\bm{A}, \bm{M}\right\rangle - \left\langle \bm{\Lambda}, \bm{M} \right\rangle - \sum_{i = 1}^n \left\langle \bm{Z}_i, \bm{M}_{ii} - \bm{I}_d \right\rangle - \sum_{i = 1}^n \bigg \langle \mu_i, m\sqrt{d} - \sum_{j = 1}^n  \left \|\bm{M}_{ij} \right \|_\mathrm{F} \bigg \rangle,
    \label{eq:lagrangian_equal_main}
\end{equation*}  
where $\bm{\Lambda}$, $\bm{Z} = \mathrm{diag}\left(\left\{\bm{Z}_i\right\}_{i = 1}^n \right)$ and $\bm{\mu} = \left\{\mu_i\right\}_{i = 1}^n$ are the dual variables associated with the constraints $\bm{M} \succeq 0, \; \bm{M}_{ii} = \bm{I}_d$ and $\sum_{j=1}^n\|\bm{M}_{ij}\|_\mathrm{F} \leq m\sqrt{d}$ for $i = 1,\ldots, n$, respectively. Then the KKT conditions at $\bm{M} = \bm{M}^*$ are listed as,
\begin{alignat}{2}
    &\bullet \; \textit{Stationarity:} \; &&-\bm{A} - \bm{\Lambda} - \mathrm{diag}(\{\bm{Z}_i\}_{i = 1}^n) + \bm{\Theta} + \bm{\Theta}^{\top}  = \bm{0}, \label{eq:KKT_stationarity_main}\\
    & &&\bm{\Theta} = \left [ \bm{\Theta}_{ij} \right ]_{i, j = 1}^n, \quad
    \begin{cases} \displaystyle
    \bm{\Theta}_{ij} = \mu_i\bm{M}_{ij}^*/\sqrt{d}, & \bm{M}_{ij}^* \neq \bm{0}, \\
    \|\bm{\Theta}_{ij}\|_\mathrm{F} \leq \mu_i  & \bm{M}_{ij}^* = \bm{0}.
    \end{cases}
    \label{eq:theta_definition_main}\\
    &\bullet \; \textit{Comp. slackness:} \quad \quad  &&\left\langle \bm{\Lambda}, \bm{M}^* \right\rangle = 0, \quad \bigg \langle \mu_i, m\sqrt{d} - \sum_{j = 1}^n \left \|\bm{M}_{ij}^* \right \|_\mathrm{F} \bigg \rangle = 0, \; i = 1, \ldots, n. \label{eq:KKT_complementary_main}\\
    &\bullet \; \textit{Dual feasibility:} \; &&\bm{\Lambda} \succeq 0, \quad \mu_i \geq 0, \quad i = 1,\ldots, n. 
    \label{eq:KKT_dual_feasibility_main}
\end{alignat}
Notably, in \eqref{eq:theta_definition_main} we introduce a matrix of dual variables $\bm{\Theta} \in \mathbb{R}^{nd \times nd}$ that incorporates $\{\mu_i\}_{i = 1}^n$ when $\bm{M}^{*}_{ij} \neq \bm{0}$. When $\bm{M}_{ij}^* = \bm{0}$, $\bm{\Theta}_{ij}$ is not determined but can be chosen arbitrarily as long as $\|\bm{\Theta}_{ij}\|_\mathrm{F} \leq \mu_i$. This is because the derivative of $\|\bm{M}_{ij}\|_\mathrm{F}$ is undefined at $\bm{M}_{ij} = \bm{0}$, then we use subderivative instead.
Given the above, the following establishes the uniqueness and optimality of $\bm{M}^*$:
\begin{lemma}
Given $\bm{M}^*$ defined in \eqref{eq:ground_truth_def}, suppose there exist dual variables $\bm{\Lambda}$, $\bm{Z}$, and $\bm{\Theta}$
that satisfy the KKT conditions \eqref{eq:KKT_stationarity_main} - \eqref{eq:KKT_dual_feasibility_main} as well as 
\begin{equation}
    \mathcal{N}(\bm{\Lambda}) = \mathcal{R}(\bm{M}^*),
    \label{eq:1_main}
\end{equation}
where $\mathcal{N}(\cdot)$ and $\mathcal{R}(\cdot)$ denote the null space and column space respectively. Then $\bm{M}^*$ is the optimal and unique solution to \eqref{eq:SDP_equal}.
\label{lemma:1_main}
\end{lemma}


\paragraph{Step 2: Construct dual variables} 
Our next step is to find dual variables that satisfy the conditions in Lemma~\ref{lemma:1_main}. 
Without loss of generality, we assume the rotation matrix $\bm{R}_i\!=\!\bm{I}_d$ for each $i$. Then for any pair of $(i,j)$ that $\kappa(i)\!=\! \kappa(j)$, we have $\bm{M}_{ij}^*\!=\!\bm{I}_d$ and $\bm{A}_{ij}\!=\!r_{ij}\bm{I}_d$, where $r_{ij} \in \{0,1\}$ is a Bernoulli random variable such that $\mathbb{P}\{r_{ij}\!=\!1\}\!=\!p$.  Based on these, inspired by~\cite{li2018convex} we make the following guess for the dual variables.

\begin{lemma} 
The following variables satisfy the KKT conditions \eqref{eq:KKT_stationarity_main} -- \eqref{eq:KKT_complementary_main} and $\mu_i \geq 0, \; i = 1, \ldots, n$ in \eqref{eq:KKT_dual_feasibility_main}:
\begin{alignat*}{2}
    \bm{\Theta}_{ij}  &= 
    \begin{cases} \displaystyle
    \mu_i\bm{I}_d/\sqrt{d}, & \kappa(i) = \kappa(j),\\
    \widetilde{\bm{\alpha}}_i,  & \kappa(i) \neq \kappa(j),
    \end{cases} \quad &&\widetilde{\bm{\alpha}}_i = \frac{1}{m}\sum_{j: \kappa(i) \neq \kappa(j) }\bm{A}_{ij} - \frac{1}{2m^2} \sum_{j: \kappa(i) \neq \kappa(j)}\sum_{s: C(s) = \kappa(i)} \bm{A}_{sj},\\
    \mu_i  &= \|\widetilde{\bm{\alpha}}_i\|_\mathrm{F}, &&\bm{Z}_i = \bigg(\frac{m\mu_i}{\sqrt{d}} +\sum_{s: C(s) = \kappa(i)}\bigg(\frac{\mu_s}{\sqrt{d}} - r_{is}\bigg)\bigg) \bm{I}_d,
\end{alignat*}
$\hspace{0.85cm} \bm{\Lambda} = -\bm{A} - \mathrm{diag} \left (\{\bm{Z}_i \}_{i = 1}^n \right) + \bm{\Theta}  + \bm{\Theta}^{\top}.$
\label{lemma:guess_dual}
\end{lemma}

When constructing the dual certificate, for $i,j$ that $\kappa(i)\neq \kappa(j)$, $\bm{\Theta}_{ij}$ can have many choices. Here we simply assume $\bm{\Theta}_{ij} = \widetilde{\bm{\alpha}}_i$ that is identical for all $j$ when $\kappa(i) \neq \kappa(j)$. This choice allows us to determine all dual variables via the KKT conditions and \eqref{eq:1_main}. Now it remains ensure $\bm{\Lambda} \succeq 0$ and \eqref{eq:1_main}. To this end, we use the following property for $\bm{\Lambda}$.
\begin{lemma} Given the dual variables in Lemma~\ref{lemma:guess_dual}, $\bm{\Lambda}$ can be expressed as
\begin{equation}
    \bm{\Lambda} = (\bm{I}_{nd} - \bm{\Pi})\underbrace{(\mathbb{E}[\bm{A}] - \bm{A} - \bm{Z} + p\bm{I}_{nd} )}_{=:\widetilde{\bm{\Lambda}}}(\bm{I}_{nd} - \bm{\Pi}), \quad \text{where} \quad \bm{\Pi} = \left(\frac{1}{m}\right)\bm{M}^*.
    \label{eq:lambda_equal_main}
\end{equation}
Then, $\bm{\Lambda} \succeq 0$ and $\mathcal{N}(\bm{\Lambda}) = \mathcal{R}(\bm{M}^*)$ in \eqref{eq:1_main} are satisfied as long as $\widetilde{\bm{\Lambda}} \succ 0$.
\label{lemma:simplify_lambda}
\end{lemma}

\paragraph{Step 3: Find the condition for $\widetilde{\bm{\Lambda}} \succ 0$} 
Applying Weyl's inequality~\cite{stewart1998perturbation} yields $\lambda_{\text{min}}(\widetilde{\bm{\Lambda}}) \geq  \lambda_{\text{min}}(p\bm{I}_{nd} - \bm{Z}) - \|\mathbb{E}[\bm{A}] - \bm{\bm{A}}\|$.
Then for $\widetilde{\bm{\Lambda}} \succ 0$, it suffices to show
\begin{equation}
    \lambda_{\text{min}}(p\bm{I}_{nd} - \bm{Z}) > \|\mathbb{E}[\bm{A}] - \bm{\bm{A}}\|.
    \label{eq:cond_equal}
\end{equation}
To this end, we have the following lower bound for $\lambda_{\text{min}}(p\bm{I}_{nd} - \bm{Z})$. 

\begin{lemma}
Let $p = \alpha \log n/n, q = \beta \log n/n$. 
For $n$ that is sufficiently large, 
\begin{equation*}
    \lambda_{\text{min}}(p\bm{I}_{nd} - \bm{Z}) \geq \min_i x_i - \max_i y_i - \max\{\epsilon_1, \epsilon_2 \} + p 
    \label{eq:lambda_min}
\end{equation*}
where $x_i$, $y_i$, $\epsilon_1$ and $\epsilon_2$ are defined as
\begin{enumerate}
    \vspace{0.1cm}
    \item $x_i := \sum_{s: C(s) = \kappa(i)}r_{is}$, which satisfies
    \begin{equation*}
        \min_{i}x_i \geq \frac{\tau}{2}\log n \quad \text{with probability} \quad 1 - n^{1 -\frac{1}{2}\left(\alpha - \tau \log \left(\frac{e\alpha}{\tau} \right) + o(1)\right)},
    \end{equation*}
    for $\tau \in [0, \alpha)$ such that $1 -\frac{1}{2}\left(\alpha - \tau \log \left(\frac{e\alpha}{\tau} \right) + o(1)\right) < 0$;
    \vspace{0.1cm}
    \item $y_i := \frac{m\mu_i}{\sqrt{d}}$, which satisfies
    \begin{equation*}
        \max_i y_i \leq
    \sqrt{\frac{\ell_d(c+1)\beta}{2}}\log n + o(\log n) \quad \text{with probability} \quad 1 - n^{-c}, \quad 
    \ell_d = 
    \begin{cases}
    1, &\; d = 2,\\
    2, &\; d > 2;
    \end{cases}
    \end{equation*}
    \item $\epsilon_{1} := \frac{1}{\sqrt{d}} \sum_{s \in C_1} \mu_s$ and $\epsilon_{2} := \frac{1}{\sqrt{d}} \sum_{s \in C_2} \mu_s$, which satisfy 
    \begin{equation*}
    \max\{\epsilon_1, \epsilon_2\} \leq \frac{2c}{3}\log n + o(\log n) \quad \text{with probability} \quad 1 - n^{-c}.
    \end{equation*}
\end{enumerate}
\label{claim:x_i_y_i_equal}
\end{lemma}

Here, $x_i \sim \textrm{Binom}(m,p)$ follows a binomial distribution and is bounded by \cite[Lemma 2]{hajek2016achievinga}. For $y_i$, the key point is bounding $\|\sum_{i = 1}^m\bm{R}_i\|_\mathrm{F}$ where $\{\bm{R}_i\}_{i = 1}^m$ are i.i.d. random matrices uniformly drawn from $\SO(d)$. When $d > 2$, $\|\sum_{i = 1}^m\bm{R}_i\|_\mathrm{F}$ is bounded by matrix Bernstein inequality~\cite{tropp2015introduction}. Specifically when $d = 2$, a sharper result is obtained by explicitly computing the moments of $\|\sum_{i = 1}^m\bm{R}_i\|_\mathrm{F}$ and then apply Markov inequality~(see Appendix~\ref{sec:sum_random_orthogonal}). Notice the $(1\!+\!c)$ factor in the bound of $\max_i{y_i}$ comes from the union bound over all $y_i$. The bound for $\max\{\epsilon_1, \epsilon_2\}$ relies on the concentration of $\mu_i$ and is bounded by Bernstein's inequality~\cite{vershynin2018high}.


It remains to derive an upper bound of $\|\mathbb{E}[\bm{A}] - \bm{A}\|$, which has been studied~(e.g. \cite{hajek2016achievinga, lei2015consistency}) in the context of SBM where $\bm{A}$ is an adjacency matrix such that $\bm{A}_{ij} \in \{0, 1\}$. However, here $\bm{A}_{ij}$ is a random rotation matrix for $\kappa(i) \neq \kappa(j)$ rather than being $\{0, 1\}$-valued. To handle this, a key ingredient is bounding the operator norm of a random block matrix with i.i.d blocks~(see Theorem~\ref{lemma:spec_bound_S_12} in Appendix~\ref{sec:norm_rand_matrix}), which leads to the following result for $\|\mathbb{E}[\bm{A}] - \bm{A}\|$.
\begin{lemma} Let the two clusters be of size $m_1$ and $m_2$ $(m_1+m_2\!=\!n)$, and $p, q = \Omega(\log n/n)$. Then for $\bm{A}$ defined in \eqref{eq:clean_observation}, it satisfies
\begin{equation*}
     \|\mathbb{E}[\bm{A}] - \bm{A}\| \leq c_1(\sqrt{pm_1} + \sqrt{pm_2}) + c_2(\sqrt{qm_1} + \sqrt{qm_2}) + O(\sqrt{\log n})
     \label{eq:EA_A_equal}
\end{equation*}
with probability $1 - n^{-c}$ for $c > 0$, where $c_1, c_2 > 0$ are some universal constants.
\label{lemma:concentration_A}
\end{lemma}

\begin{proof}[Proof of Theorem~\ref{the:2_cluster_equal_d_2}]
By applying the union bound on $\min_i x_i$, $\max_i y_i$ and $\max\{\epsilon_1, \epsilon_2\}$ in Lemma~\ref{claim:x_i_y_i_equal}, we obtain 
\begin{equation*}
    \lambda_{\text{min}}(p\bm{I}_d - \bm{Z}) \geq \bigg(\frac{\tau}{2} - \sqrt{\frac{\ell_d(c+1)\beta}{2}} - \frac{2c}{3}\bigg) \log n + o(\log n),
\end{equation*}
with probability $1 - n^{-\Omega(1)}$, as long as the condition in Lemma~\ref{claim:x_i_y_i_equal} that
\begin{equation}
    1 -\frac{1}{2}\left(\alpha - \tau \log\left(\frac{e\alpha}{\tau}\right)\right) < 0, \quad \tau \in (0,\alpha]
    \label{eq:min_x_i_bound_cond}
\end{equation}
is satisfied. Also, from Lemma~\ref{lemma:concentration_A}, we have $\|\mathbb{E}[\bm{A}] - \bm{A}\| = O(\sqrt{\log n})$ with high probability. This implies, as $n$ is large, \eqref{eq:cond_equal} holds if $\lambda_{\text{min}}(p\bm{I}_d - \bm{Z}) = \Omega(\log n)$, that is equivalent to $\frac{\tau}{2} - \sqrt{\frac{\ell_d(c+1)\beta}{2}} - \frac{2c}{3} > 0$, which reduces to $\tau > \sqrt{2\ell_d\beta}$ by taking $c \rightarrow 0$. It remains to ensure $\tau \in [0, \alpha)$ that satisfies \eqref{eq:min_x_i_bound_cond}. To this end, since the LHS of \eqref{eq:min_x_i_bound_cond} is monotonically increasing for $\tau \in [0, \alpha)$, by choosing $\tau \rightarrow \sqrt{2\ell_d\beta}$ the condition \eqref{eq:min_x_i_bound_cond} becomes \eqref{eq:2_cluster_equal_d_2}.
\end{proof}

\subsection{Two unequal-sized clusters with known cluster size}
\label{sec:SDP_unequal}
Now we move to the scenario where the two cluster sizes $m_1$ and $m_2$ are unequal and known to us.
In this case, the ground truth $\bm{M}^*$ still has the form in \eqref{eq:M_ground_truth} and convex relaxations similar to \eqref{eq:SDP_equal} can be applied. However, unlike the equal-sized case, the row-sum constraint satisfies $\sum_{j=1}^n \|\bm{M}_{ij}\|_\mathrm{F} \!=\! m_1\sqrt{d}$ if $i \in C_1$ and $\sum_{j=1}^n \|\bm{M}_{ij}\|_\mathrm{F}\!=\!m_2\sqrt{d}$ if $i \in C_2$. Therefore a convex constraint can be imposed as $\sum_{j=1}^n \|\bm{M}_{ij}\|_\mathrm{F} \!\leq\! \max\{m_1, m_2\}\sqrt{d}$. 
So this imposes an upper-bound on the cluster sizes, based on the size of the larger cluster. 
Besides, to incorporate the information of the smaller cluster, we consider an additional constraint on the sum of all blocks as $\sum_{i, j=1}^n \|\bm{M}_{ij}\|_\mathrm{F} \!\leq\! (m_1^2 + m_2^2)\sqrt{d} $. Given the above, the resulting SDP is 
\begin{equation}
    \begin{aligned}
    \bm{M}_{\text{SDP}} &= \argmax_{\bm{M}} \quad \left\langle \bm{A}, \bm{M} \right\rangle\\
    \mathrm{s.t.} &\quad \bm{M} \succeq 0, \quad \bm{M}_{ii} = \bm{I}_d, \quad \sum_{j = 1}^n\|\bm{M}_{ij}\|_\mathrm{F} \leq \max\{m_1, m_2\}\sqrt{d}, \quad\forall  i = 1,\ldots, n, \\
    & \quad \sum_{i, j = 1}^n\|\bm{M}_{ij}\|_\mathrm{F} \leq \left(m_1^2 +  m_2^2 \right) \sqrt{d}.
    \end{aligned}
    \label{eq:SDP_unequal_known}
\end{equation}

For \eqref{eq:SDP_unequal_known}, we also obtain the corresponding condition for exact recovery as the following. 
\begin{theorem}
Let $p = \alpha \log n/n, q = \beta \log n/n$. Suppose $m_1 = \rho n, m_2 = (1-\rho)n$ for some $\rho \in (1/2,1)$ such that $m_1 > m_2$. Let $\tau_1^*, \tau_2^* \in [0, \alpha)$ be the roots of the equations
\begin{equation}
    \alpha - \tau_1^* \log\left(\frac{e\alpha}{\tau_1^*}\right) = \frac{1}{\rho}, \quad \alpha - \tau_2^* \log\left(\frac{e\alpha}{\tau_2^*}\right) = \frac{1}{1-\rho}.
    \label{eq:tau_1_2_star}
\end{equation}
As $n$ is sufficiently large, \eqref{eq:SDP_unequal_known} achieves exact recovery with probability $1 - n^{-\Omega(1)}$ if 
\begin{equation}
\begin{cases}
\displaystyle \alpha > \frac{1}{1-\rho}, \quad & \text{when $\beta = 0$},\\
\displaystyle \alpha > \frac{1}{1-\rho}, \quad \frac{\tau_1^*}{\delta_2} > 1, \quad \text{and} \quad \frac{\tau_1^*}{\delta_1} > \max\left\{1 -  \frac{\tau_2^*}{2\delta_2},\; \frac{1}{2}\right\}, \quad & \text{when $\beta > 0$},
\end{cases}
\label{eq:cond_unequal_two_case}
\end{equation}
where 
$\delta_1\!:=\!\left(\frac{2}{\sqrt{\rho}} \!+\! \frac{1}{\sqrt{1-\rho}}\right)\sqrt{\ell_d\beta}, \; \delta_2\!:=\!\left(\frac{1}{\sqrt{\rho}} \!+\! \frac{1}{\sqrt{1-\rho}}\right)\sqrt{\ell_d\beta}$. $\ell_d=1$ if $d=2$ and $\ell_d=2$ if $d>2$. 
\label{the:2_cluster_unequal_d_2}
\end{theorem}

Notice that the function $f(\tau) := \alpha - \tau\log\left(e\alpha/\tau\right)$ monotonically decreases for $\tau \in [0, \alpha)$, also it satisfies $f(\alpha) = 0$ and $\lim_{\tau \rightarrow 0} f(\tau) = \alpha$. Therefore, there always exists a unique pair of solutions $\tau_1^*, \tau_2^*$ of \eqref{eq:tau_1_2_star} for $ \tau_1^*, \tau_2^* \in [0, \alpha)$ as long as $\alpha > 1/(1-\rho)$.

\subsubsection{The sketch of proof} 
\label{sec:proof_sketch_uneuqal_main}
We outline the key steps for proving Theorem~\ref{the:2_cluster_unequal_d_2} and leave the detail to supplementary material Section~\ref{sec:proof_two_unequal}. 
The steps are similar to those outlined in Section~\ref{sec:sketch_proof_equal_main}, while a considerable amount of effort is devoted to handle the imbalanced clusters and the additional constraint $\sum_{i, j = 1}^n\|\bm{M}_{ij}\|_\mathrm{F} \leq \left(m_1^2 +  m_2^2 \right) \sqrt{d}$. 

\vspace{0.15cm}
\paragraph{Step 1: Derive KKT conditions, uniqueness and optimality of $\bm{M}^*$}
The Lagrangian of \eqref{eq:SDP_unequal_known} is given as
\begin{equation*}
    \begin{aligned}
    L &= -\left\langle\bm{A}, \bm{M}\right\rangle - \left\langle\bm{\Lambda}, \bm{M}\right\rangle - \sum_{i = 1}^n\left\langle\bm{Z}_i, \bm{M}_{ii} - \bm{I}_d\right\rangle + \sum_{i = 1}^n \bigg\langle \mu_i, \sum_{j = 1}^n \|\bm{M}_{ij}\|_\mathrm{F} - m_1\sqrt{d}\bigg\rangle\\
    &\quad + \nu\bigg(\sum_{i, j}\|\bm{M}_{ij}\|_\mathrm{F} - (m_1^2 + m_2^2)\sqrt{d}\bigg),
    \end{aligned}
\end{equation*}
where $\nu$ is the dual variable associated with the extra constraint. Then the KKT conditions of \eqref{eq:SDP_unequal_known} when $\bm{M} = \bm{M}^*$ are given as
\begin{alignat}{2}
    &\bullet \; \textit{Stationarity:} \; &&-\bm{A} - \bm{\Lambda} - \diag(\{\bm{Z}_i\}_{i = 1}^n) + \bm{\Theta} + \bm{\Theta}^{\top}  = \bm{0}, \label{eq:KKT_stationarity_unequal_new} \\
    & &&\bm{\Theta} \!=\! [\bm{\Theta}_{ij}]_{i,j = 1}^n, \quad 
    \begin{cases}
    \displaystyle 
    \bm{\Theta}_{ij} \!=\! (\mu_i + \nu)\bm{M}_{ij}^*/\sqrt{d}, & \bm{M}_{ij}^* \neq \bm{0},\\
    \|\bm{\Theta}_{ij}\|_\mathrm{F} \leq \mu_i +\nu,  & \bm{M}_{ij}^* = \bm{0}.
    \end{cases} \label{eq:KKT_theta_unequal_new}\\
    &\bullet \; \textit{Comp. slackness:} \quad &&\left\langle \bm{\Lambda}, \bm{M}^* \right\rangle = 0, \quad \nu\bigg(\sum_{i, j}\|\bm{M}_{ij}^*\|_\mathrm{F} - (m_1^2 + m_2^2)\sqrt{d}\bigg) = 0, \label{eq:KKT_lambda_M_unequal_new}\\
    & && \bigg\langle \mu_i, m_1\sqrt{d} - \sum_{j = 1}^n \|\bm{M}^*_{ij}\|_{\mathrm{F}}\bigg\rangle = 0, \quad i = 1, \ldots, n,\label{eq:KKT_mu_i_unequal_new} \\
    &\bullet \; \textit{Dual feasibility:} \; &&\bm{\Lambda} \succeq 0, \quad \nu \geq 0, \quad \mu_i \geq 0, \quad i = 1,\ldots, n. 
    \label{eq:KKT_dual_feasibility_unequal_new}
\end{alignat}
The following lemma establishes the uniqueness and optimality of $\bm{M}^*$:
\begin{lemma}
Given $\bm{M}^*$ defined in \eqref{eq:ground_truth_def}, suppose there exists dual variables $\bm{\Lambda}$, $\bm{Z}$, $\bm{\Theta}$, and $\nu$
that satisfy the KKT conditions \eqref{eq:KKT_stationarity_unequal_new} - \eqref{eq:KKT_dual_feasibility_unequal_new} as well as 
\begin{equation}
    \mathcal{N}(\bm{\Lambda}) = \mathcal{R}(\bm{M}^*).
    \label{eq:unequal_N_R_new}
\end{equation}
Then $\bm{M}^*$ is the optimal and unique solution to \eqref{eq:SDP_unequal_known}. 
\label{lemma:uniqueness_unequal_new}
\end{lemma}

\paragraph{Step 2: Construct dual variables} We follow the assumption that $\bm{M}^*_{ij} = \bm{I}_d$ and $\bm{A}_{ij} = r_{ij}\bm{I}_d$ for any $(i,j)$ that $\kappa(i) = \kappa(j)$. Then we make the following guess of the dual variables.
\begin{lemma}With a constant $\gamma \in [1/2, 1]$, the following forms of the variables satisfy \eqref{eq:KKT_stationarity_unequal_new} - \eqref{eq:KKT_mu_i_unequal_new} and $\nu \geq 0$, $\mu_i \geq 0$, for $i = 1, \ldots, n$ in \eqref{eq:KKT_dual_feasibility_unequal_new}.
\begin{align*}
    \widetilde{\bm{\alpha}}_{ij} &= \frac{1}{m_1}\sum_{s_1 \in C_1}\bm{A}_{s_1j} + \frac{1}{m_2}\sum_{s_2 \in C_2}\bm{A}_{is_2} - \frac{1}{m_1m_2}\sum_{s_1 \in C_1}\sum_{s_2 \in C_2}\bm{A}_{s_1s_2}, \quad i \in C_1, j \in C_2,\\
    \bm{\Theta}_{ij} &= 
    \begin{cases} \displaystyle
    (\mu_i + \nu)\bm{I}_d/\sqrt{d}, &  \kappa(i) = \kappa(j),\\
    \gamma \widetilde{\bm{\alpha}}_{ij}, & i \in C_1, j \in C_2,\\
    (1-\gamma) \widetilde{\bm{\alpha}}_{ji}^\top, & i \in C_2, j \in C_1,
    \end{cases} \quad \quad \mu_i = 
    \begin{cases}
        \displaystyle
        \max\{ 0, \gamma \cdot \max_{j \in C_2}\|\widetilde{\bm{\alpha}}_{ij}\|_\mathrm{F} - \nu\}, &\quad i \in C_1, \\
        \displaystyle
        0, &\quad i \in C_2,
    \end{cases}\\
    \nu &= (1-\gamma)\max_{i \in C_2, j \in C_1}\|\widetilde{\bm{\alpha}}_{ji}\|_\mathrm{F}, \hspace{0.25cm} 
    \bm{Z}_i = 
    \begin{cases}
    \displaystyle
    \bigg(\frac{m_1(\mu_i + \nu)}{\sqrt{d}} + \sum_{s \in C_1}\bigg(\frac{\mu_s + \nu}{\sqrt{d}} - r_{is}\bigg)\bigg)\bm{I}_d, &\quad i \in C_1,\\
    \displaystyle
    \bigg(\frac{2m_2\nu}{\sqrt{d}} - \sum_{s \in C_2} r_{is}\bigg)\bm{I}_d, &\quad i \in C_2, \, 
    \end{cases}
    \end{align*}
$\quad \bm{\Lambda} = -\bm{A} - \mathrm{diag} \left (\{\bm{Z}_i \}_{i = 1}^n \right) + \bm{\Theta}  + \bm{\Theta}^{\top}.$
\label{lemma:guess_dual_unequal_new}
\end{lemma}
Furthermore, $\bm{\Lambda}$ can be simplified as the following form.
\begin{lemma}
Given the dual variables in Lemma~\ref{lemma:guess_dual_unequal_new}, $\bm{\Lambda}$ can be expressed as
\begin{equation*}
    \bm{\Lambda} = (\bm{I}_{nd} - \bm{\Pi})\underbrace{(\mathbb{E}[\bm{A}] - \bm{A} - \bm{Z} + p\bm{I}_{nd} )}_{=: \widetilde{\bm{\Lambda}}}(\bm{I}_{nd} - \bm{\Pi}) \quad \text{where} \quad \bm{\Pi} := 
    \begin{pmatrix}
    \widetilde{\bm{M}}^*_{1}/m_1 &\bm{0}\\
    \bm{0} &\widetilde{\bm{M}}^*_{2}/m_2
    \end{pmatrix}.
\end{equation*}
\label{lemma:simplify_lambda_unequal}
Then, $\bm{\Lambda} \succeq 0$ and $\mathcal{N}(\bm{\Lambda}) = \mathcal{R}(\bm{M}^*)$ are satisfied if $\widetilde{\bm{\Lambda}} \succ 0$.
\end{lemma}

Different from Step 2 in Section~\ref{sec:sketch_proof_equal_main}, here, we adopt a new way to construct $\boldsymbol{\Theta_{ij}}$ 
in order to obtain a threshold close to the empirical phase transition line by tuning $\gamma$. 
The parameter $\gamma$ controls the values of $\nu$,  $\mu_i$, and the diagonal entries of $\bm{Z}_i$.  
As we shall see in the next step,  by adjusting $\gamma$, we can balance the diagonal values of $\bm{Z}_i$ for $i \in C_1$ and $i \in C_2$, which helps to ensure  $\widetilde{\bm{\Lambda}} \succ 0$.
We remark that our current construction of the dual variables may still be sub-optimal, which indicates the condition in \eqref{eq:cond_unequal_two_case} is sufficient but not necessary. Finding the optimal dual appears to be challenging and is left for future work.

\vspace{0.1cm}
\paragraph{Step 3: Find the condition for $\widetilde{\bm{\Lambda}} \succ 0$} 
Again, by using the same argument as in Section~\ref{sec:sketch_proof_equal_main}, one can see that for exact recovery it suffices to show $\lambda_{\text{min}}(p\bm{I}_{nd} - \bm{Z}) > \|\mathbb{E}[\bm{A}] - \bm{\bm{A}}\|$. Here we get the following bound for $\lambda_{\text{min}}(p\bm{I}_{nd} - \bm{Z})$.
\begin{lemma}
Let $p = \alpha \log n/n$,  $q = \beta \log n/n$ and $\gamma \in [1/2, 1]$. Suppose $m_1 = \rho n$ and $m_2 = (1-\rho)n$ for some $\rho \in (1/2, 1)$. $\delta_1$ and $\delta_2$ are defined in Theorem~\ref{the:2_cluster_unequal_d_2}. We have
\begin{equation*}
    \lambda_{\text{min}}(p\bm{I}_{nd} - \bm{Z}) \geq \min \{\rho(\tau_1 - \max\{\gamma\delta_1, \delta_2\}), \; (1-\rho)(\tau_2 - 2(1-\gamma)\delta_2)\} \log n - o(\log n)
\end{equation*}
with probability $1 - n^{-\Omega(1)}$ for some $\tau_1, \tau_2 \in [0, \alpha]$ such that 
\begin{align}
    1-\rho\left(\alpha - \tau_1 \log \left(\frac{e\alpha}{\tau_1} \right) \right) < 0, \quad 
    1-(1- \rho)\left(\alpha - \tau_2 \log \left(\frac{e\alpha}{\tau_2} \right) \right) < 0.
    \label{eq:cond_tau_unequal} 
\end{align}
\label{lemma:bound_nu_mu}
\end{lemma}

\begin{proof}[Proof of Theorem~\ref{the:2_cluster_unequal_d_2}]
From Lemma~\ref{lemma:bound_nu_mu}, we get $\lambda_{\text{min}}(p\bm{I}_{nd} - \bm{Z}) = O(\log n)$. From Lemma~\ref{lemma:concentration_A} we have $\|\mathbb{E}[\bm{\bm{A}}] - \bm{A}\|\!=\!o(\sqrt{\log n})$. Therefore, as $n$ is large, the condition $\lambda_{\text{min}}(p\bm{I}_{nd} - \bm{Z}) \!>\!\|\mathbb{E}[\bm{\bm{A}}]\!-\!\bm{A}\|$ is satisfied if $\lambda_{\text{min}}(p\bm{I}_{nd}\!-\!\bm{Z})\!=\!\Omega(\log n)$, which is equivalent to 
\begin{equation}
    \min \{\rho(\tau_1 - \max\{\gamma\delta_1, \delta_2\}), \; (1-\rho)(\tau_2 - 2(1-\gamma)\delta_2)\} > 0.
    \label{eq:bound_tau_1_2} 
\end{equation}
By defining $\bar{\tau}_1 := \max\{\gamma \delta_1, \delta_2\}, \bar{\tau}_2 := 2(1-\gamma)\delta_2$, \eqref{eq:bound_tau_1_2} reduces to $\tau_1 > \bar{\tau}_1, \tau_2 > \bar{\tau}_2$. 
Now it remains to check if  \eqref{eq:cond_tau_unequal} holds for some $\tau_1, \tau_2 \in [0, \alpha)$. To this end, let $\tau_1^*, \tau_2^* \in [0, \alpha)$ be the roots of the equations in~\eqref{eq:tau_1_2_star} (such $\tau_1^*, \tau_2^*$ exist only if $\alpha > 1/(1-\rho)$).
Since $f(\tau)\!:=\!\alpha\!-\!\tau\log\left(e\alpha/\tau\right)$ monotonically decreases for $\tau \in [0, \alpha)$, then \eqref{eq:cond_tau_unequal} holds only if $\tau_1 < \tau_1^*, \tau_2 < \tau_2^*$. Given the above, the exact recovery conditions are $\alpha > 1/(1-\rho), \bar{\tau}_1 < \tau_1 < \tau_1^*$ and $\bar{\tau}_2 < \tau_2 < \tau_2^*$. Moreover, since $\gamma$ is not specified, it suffices to ensure  $\bar{\tau}_1 < \tau_1^*, \bar{\tau}_2 < \tau_2^*$ for some $\gamma \in [1/2, 1]$ such that the range of $\tau_1, \tau_2$ is valid. This leads to the condition in \eqref{eq:cond_unequal_two_case}. 
\end{proof}

\subsection{Two clusters with unknown cluster size}
\label{sec:rounding}
In some real applications where the cluster sizes is absent, an alternative approach that does not rely on such prior knowledge is necessary. To this end, we consider a normalized ground truth $\bar{\bm{M}}^*$, whose $(i,j)$-th block is 
\begin{equation}
    \bar{\bm{M}}^*_{ij} = 
    \begin{cases}
    \frac{1}{m_{\kappa(i)}}\bm{R}_i\bm{R}_j^\top,  &\;  \kappa(i) = \kappa(j), \\
    \bm{0}, &\; \text{otherwise}.  
    \end{cases}
    \label{eq:ground_truth_def_normalize}
\end{equation}
Here $m_{\kappa(i)}$ denotes the size of the cluster that $i$ belongs to. By assuming the first $m_1$ nodes form $C_1$ and the remaining $m_2$ nodes form $C_2$, we have 
\begin{equation}
    \bar{\bm{M}}^* :
    = \frac{1}{m_1}\bm{V}^{(1)}(\bm{V}^{(1)})^{\top} + \frac{1}{m_2}\bm{V}^{(2)}(\bm{V}^{(2)})^{\top} = 
    \begin{pmatrix}
    \frac{1}{m_1}\widetilde{\bm{M}}_{1}^* &\bm{0}\\
    \bm{0} &\frac{1}{m_2}\widetilde{\bm{M}}_{2}^* 
    \end{pmatrix},
    \label{eq:M_ground_truth_unknown}
\end{equation}
where each diagonal block has been normalized by the corresponding cluster size. Then $\bar{\bm{M}}^*$ satisfies the following structural properties: 
\begin{enumerate}
    \vspace{0.05cm}
    \item $\bar{\bm{M}}^*$ is positive semi-definite, i.e. $\bar{\bm{M}}^* \succeq 0$;
    \vspace{0.05cm}
    \item The sum of the diagonal blocks satisfies $\sum_{i}\bar{\bm{M}}^*_{ii} = 2\bm{I}_d$;
    \vspace{0.05cm}
    \item The $i$-th block row of $\bar{\bm{M}}^*$ satisfies $\sum_{j = 1}^n \|\bar{\bm{M}}^*_{ij}\|_\mathrm{F} \leq \sqrt{d}$, for $i = 1,\ldots, n$. 
    \vspace{0.05cm}
\end{enumerate}
Based on these, we formulate the SDP for solving $\bar{\bm{M}}^*$ as
\begin{equation}
    \begin{aligned}
    \bm{M}_{\text{SDP}} &= \argmax_{\bm{M}} \quad \left\langle \bm{A}, \bm{M} \right\rangle\\
    \mathrm{s.t.} &\quad \bm{M} \succeq 0, \quad \sum_{i = 1}^n\bm{M}_{ii} = 2\bm{I}_d, \quad \sum_{j=1}^n\|\bm{M}_{ij}\|_\mathrm{F} \leq \sqrt{d}, \quad i = 1,\ldots, n.
    \end{aligned}
    \label{eq:SDP_unknown}
\end{equation}
Due to the normalization factor, \eqref{eq:SDP_unknown} differs from the previous SDPs that it does not rely on the cluster sizes, and exact recovery is achieved if the normalized $\bar{\bm{M}}^*$ is obtained. Besides, similar to \eqref{eq:SDP_unequal_known} a sum of all blocks constraint $\sum_{i, j = 1}^n \| \bm{M}_{ij}\|_F \leq n \sqrt{d}$ could apply to \eqref{eq:SDP_unknown}, via the row sum constraint that $\sum_{j = 1}^n \| \bm{M}_{ij}\|_\mathrm{F} \leq  \sqrt{d}$ for all $i$ and is thus redundant.
Notably, in the context of community detection, similar normalization has been studied~\cite{peng2007approximating, yan2017provable, perry2017semidefinite} when cluster sizes are unknown, where the row-sum constraint is formulated as $\sum_{j = 1}^n \bm{M}_{ij} = 1$. In contrast, we use an inequality constraint for the sake of convexity. Here, we obtain the following exact recovery guarantees for \eqref{eq:SDP_unknown}.
\begin{theorem}
Let $p = \alpha \log n/n, q = \beta \log n/n$. Suppose $m_1 = \rho n$ and $m_2 = (1-\rho)n$ for some $\rho \in (0, 1)$. $\tau_1^*, \tau_2^*$ are defined in \eqref{eq:tau_1_2_star}. As $n$ is sufficiently large, \eqref{eq:SDP_unknown} achieves exact recovery i.e. $\bm{M}_{\text{SDP}} = \bar{\bm{M}}^*$ with probability $1 - n^{-\Omega(1)}$ if
\begin{equation}
     \max\bigg\{\frac{1}{\rho}, \; \frac{1}{1-\rho}\bigg\} < \alpha \leq \min\bigg\{2\tau_1^*,\; 2\tau_2^*, \; \tau_1^* + \tau_2^* - \bigg(\frac{1}{\sqrt{\rho}} + \frac{1}{\sqrt{1-\rho}}\bigg)\sqrt{\ell_d \beta}\bigg\},
    \label{eq:unknown_cond_theorem}
\end{equation}
where $\ell_d = 2$ when $d > 2$ and $\ell_d = 1$ when $d = 2$. 
\label{the:2_cluster_unknown_d_2}
\end{theorem}
The proof of Theorem~\ref{the:2_cluster_unknown_d_2} is deferred to supplementary material Section~\ref{sec:proof_two_unknown}, as it follows a similar structure as the one presented in Section~\ref{sec:sketch_proof_equal_main}, with a slight difference on the construction of dual variables.

\vspace{0.1cm}
\begin{figure}[t!]
    \vspace{-0.55cm}
	\centering
	\scriptsize
	\subfloat[$|\bar{\bm{M}}^*|$]{\includegraphics[width = 0.24\textwidth]{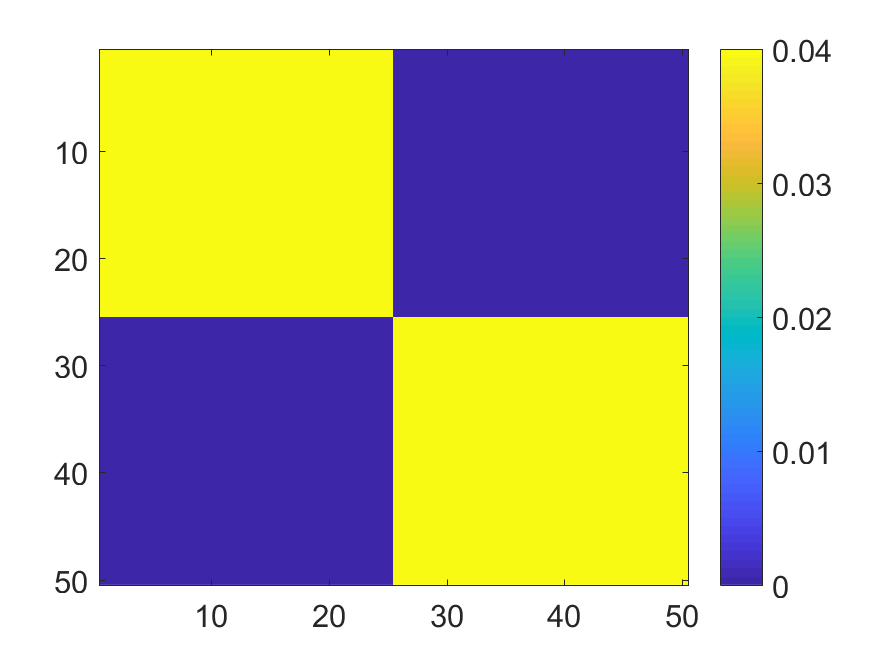}\label{fig:unknown_example_a}}
    \subfloat[$|\bm{A}|$]{\includegraphics[width = 0.24\textwidth]{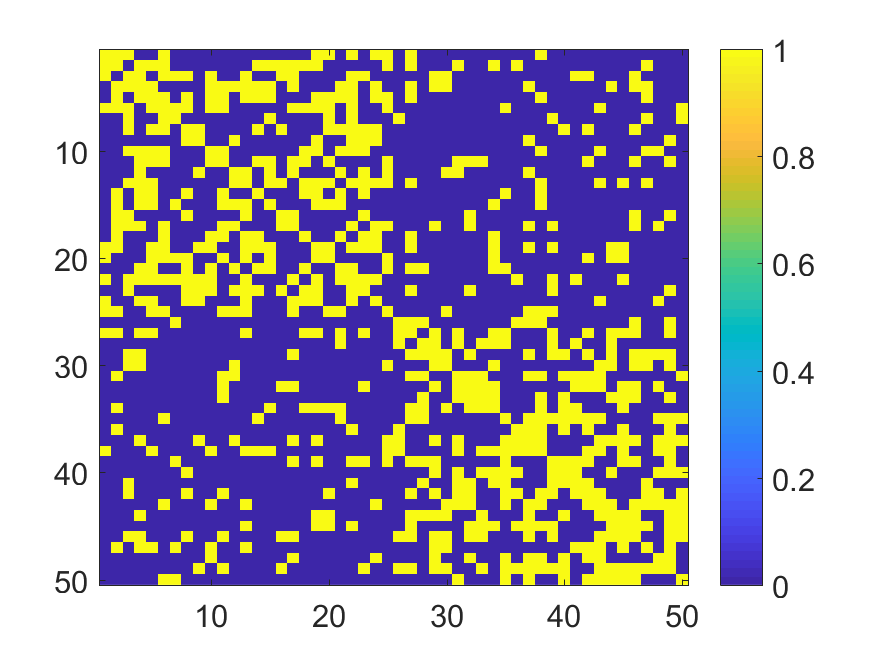}}
    \subfloat[$|\bm{M}_{\text{SDP}}|$]{\includegraphics[width = 0.24\textwidth]{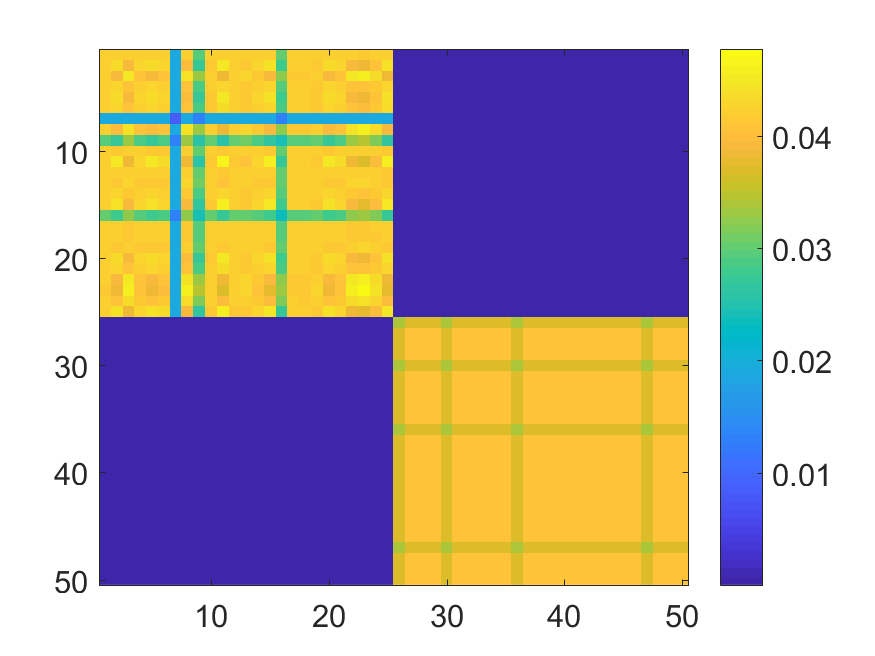} \label{fig:unknown_example_c}}
    \subfloat[{$|\bm{M}_{\text{round}}|$}]{\includegraphics[width = 0.24\textwidth]{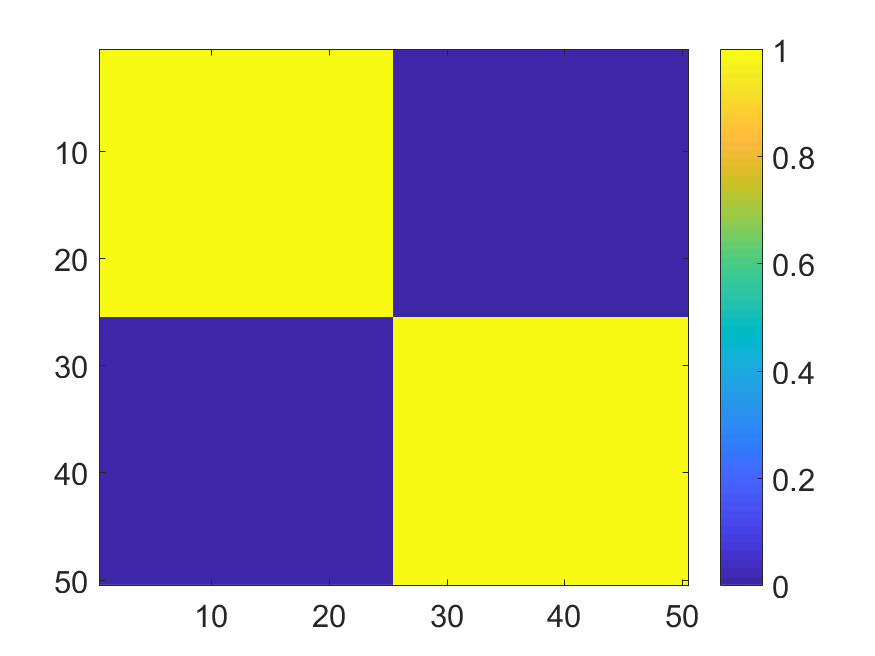}\label{fig:unknown_example_d}}\\[-0.4cm]
    \subfloat[$\angle\bar{\bm{M}}^*$]{\includegraphics[width = 0.24\textwidth]{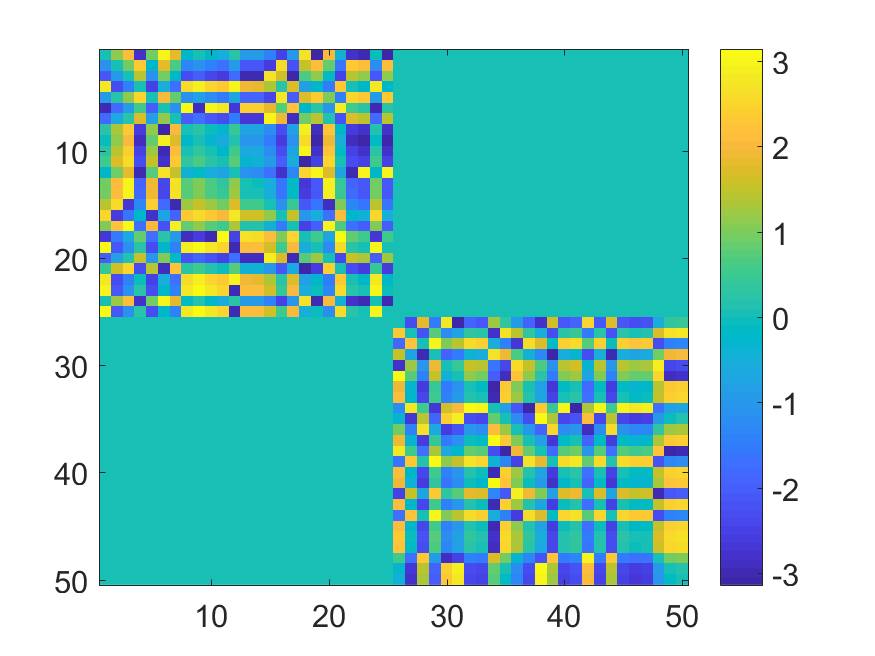}\label{fig:unknown_example_e}}
    \subfloat[$\angle\bm{A}$]{\includegraphics[width = 0.24\textwidth]{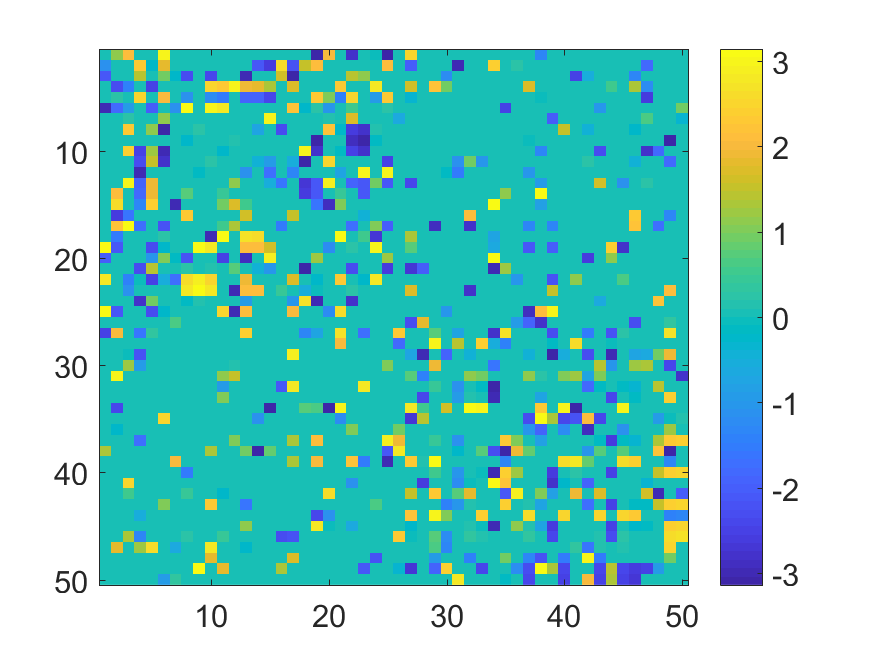}\label{fig:unknown_example_f}}
    \subfloat[$\angle\bm{M}_{\text{SDP}}$]{\includegraphics[width = 0.24\textwidth]{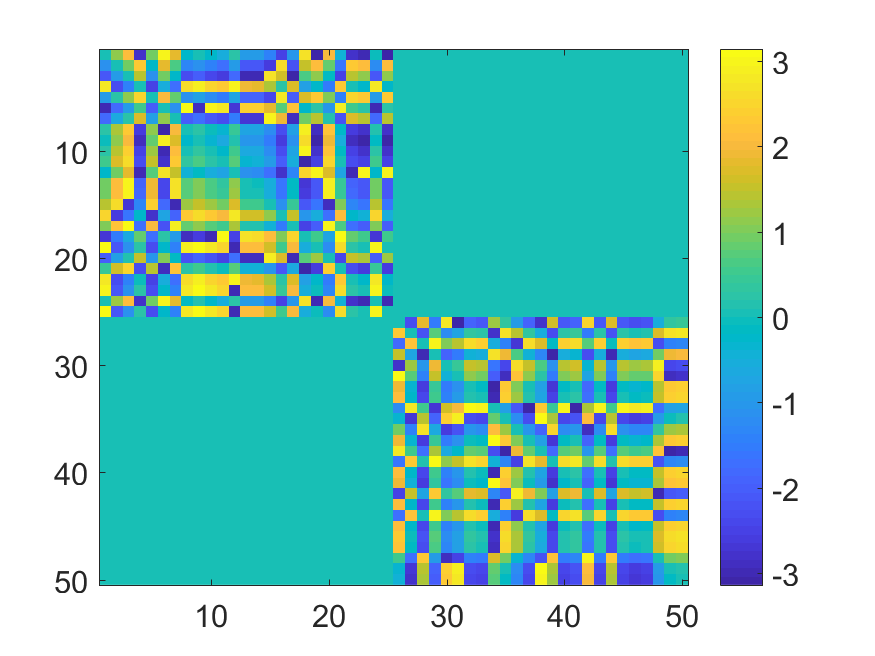}\label{fig:unknown_example_g}}
    \subfloat[{$\angle\bm{M}_{\text{round}}$}]{\includegraphics[width = 0.24\textwidth]{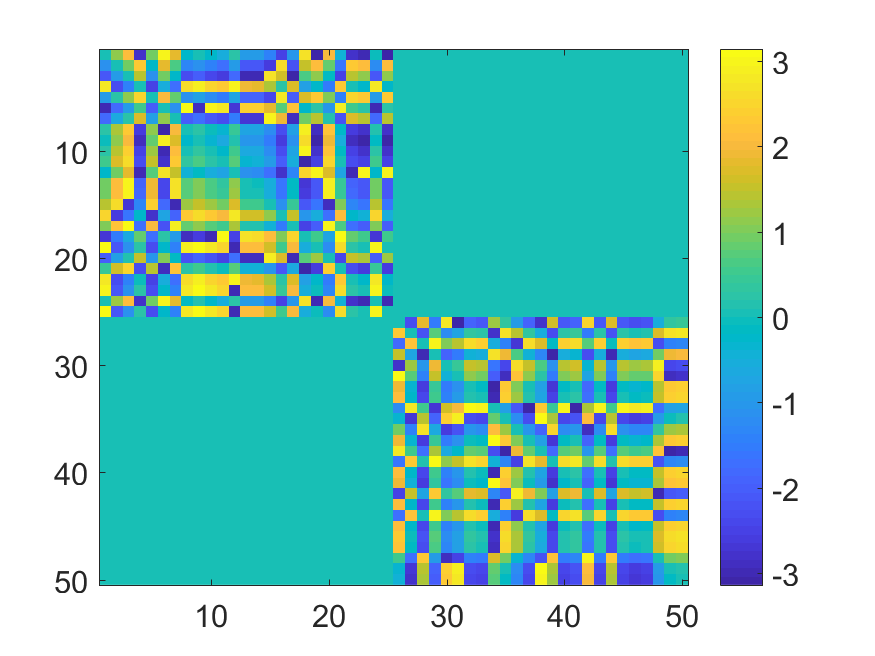}\label{fig:unknown_example_h}}
    \vspace{-0.25cm}
	\caption{\textit{An example of rounding when $d = 2$}.
	We plot the matrix of amplitudes (\protect \subref{fig:unknown_example_a}--\protect \subref{fig:unknown_example_d}) and rotation angles (\protect \subref{fig:unknown_example_e}--\protect \subref{fig:unknown_example_h}) defined in \eqref{eq:M_ij_complex} for the ground truth $\bar{\bm{M}}^*$, the observation matrix $\bm{A}$, the SDP result $\bm{M}_{\text{SDP}}$ by \eqref{eq:SDP_unknown} and the rounding result $\bm{M}_{\text{round}}$.} 
    \label{fig:unknown_example}
\end{figure}

\paragraph{Rounding}
Besides, we observe that an extra rounding step can improve the recovery by \eqref{eq:SDP_unknown}. As an illustration, in Figure~\ref{fig:unknown_example} we present an example with 2D rotation and two equal-sized clusters. 
For convenience, we use a complex number to represent each $\bar{\bm{M}}_{ij}^*$ such that $    \bar{\bm{M}}_{ij}^* = |\bar{\bm{M}}^*_{ij}| e^{\imath \theta_{ij}}$,
where $|\bar{\bm{M}}^*_{ij}|$ denotes the ``amplitude'' and $\theta_{ij} \in [0, 2\pi]$ represents the rotation angle.
Then we form two $n \times n$ matrices $|\bar{\bm{M}}^*|$ and $\angle \bar{\bm{M}}^*$ that contain  
\begin{equation}
\text{amplitude: }    |\bar{\bm{M}}^*|_{ij} := |\bar{\bm{M}}^*_{ij}|, \quad \text{and rotation angle } \angle \bar{\bm{M}}^*_{ij} := \theta_{ij} - \pi
\label{eq:M_ij_complex}
\end{equation}
respectively. A similar procedure is applied to $\bm{A}$ and $\bm{M}_{\textrm{SDP}}$. In Figures~\ref{fig:unknown_example_a} and~\ref{fig:unknown_example_c}, one can see that $|\bm{M}_{\text{SDP}}| \neq |\bar{\bm{M}}^*|$ and thus exact recovery fails. However, they have the same support on the two clusters. Moreover, in Figures~\ref{fig:unknown_example_e} and~\ref{fig:unknown_example_g} we observe $\angle \bm{M}_{\textrm{SDP}} = \angle \bar{\bm{M}}^*$, meaning the rotational alignments are exactly recovered. These observations inspire us to design the following rounding procedure. For each block of $\bm{M}_{\textrm{SDP}}$ denoted by $(\bm{M}_{\textrm{SDP}})_{ij}$, we round it to
\begin{equation}
    (\bm{M}_{\textrm{round}})_{ij} = 
    \begin{cases}
    \displaystyle \frac{(\bm{M}_{\textrm{SDP}})_{ij}}{\|(\bm{M}_{\textrm{SDP}})_{ij}\|_\text{F}} \cdot \sqrt{d}, &\quad \|(\bm{M}_{\textrm{SDP}})_{ij}\|_\textrm{F} \geq  \epsilon, \\
    \bm{0}, &\quad \text{otherwise},
    \end{cases}
    \label{eq:rounding_formula}
\end{equation}
where $\epsilon$ is some small value for avoiding numerical issues. That is, for each non-zero block of $\bm{M}_{\textrm{SDP}}$ we normalize its ``amplitude'' (Frobenius norm) to be $\sqrt{d}$ and keep its ``phase'' unchanged. Let $\bm{M}_{\textrm{round}}$ be the rounded solution based on all blocks of $\bm{M}_{\textrm{SDP}}$, then in Figures~\ref{fig:unknown_example_d} and \ref{fig:unknown_example_h} we obtain $\bm{M}_{\textrm{round}} = \bm{M}^*$ which is equal to the unnormalized ground truth in \eqref{eq:ground_truth_def}. Therefore, empirically, we observe exact recovery after an appropriate rounding.

\subsection{SDPs for general cluster structures}
\label{sec:kcluster}
Our SDP formulation can be easily extended to general cases with an arbitrary number of clusters. When the sizes of $K$ clusters are known, the convex relaxations introduced in \eqref{eq:SDP_unequal_known} still apply to $\bm{M}^*$ in principle, and we can formulate an SDP relaxation as
\begin{equation}
    \begin{aligned}
    \bm{M}_{\text{SDP}} &= \argmax_{\bm{M}} \quad \left\langle \bm{A}, \bm{M} \right\rangle\\
    \mathrm{s.t.} &\quad \bm{M} \succeq 0, \quad \bm{M}_{ii} = \bm{I}_d, \quad \sum_{j}\|\bm{M}_{ij}\|_\mathrm{F} \leq (\max_i m_i)\sqrt{d}, \quad i = 1,\ldots, n,\\
    &\quad \sum_{i,j = 1}^{n}\|\bm{M}_{ij}\|_\mathrm{F} \leq \sqrt{d} \sum_{k = 1}^K m_k^2.
    \end{aligned}
    \label{eq:SDP_general_known}
\end{equation}
When the cluster sizes are unknown, an SDP similar to \eqref{eq:SDP_unknown} can be designed as:
\begin{equation}
    \begin{aligned}
    \bm{M}_{\text{SDP}} &= \argmax_{\bm{M}} \quad \left\langle \bm{A}, \bm{M} \right\rangle\\
    \mathrm{s.t.} &\quad \bm{M} \succeq 0, \quad \sum_{i}\bm{M}_{ii} = K\bm{I}_d, \quad \sum_{j}\|\bm{M}_{ij}\|_\mathrm{F} \leq \sqrt{d}, \quad i = 1,\ldots, n,
    \end{aligned}
    \label{eq:SDP_unequaK}
\end{equation}
which only requires the number of clusters $K$ as a parameter. 
Also, the rounding step \eqref{eq:rounding_formula} can be used to improve the recovery by \eqref{eq:SDP_unequaK}. Establishing the exact recovery guarantee for both \eqref{eq:SDP_general_known} and \eqref{eq:SDP_unequaK} is an interesting but challenging task, which is left for future work.

\begin{figure}[t!]
    \centering
    \vspace{-0.3cm}
    \captionsetup[subfloat]{farskip=1pt,captionskip=0pt}
    \subfloat[\scriptsize{$n = 50$}]{\includegraphics[width = 0.30\textwidth]{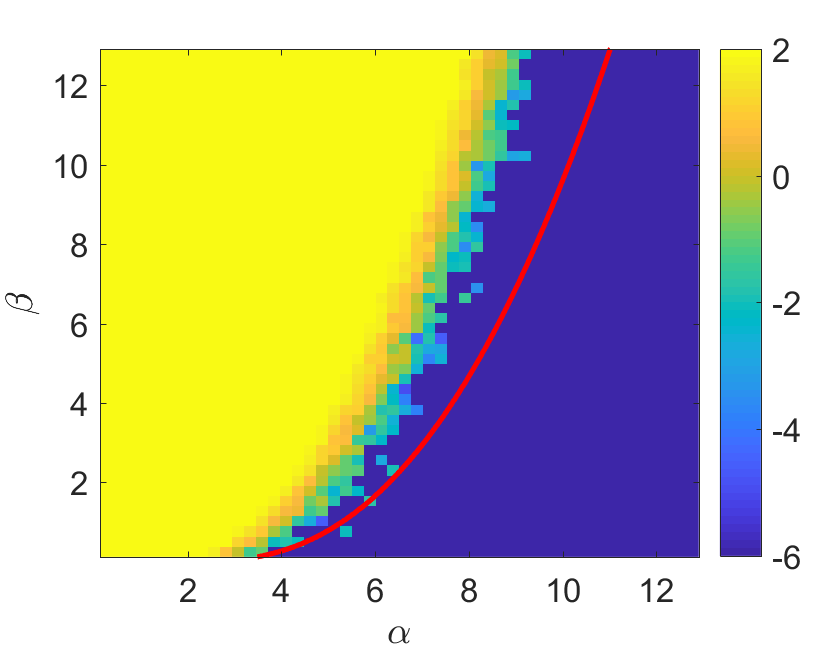}\label{fig:two_equal_a}}
    \subfloat[\scriptsize{$n = 100$}]{\includegraphics[width = 0.30\textwidth]{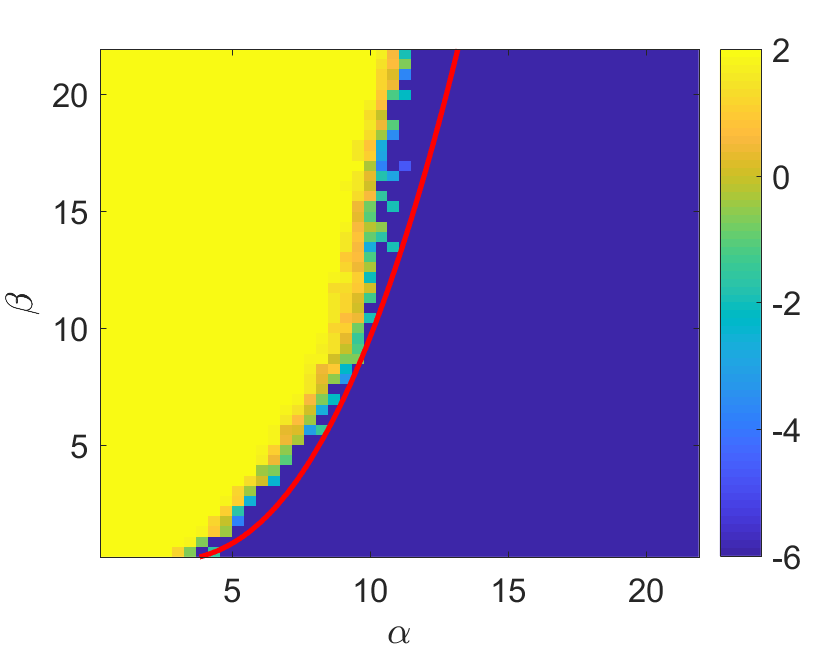}}
    \subfloat[\scriptsize{$n = 150$}]{\includegraphics[width = 0.30\textwidth]{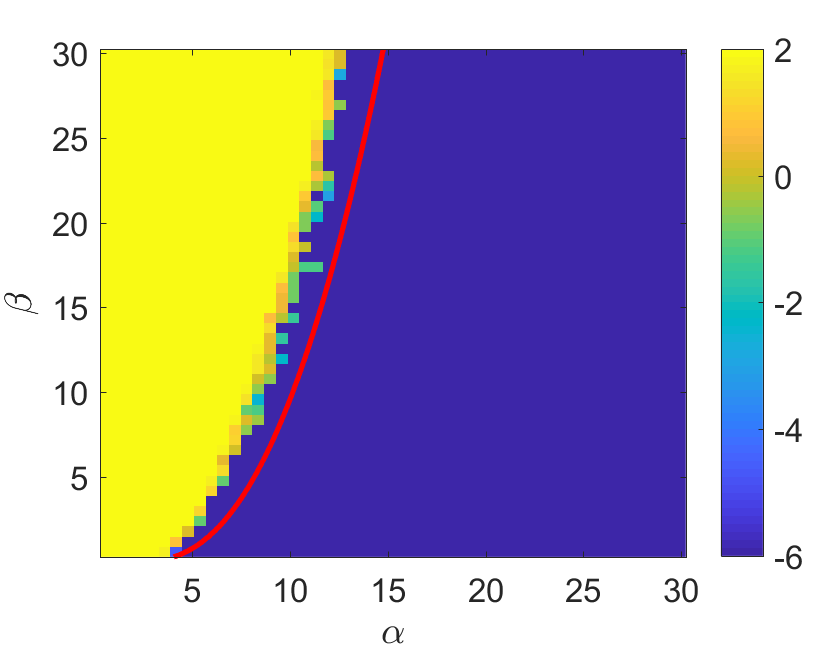}\label{fig:two_equal_c}}\\[2pt]
    \subfloat[\scriptsize{$n = 150$, {by \cite[Eq.~(5)]{hajek2016achievinga}}}]{\includegraphics[width = 0.30\textwidth]{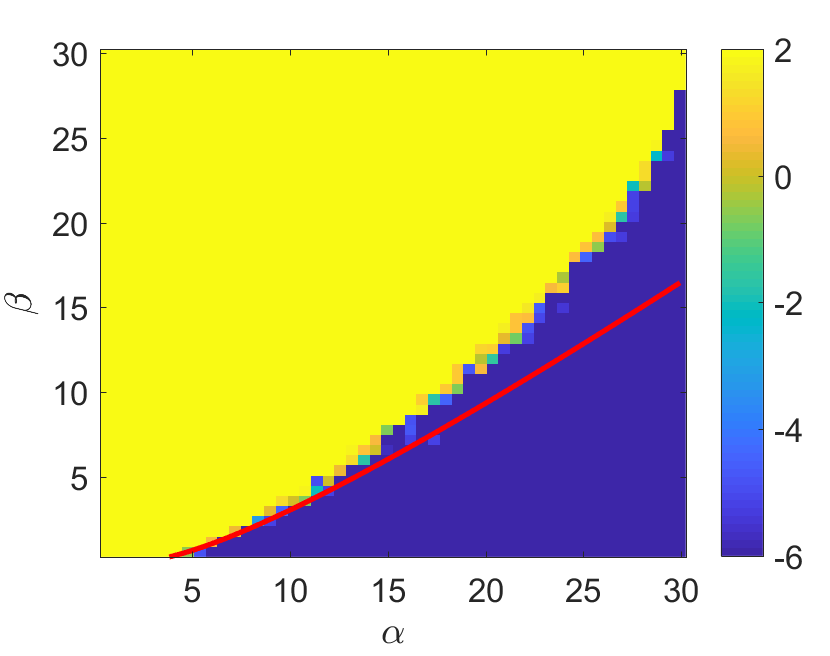}\label{fig:original_SDP_equal}
    \label{fig:sbm_sdp}
    }
    \subfloat[\scriptsize{Comparison of two thresholds}]{\includegraphics[width = 0.30\textwidth]{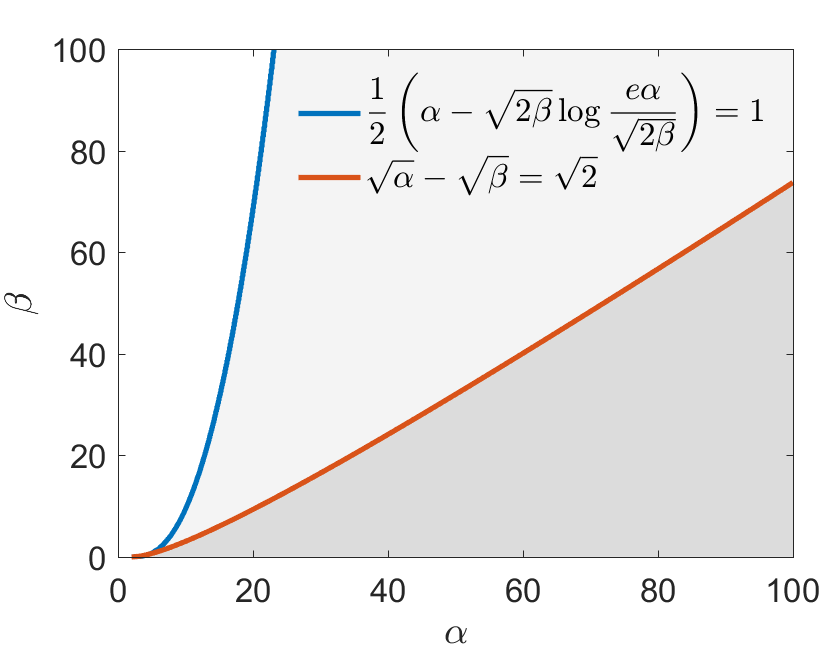}\label{fig:equal_original_phase}
    \label{fig:compare_thresholds}
    }
    \vspace{-0.1cm}
    \caption{\textit{Results on two equal-sized clusters with known cluster size.} \protect \subref{fig:two_equal_a}--\protect \subref{fig:two_equal_c} Recovery error by \eqref{eq:SDP_equal} with different $n$. The threshold in Theorem~\ref{the:2_cluster_equal_d_2} for $d = 2$ is shown in red. 
    \protect \subref{fig:sbm_sdp} Recovery error by first performing clustering on graph scalar edge connection using the SDP in~\cite[Eq.~(5)]{hajek2016achievinga} and then synchronization, the exact recovery threshold $\sqrt{\alpha} - \sqrt{\beta} = \sqrt{2}$ (see \cite[Theorem 2]{hajek2016achievinga}) is shown in red. \protect \subref{fig:compare_thresholds} Comparison of the thresholds in Theorem~\ref{the:2_cluster_equal_d_2} for $d = 2$ (blue) and~\cite[Theorem 2]{hajek2016achievinga} (red), exact recovery is guaranteed in the region below each threshold. The recovery error is defined in \eqref{eq:relative_error}.} 
    \label{fig:two_equal_unequal_exp}
    \vspace{-0.5cm}
\end{figure}

\section{Numerical results}
\label{sec:exp}

This section is devoted to empirically investigating the performance of the SDPs in Section~\ref{sec:method}. 
All experiments\footnote{The code is available in \url{https://github.com/frankfyf/joint_cluster_sync_SDP}.} are performed in MATLAB on a machine with 60 Intel Xeon CPU cores, running at 2.3GHz with 512GB RAM in total, and only one core is used for each experiment. We use MOSEK~\cite{aps2019mosek} in MATLAB as the SDP solver. For solving an SDP with $n = 100$ and $n = 150$, it takes 5 mins and 30 mins in average respectively.

In Sections~\ref{sec:exp_2D_equal}--\ref{sec:exp_2D_unknown}, we focus on experiments with two clusters and 2D rotations ($d = 2$). In Section~\ref{sec:exp_general}, we include results for more general settings, e.g., 3D rotations and the number of clusters is larger than two. We generate the observation matrix $\bm{A}$ based on the model introduced in Section~\ref{sec:model}. Then we evaluate the SDP solution by measuring the following error metric:
\begin{equation}
    \text{Error} = \log(\|\bm{M}_{\text{SDP}} - \bm{M}^*\|_\textrm{F}).
    \label{eq:relative_error}
\end{equation}
When cluster sizes are unknown and the SDP in \eqref{eq:SDP_unknown} is adopted, the ground truth $\bm{M}^*$ in \eqref{eq:relative_error} should be replaced by the normalized $\bar{\bm{M}}^*$ defined in \eqref{eq:ground_truth_def_normalize}.
For the case of two clusters, we evaluate the phase transition of the condition $\widetilde{\bm{\Lambda}} \succ 0$ according to our guess of dual variables (e.g. in Lemmas~\ref{lemma:guess_dual} and \ref{lemma:guess_dual_unequal_new}) via the following metric:
\begin{equation}
    \text{Failure rate} = 1 - \text{the rate that } \widetilde{\bm{\Lambda}} \succ 0 \text{ is satisfied}.
    \label{eq:failure_rate} 
\end{equation}
$\text{Failure rate} = 0$ implies our dual certificates satisfy the optimality and uniqueness condition for $\bm{M}^*$ or $\bar{\bm{M}}^*$ (e.g. in Lemmas~\ref{lemma:1_main} and \ref{lemma:uniqueness_unequal_new}) and exact recovery is guaranteed. 
For each experiment, we average the result over 10 realizations of the observation matrix $\bm{A}$.

\subsection{Two equal-sized clusters}
\label{sec:exp_2D_equal}
First, we consider two equal-sized clusters where the SDP in \eqref{eq:SDP_equal} is used for recovery. In Figures~\ref{fig:two_equal_a}--\ref{fig:two_equal_c} we plot the recovery error with varying $\alpha$ and $\beta$ for $n = 50, \, 100, \, 150$. In each plot, the red solid curve corresponds to the threshold of the condition in Theorem~\ref{the:2_cluster_equal_d_2} for $d = 2$ . As we can see, \eqref{eq:SDP_equal} achieves exact recovery within a wide range of $\alpha$ and $\beta$. Also, the theoretical threshold closely fits the empirical phase transition boundary, especially as $n$ is large, which indicates the sharpness of the condition. 

\begin{figure}[t!]
    \centering
    \vspace{-0.7cm}
    \setlength\tabcolsep{1.0pt}
	\renewcommand{\arraystretch}{0.6}
	\begin{tabular}{p{0.35\textwidth}<{\centering} p{0.35\textwidth}<{\centering}}
	\subfloat[\scriptsize{Error of $\bm{M}_{\text{SDP}}$, $(m_1,\!m_2) \!=\! (60,\!40)$}]{\includegraphics[width = 0.32\textwidth]{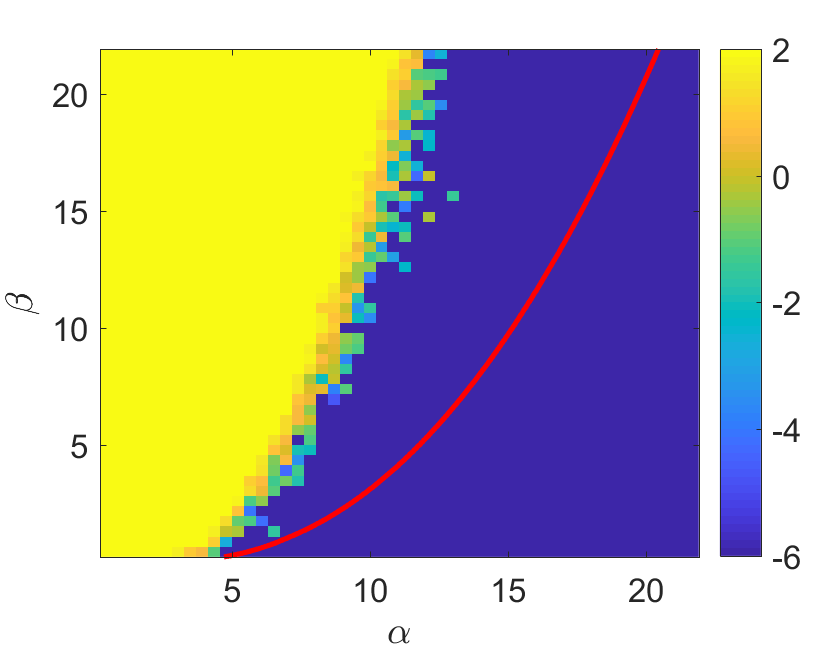}\label{fig:unequal_error_a}}
    &\subfloat[\scriptsize{Failure rate, $(m_1,\!m_2) \!=\! (60,\! 40)$}]{\includegraphics[width = 0.32\textwidth]{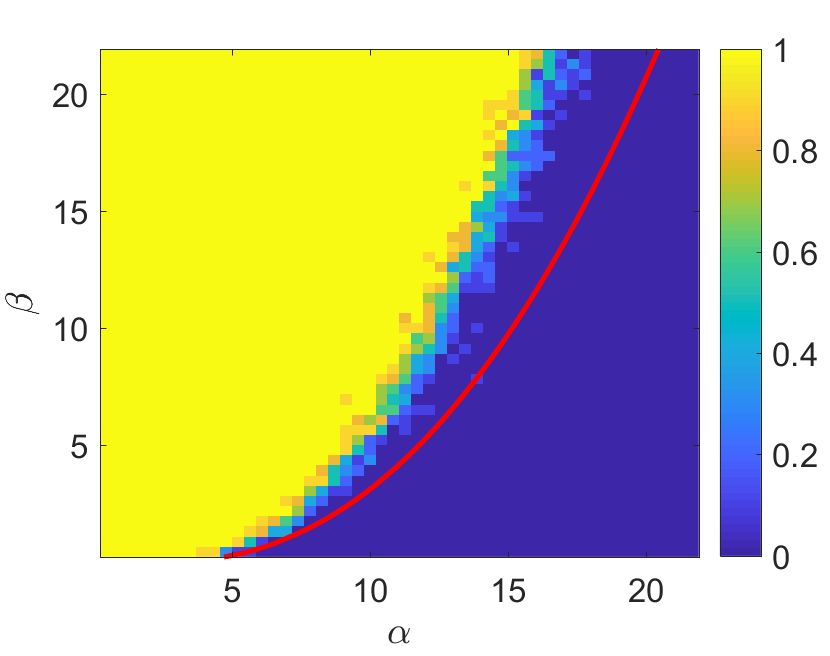}\label{fig:unequal_rate_b}}\\[-5pt]
    \subfloat[\scriptsize{Error of $\bm{M}_{\text{SDP}}$, $(m_1,\!m_2) \!=\! (80,\!20)$}]{\includegraphics[width = 0.32\textwidth]{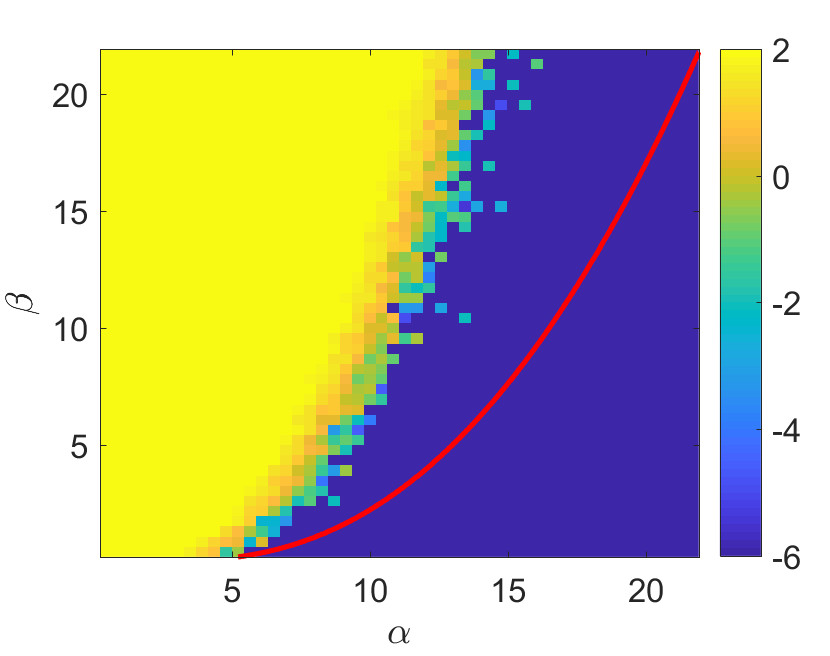}\label{fig:unequal_error_c}}
    &\subfloat[\scriptsize{Failure rate, $(m_1,\!m_2) \!=\! (80,\! 20)$}]{\includegraphics[width = 0.32\textwidth]{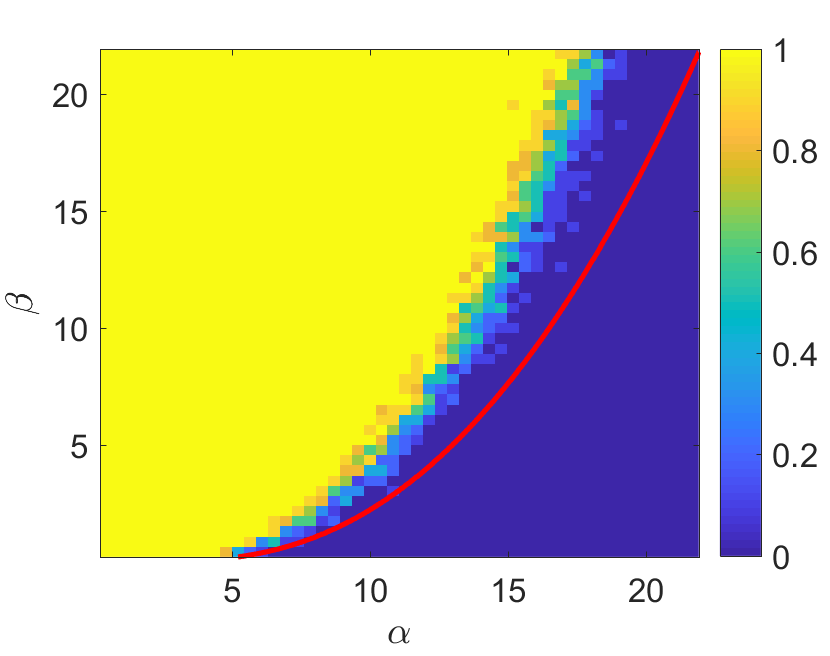}\label{fig:unequal_rate_d}}
	\end{tabular}
    \vspace{-0.2cm}
    \caption{\textit{Results on two unequal-sized clusters with known cluster sizes.} 
    \protect \subref{fig:unequal_error_a} and \protect \subref{fig:unequal_error_c}: Recovery error of the solution of \eqref{eq:SDP_unequal_known}; \protect \subref{fig:unequal_rate_b} and \protect \subref{fig:unequal_rate_d}: The failure rate defined in \eqref{eq:failure_rate} for $\widetilde{\bm{\Lambda}}$ in Lemma~\ref{lemma:simplify_lambda_unequal}. The threshold in Theorem~\ref{the:2_cluster_unequal_d_2} for $d = 2$ is shown in red. }
    \label{fig:two_unequal_exp}
\end{figure}

\begin{figure}[t!]
    \centering
    \vspace{-0.5cm}
    \setlength\tabcolsep{1.0pt}
	\renewcommand{\arraystretch}{0.6}
	\begin{tabular}{p{0.35\textwidth}<{\centering} p{0.35\textwidth}<{\centering}}
    \subfloat[\scriptsize{$(m_1, m_2) = (600, 400)$}]{\includegraphics[width = 0.32\textwidth]{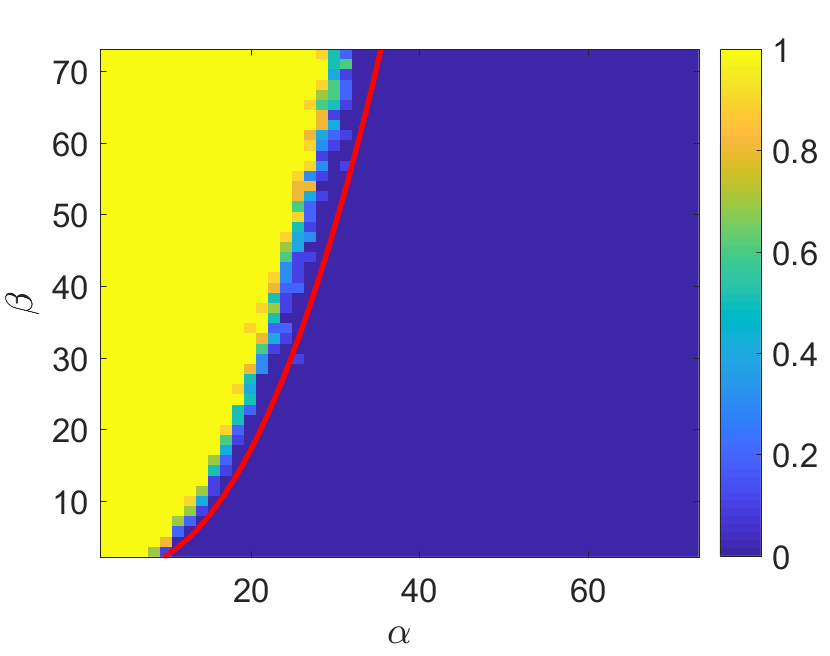}\label{fig:unequal_dual_large_a}}
    \hspace{0.50cm}
    &\subfloat[\scriptsize{$(m_1, m_2) = (800, 200)$}]{\includegraphics[width = 0.32\textwidth]{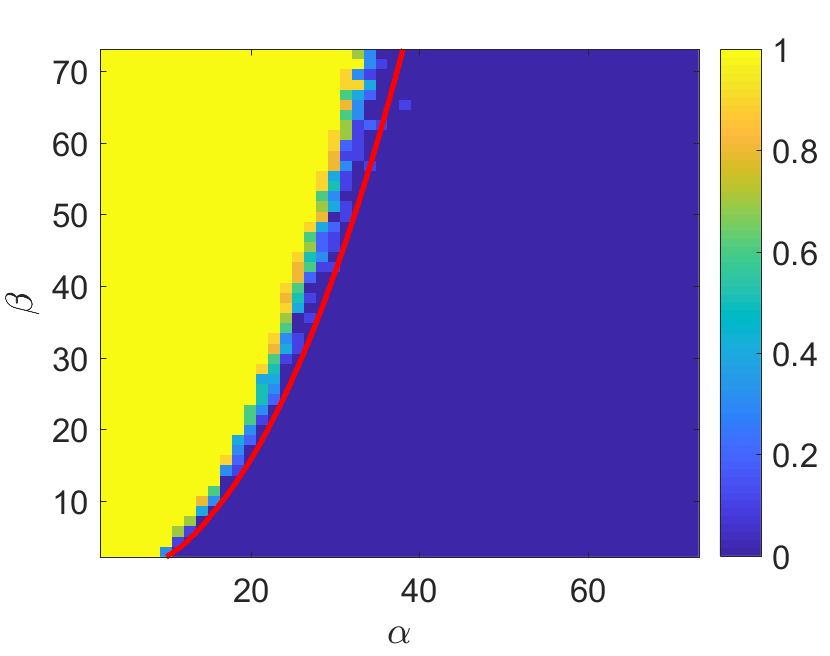}\label{fig:unequal_dual_large_b}}
	\end{tabular}
    \vspace{-0.2cm}
    \caption{\textit{Results on two unequal-sized clusters with known cluster sizes, large $n$}. We plot the failure rate \eqref{eq:failure_rate} for $\widetilde{\bm{\Lambda}}$ in Lemma~\ref{lemma:simplify_lambda_unequal}, and the threshold in Theorem~\ref{the:2_cluster_unequal_d_2} at $d = 2$ in red. }
    \label{fig:two_unequal_exp_large_n}
    \vspace{-0.5cm}
\end{figure}


In addition, we consider using the two-stage approach mentioned in the introduction that first applies existing methods for community detection and then performs synchronization for each community separately. We use the SDP proposed in~\cite[Eq.~(5)]{hajek2016achievinga} to detect the cluster membership, which is proved to achieve the information-theoretical limit as $n \rightarrow \infty$ \cite{abbe2015exact,hajek2016achievinga}. As a result, Figure~\ref{fig:original_SDP_equal} displays the corresponding recovery error with $n = 150$, which exhibits worse performance than Figure~\ref{fig:two_equal_c}. In Figure~\ref{fig:equal_original_phase}, we compare the theoretical thresholds for exact recovery in Theorem~\ref{the:2_cluster_equal_d_2} for $d = 2$ and \cite[Theorem 2]{hajek2016achievinga}. It is clear to see that, by incorporating the additional pairwise alignment information, our method outperforms the classical community detection method which only exploits the edge connectivity.

\subsection{Two unequal-sized clusters}
\label{sec:exp_2D_unequal}
Next, we consider two unequal-sized clusters where the SDP in \eqref{eq:SDP_unequal_known} is used for recovery. 
We simulate $\bm{A}$ under different settings of $(m_1,m_2)$, and plot the recovery error with varying $\alpha$ and $\beta$ in Figures~\ref{fig:unequal_error_a} and \ref{fig:unequal_error_c}. By comparing the results for $(m_1, m_2) = (60, 40)$ and $(m_1, m_2) = (80, 20)$, we find that it gets more challenging to recover $\bm{M}^*$ if the cluster sizes are more imbalanced. The red curve in each plot corresponds the threshold of the sufficient conditions given in Theorem~\ref{the:2_cluster_unequal_d_2} for $d = 2$. Clearly, there is a gap between the empirical phase transition boundary and the theoretical threshold, which implies the condition in Theorem~\ref{the:2_cluster_unequal_d_2} may not be necessary.

To understand why there exists such gap, we construct the dual variables according to our construction in Lemma~\ref{lemma:guess_dual_unequal_new} for each realization of $\bm{A}$, and in Figures~\ref{fig:unequal_rate_b} and \ref{fig:unequal_rate_d} we plot the failure rate in \eqref{eq:failure_rate} for $\widetilde{\bm{\Lambda}}$ defined in Lemma~\ref{lemma:simplify_lambda_unequal}. As a result, we see that in some certain regime of $\alpha$ and $\beta$, even though our dual construction does not satisfy such conditions, empirically we still observe exact recovery of \eqref{eq:SDP_unequal_known} in Figures~\ref{fig:unequal_error_a} and \ref{fig:unequal_error_c}. This implies our guess of the dual variables in Lemma~\ref{lemma:guess_dual_unequal_new} is sub-optimal. 
In addition, to further verify the sharpness of the conditions in Theorem~\ref{the:2_cluster_unequal_d_2}. As $n$ is large, we test the failure rate under different setting of $(m_1,m_2)$. As we can see in Figure~\ref{fig:two_unequal_exp_large_n}, the threshold of the condition in Theorem~\ref{the:2_cluster_unequal_d_2} sharply characterizes the empirical phase transition boundaries, which indicates the derived condition is tight based on our guessed dual variables.

\begin{figure}[t!]
    \centering
    \vspace{-0.7cm}
    \setlength\tabcolsep{1.0pt}
	\renewcommand{\arraystretch}{0.6}
	\begin{tabular}{p{0.32\textwidth}<{\centering} p{0.32\textwidth}<{\centering} p{0.32\textwidth}<{\centering}}
	\subfloat[\tiny{Error of $\bm{M}_{\text{SDP}}$, $(m_1, m_2) = (50, 50)$}]{\includegraphics[width = 0.31\textwidth]{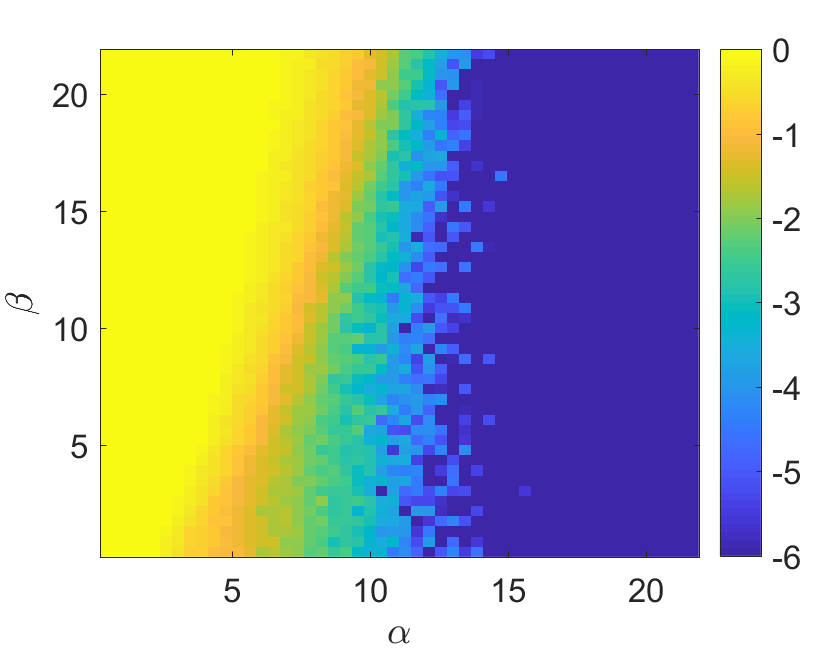}\label{fig:unknown_a}}
	&\subfloat[\tiny{Error of $\bm{M}_{\text{round}}$, $(m_1, m_2) = (50, 50)$}]{\includegraphics[width = 0.31\textwidth]{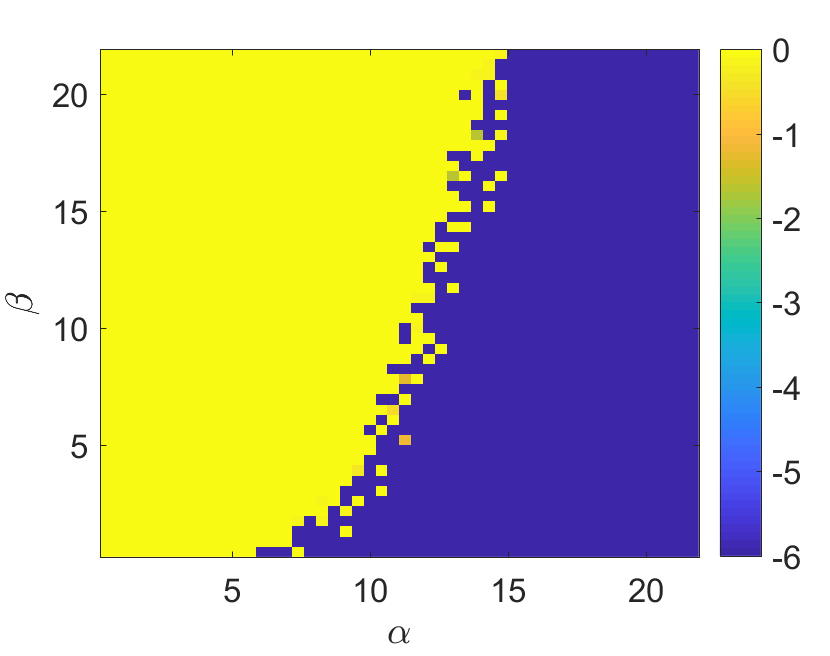}\label{fig:unknown_b}}
    &\subfloat[\tiny{Failure rate, $(m_1, m_2) = (50, 50)$}]{\includegraphics[width = 0.31\textwidth]{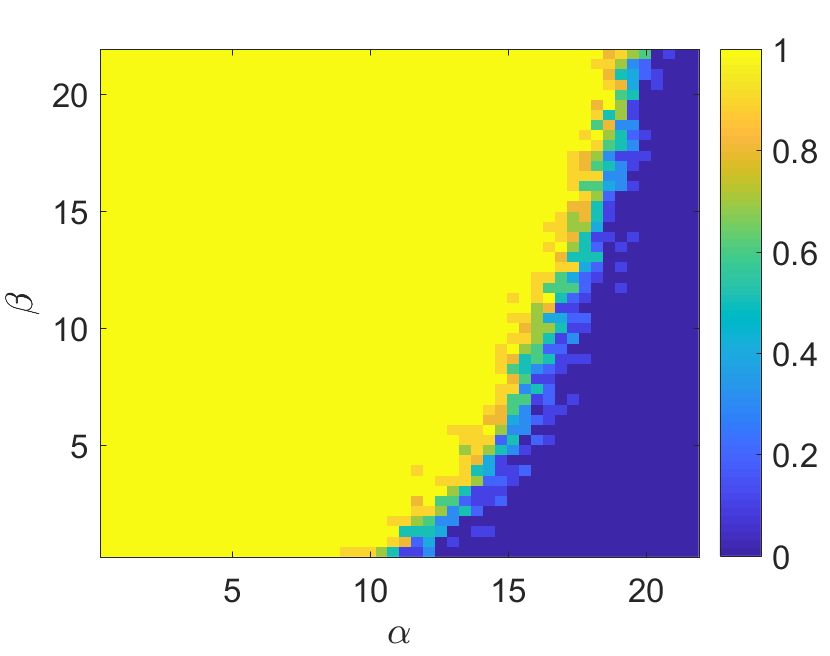}\label{fig:unknown_c}}\\[-5pt]
    \subfloat[\tiny{Error of $\bm{M}_{\text{SDP}}$, $(m_1, m_2) = (80, 20)$}]{\includegraphics[width = 0.31\textwidth]{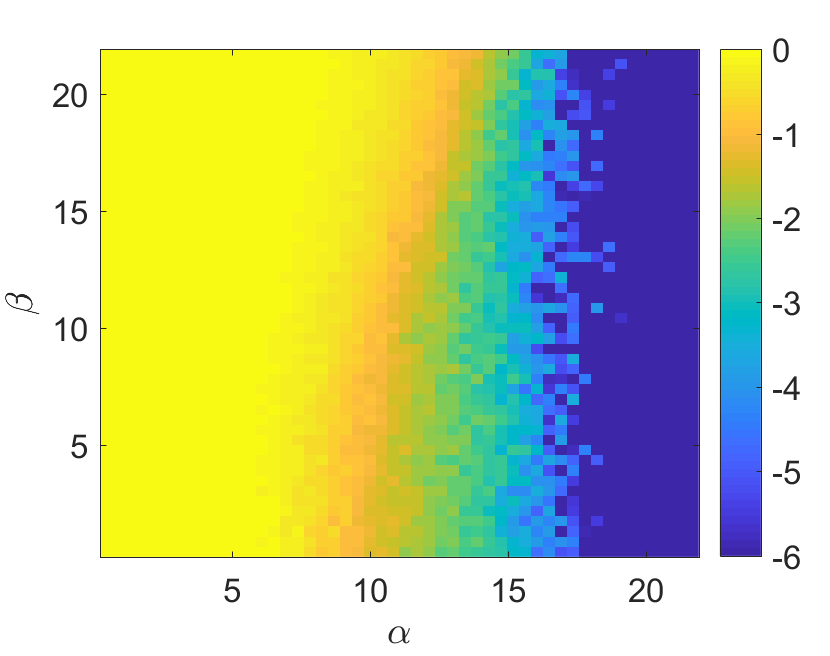}\label{fig:unknown_d}}
	&\subfloat[\tiny{Error of $\bm{M}_{\text{round}}$, $(m_1, m_2) = (80, 20)$}]{\includegraphics[width = 0.31\textwidth]{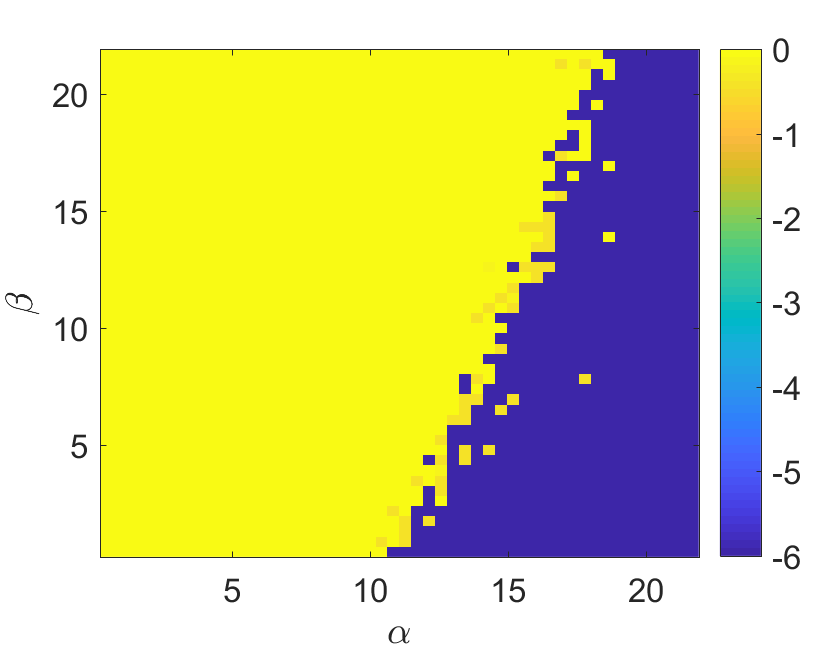}\label{fig:unknown_e}}
    &\subfloat[\tiny{Failure rate, $(m_1, m_2) = (80, 20)$}]{\includegraphics[width = 0.31\textwidth]{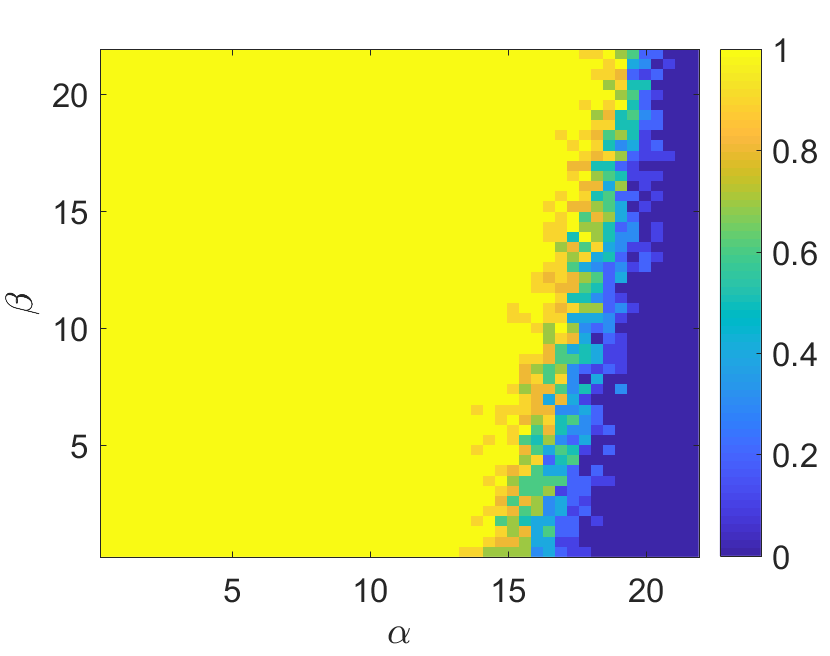}\label{fig:unknown_f}}
	\end{tabular}
    \vspace{-0.2cm}
    \caption{\textit{Results on two clusters with unknown cluster sizes}. \protect\subref{fig:unknown_a} and \protect\subref{fig:unknown_d}: Recovery error of $\bm{M}_{\text{SDP}}$ by \eqref{eq:SDP_unknown}; \protect\subref{fig:unknown_b} and \protect\subref{fig:unknown_e}: Recovery error of $\bm{M}_{\text{round}}$ by \eqref{eq:rounding_formula}; \protect\subref{fig:unknown_c} and \protect\subref{fig:unknown_d}: The failure rate defined in \eqref{eq:failure_rate} for $\widetilde{\bm{\Lambda}}$ in Lemma~\ref{lemma:simplify_lambda_unknown}. }
    \label{fig:two_unknown_rounding}
\end{figure}

\begin{figure}[t!]
    \vspace{-0.5cm}
    \centering
    \subfloat[\scriptsize{$(m_1, m_2) = (600, 400)$}]{\includegraphics[width = 0.31\textwidth]{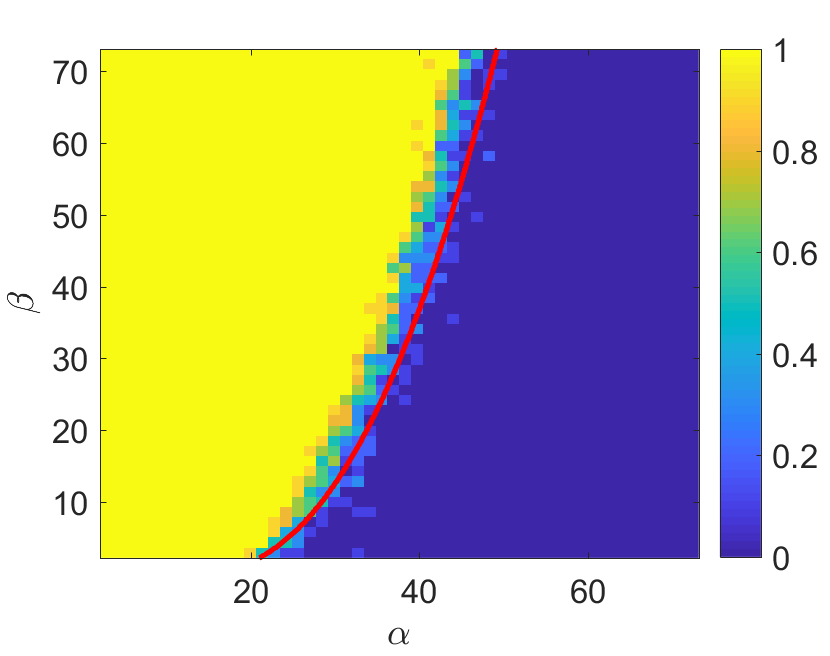}
    \label{fig:unknown_500_500}
    }
    \subfloat[\scriptsize{$(m_1, m_2) = (800, 200)$}]{\includegraphics[width = 0.31\textwidth]{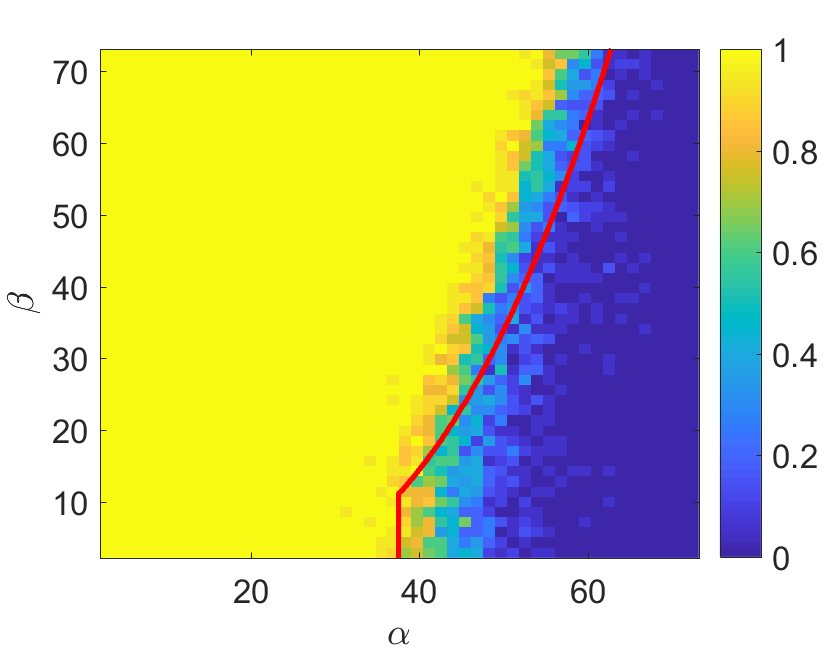}
    \label{fig:unknown_600_400}
    }
    \vspace{-0.25cm}
    \caption{\textit{Results on two clusters with unknown cluster sizes with large $n$}. We plot the failure rate in \eqref{eq:failure_rate} for $\widetilde{\bm{\Lambda}}$ in Lemma~\ref{lemma:simplify_lambda_unknown}. The threshold in Theorem~\ref{the:2_cluster_unknown_d_2} for $d = 2$ is shown in red.}
    \label{fig:two_unknown_large_n}
        \vspace{-0.5cm}
\end{figure}

\subsection{Two clusters with unknown cluster sizes} 
\label{sec:exp_2D_unknown}
Here we consider two clusters with unknown cluster sizes, where \eqref{eq:SDP_unknown} is used for recovery.  
We simulate $\bm{A}$ with different settings of $(m_1, m_2)$. In Figures~\ref{fig:unknown_a} and \ref{fig:unknown_d} we plot the recovery error of the SDP solutions. Also, we apply the rounding step \eqref{eq:rounding_formula} to refine the SDP solutions and plot the error after rounding in Figures~\ref{fig:unknown_b} and \ref{fig:unknown_e}. As a result, one can see that the result after rounding clearly exhibits certain improvement against the original one. 
In addition, to check whether our construction of the dual variables in Lemma~\ref{lemma:guess_dual_unknown} is optimal, in Figures~\ref{fig:unknown_c} and \ref{fig:unknown_f} we show the failure rate in \eqref{eq:failure_rate} for $\widetilde{\bm{\Lambda}}$ in Lemma~\ref{lemma:simplify_lambda_unknown}. Again, we see that there is certain regime of $\alpha, \beta$ where exact recovery by \eqref{eq:SDP_unknown} is possible (see Figures~\ref{fig:unknown_a} and \ref{fig:unknown_d}), but the constructed dual certificate fails to satisfy the condition, indicating a lack of optimality in the choice of dual variables.


While our guess of the dual variables is not optimal, as $n$ is large we are able to sharply characterize when such a dual certificate fails to satisfy the condition in Lemma~\ref{lemma:guess_dual_unknown}. This is demonstrated in Figure~\ref{fig:two_unknown_large_n}, where we show good agreement between the predicted boundary in Theorem~\ref{the:2_cluster_unknown_d_2} and the empirical boundary of the failure rate.


\begin{figure}[t!]
    \vspace{-0.4cm}
    \centering
	\subfloat[\scriptsize{Equal, by \eqref{eq:SDP_equal}}]{
	\includegraphics[width = 0.23\textwidth]{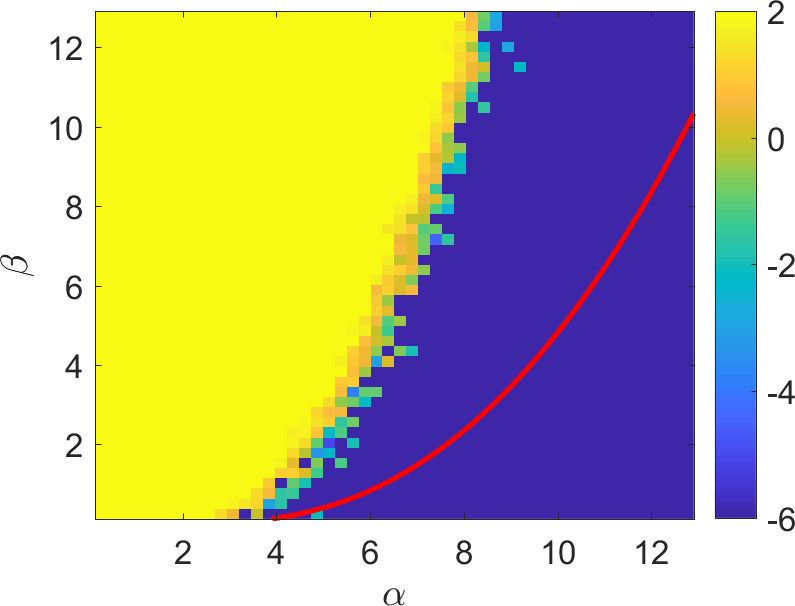}
	\label{fig:SO3_pt_1}
	}
    \subfloat[\scriptsize{Unequal, by \eqref{eq:SDP_unequal_known}}]{\includegraphics[width = 0.23\textwidth]{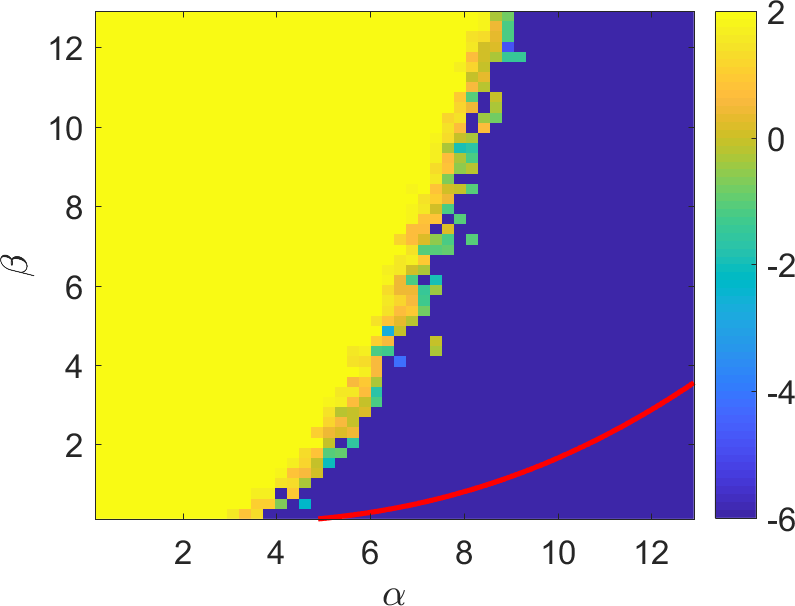}
    \label{fig:SO3_pt_2}
    }
    \subfloat[\scriptsize{Unknown, by \eqref{eq:SDP_unknown}}]{\includegraphics[width = 0.23\textwidth]{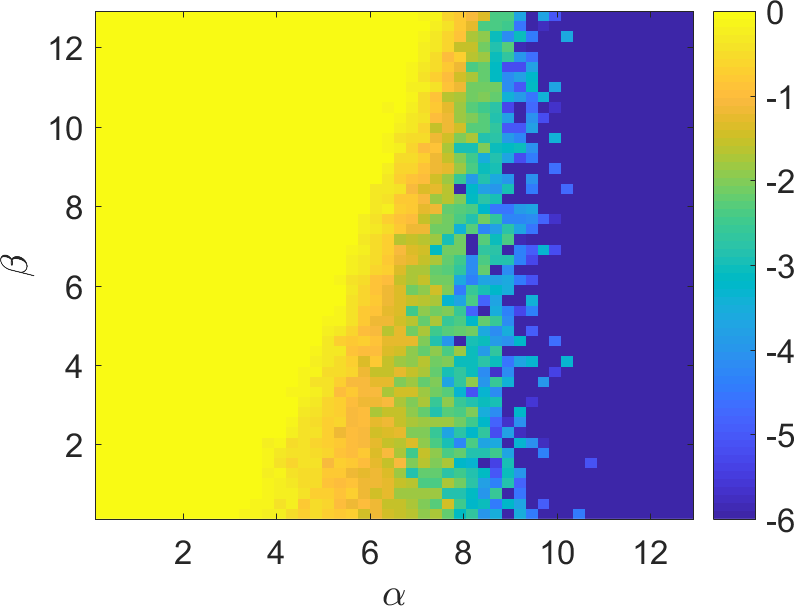}
    \label{fig:SO3_pt_3}
    }
    \subfloat[\scriptsize{Rounding, by \eqref{eq:rounding_formula}}]{\includegraphics[width = 0.23\textwidth]{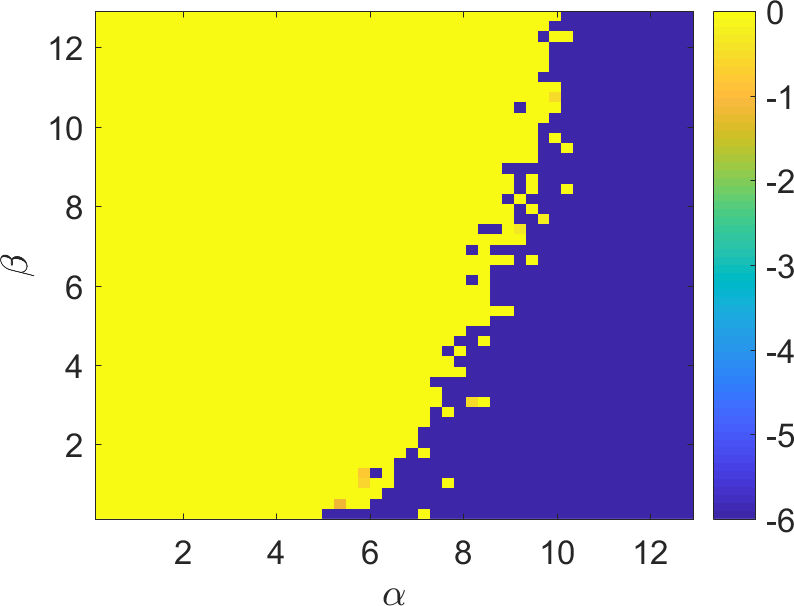}
    \label{fig:SO3_pt_4}
    }
    \vspace{-0.15cm}
    \caption{\textit{Results on $\SO(3)$ transformation}.
    We plot the recovery error of \protect \subref{fig:SO3_pt_1} two equal-sized clusters, the threshold in Theorem~\ref{the:2_cluster_equal_d_2} for $d > 2$ is shown in red; \protect \subref{fig:SO3_pt_2} two unequal-sized clusters, the threshold in Theorems~\ref{the:2_cluster_unequal_d_2} for $d > 2$ is shown in red; \protect \subref{fig:SO3_pt_3}--\protect \subref{fig:SO3_pt_4} two clusters with unknown cluster sizes. We test under $(m_1,m_2) = (25, 25)$ for \protect \subref{fig:SO3_pt_1}, and $(m_1,m_2) = (30, 20)$ for \protect \subref{fig:SO3_pt_2}--\protect \subref{fig:SO3_pt_4}. }
    \label{fig:exp_SO_3}
\end{figure}

\begin{figure}[t!]
    \vspace{-0.3cm}
    \centering
    \subfloat[\scriptsize{$\bm{M}_{\text{SDP}}$ by \eqref{eq:SDP_general_known}}]{\includegraphics[width = 0.29\textwidth]{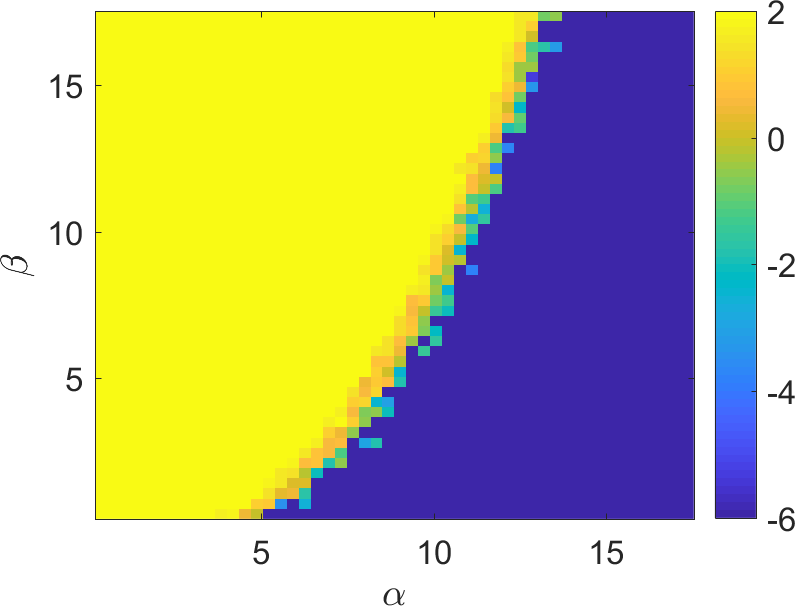} \label{fig:SDP_K_3_unequal}
        }
	\subfloat[\scriptsize{$\bm{M}_{\text{SDP}}$ by \eqref{eq:SDP_unequaK}}]{\includegraphics[width = 0.29\textwidth]{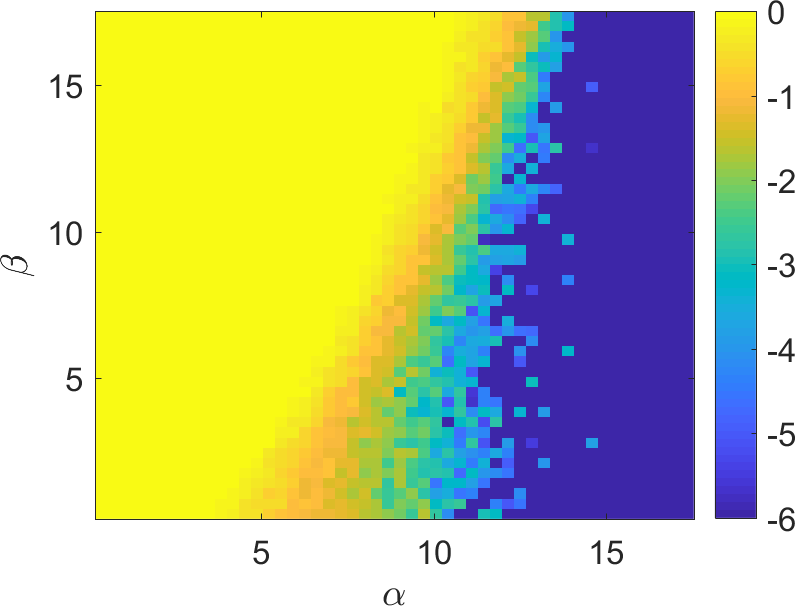} \label{fig:SDP_K_3}
    }
    \subfloat[\scriptsize{$\bm{M}_{\text{round}}$ by \eqref{eq:rounding_formula}}]{\includegraphics[width = 0.29\textwidth]{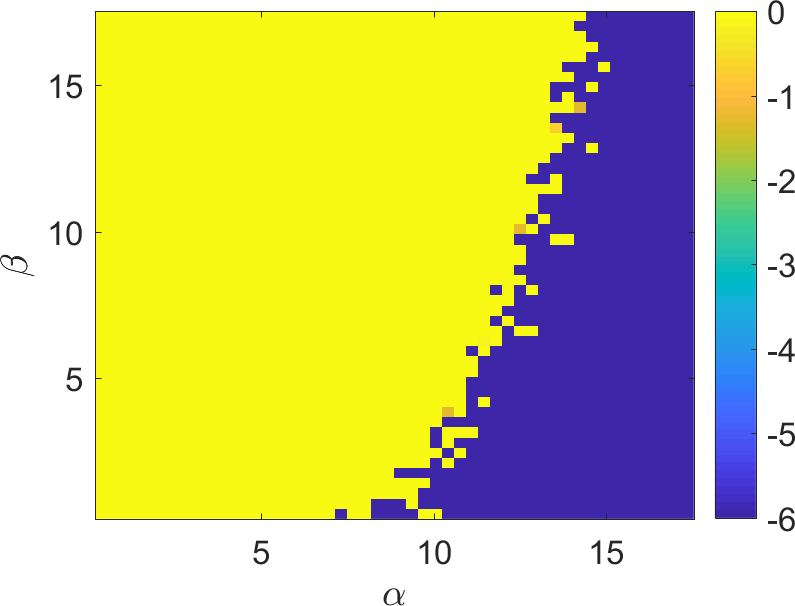}\label{fig:SDP_K_3_round}
    }
    \vspace{-0.15cm}
    \caption{\textit{Results on three clusters}. We test under the setting  $(m_1, m_2, m_3) = (25, 25, 25)$. We plot the recovery error under \protect \subref{fig:SDP_K_3_unequal} known cluster sizes, recovered by \eqref{eq:SDP_general_known}; \protect \subref{fig:SDP_K_3}--\protect \subref{fig:SDP_K_3_round} unknown cluster sizes, recovered by \eqref{eq:SDP_unequaK} and rounded by \eqref{eq:rounding_formula} respectively.}
    \label{fig:K_3_plot}
        \vspace{-0.5cm}
\end{figure}

\subsection{3D rotations and $K = 3$}
\label{sec:exp_general}
Now, we extend our experiments to $\SO(3)$ group transformation. In Figure~\ref{fig:exp_SO_3} we plot the recovery errors by SDPs on different scenarios under two clusters. Similar to the results for $\SO(2)$, we observe sharp phase transitions in the SDP solutions. Specifically, in Figures~\ref{fig:SO3_pt_1} and \ref{fig:SO3_pt_2} we display the thresholds for exact recovery derived in Sections~\ref{sec:SDP_equal} and \ref{sec:SDP_unequal}, respectively. One can see that exact recovery is still possible even above the thresholds, which implies there exists room for improvement of our conditions. 
As mentioned, such looseness is due to two reasons: (1) The factor $\ell_d$ in the exact recovery condition might be loose when $d > 2$; (2) the dual certificate construction might be sub-optimal.

Besides, we test the generalized SDPs proposed in \eqref{eq:SDP_general_known} and \eqref{eq:SDP_unequaK}, under the setting of $K = 3$ and $(m_1, m_2, m_3) = (25, 25, 25)$. The plots of recovery error are shown in Figure~\ref{fig:K_3_plot}. As we can see, the proposed SDPs achieves exact recovery with sharp phase transitions. In addition, the rounding step is effective to further improve the SDP solutions in both cases.

\section{Conclusion and open problems}
\label{sec:discussion_summary}
In this paper, we propose to solve the community detection and rotational synchronization simultaneously via semidefinite programming. Compared with the existing SDP for community detection under SBM, our proposed SDPs not only take into account the graph structure but also the additional information on pairwise rotational alignments, which gives better clustering results.
In particular, for the case of two clusters we obtain sufficient conditions for the SDPs to achieve exact recovery, which characterize the empirical phase transition boundaries. In addition, we acquire new concentration inequalities for the Frobenius norm of sum of random rotation matrices and the operator norm of a random block matrix as interesting byproducts from our analysis.  

There are several directions to be further explored. First, it is natural to expect that the results obtained in this paper can be extended to a much more general family of groups, such as symmetric group and orthogonal group, using techniques introduced in \cite{liu2020unified}. Also, our analysis can be applied to other probabilistic models on the pairwise relations, such as the additive Gaussian model considered in \cite{ling2020near}. In addition, the conditions for exact recovery by \eqref{eq:SDP_unequal_known} and \eqref{eq:SDP_unknown} in the case of unequal and unknown cluster sizes can be improved by finding a more optimal construct of dual variables. 
For $d> 2$, current results can be improved with a sharp Frobenius norm bound on the sum of random rotation matrices that is tighter than the one given in Theorem~\ref{lemma:large_deviation_random_orthogonal}, i.e., the determination of $l_d$ when $d > 2$ in Theorems \ref{the:2_cluster_equal_d_2}, \ref{the:2_cluster_unequal_d_2}, and \ref{the:2_cluster_unknown_d_2} can be improved.
Furthermore, we note that finding the information-theoretical limit of exact recovery on the proposed probabilistic model is still an open problem. Finally, although the SDP can be solved in polynomial time complexity, it is still computationally expensive when the graph size is large, and one can instead consider spectral methods and message passing algorithms in that situation.

\appendix

\section{Important technical ingredients}
\label{sec:other_proof}
This section is devoted to technical ingredients that support our proofs in supplementary material Section~\ref{sec:proof}. Specifically, we derive new matrix concentration inequalities for random rotation matrices in Appendix~\ref{sec:sum_random_orthogonal}, and new upper bounds of the operator norm of random block matrices in Appendix~\ref{sec:norm_rand_matrix}.

\subsection{A tail bound of Bernoulli trials}
\label{sec:bound_binomial} Our analysis relies on the following tail bound of a Bernoulli trail.

\begin{lemma}[\textup{\cite[Lemma 2]{hajek2016achievinga}}]
Let $X \sim \mathrm{Binom}\left(m, \alpha \log n/n\right)$ for $m \in \mathbb{N},\; \alpha = O(1)$, where $ m = \rho n$ for some $\rho > 0$.  Let $k_n = \tau\rho \log n$ for some $\tau \in (0, \alpha]$. Then for a sufficiently large $n$,
\begin{equation*}
    \mathbb{P}\left\{ X \leq k_n\right\} = n^{-\rho\left(\alpha - \tau \log(\frac{e\alpha}{\tau}) + o(1)\right)}.
    \label{eq:bound_bernoulli}
\end{equation*}
\label{lemma:Bound_bernoulli}
\end{lemma}

\subsection{Concentration inequalities for random rotation matrices}
\label{sec:sum_random_orthogonal}
Given a series of i.i.d. random matrices $\bm{X}_i \in \mathbb{R}^{d\times d}$ for $i = 1,\ldots m$ such that 
\begin{equation}
    \bm{X}_i = 
    \begin{cases}
    \bm{R}_i, &\quad \text{with probability $q$},\\
    \bm{0}, &\quad \text{otherwise},
    \end{cases}
    \label{eq:X_i_definition}
\end{equation}
where $\bm{R}_i$ is a random rotation matrix with $\det(\bm{R}_i) = 1$, uniformly drawn from $\mathrm{SO}(d)$ according to Haar measure. Our analysis relies on bounding the Frobenius norm of $\bm{Z} = \sum_{i = 1}^m \bm{X}_i$, namely $\|\bm{Z} \|_\mathrm{F}$. 
In general, a fairly tight bound can be obtained from matrix concentration inequalities such as matrix Bernstein~\cite[Theorem 1.6]{tropp2012user}, which is presented as following:
\begin{theorem}[\textup{Matrix Bernstein inequality, \cite[Theorem 1.6]{tropp2012user}}]
Let $\bm{X}_1, \ldots, \bm{X}_m \in \mathbb{R}^{d \times d}$ be independent, centered random matrices and each one is uniformly bounded as 
\begin{equation*}
    \mathbb{E}[\bm{X}_i] = \bm{0} \quad \text{and} \quad \|\bm{X}_i\| \leq L, \quad \text{for } i = 1,\ldots, n.
\end{equation*}
Let $\bm{Z} = \sum \limits_{i = 1}^n \bm{X}_i$ and $v(\bm{Z})$ denotes the matrix variance statistic of $\bm{Z}$ as 
\begin{equation*}
    v(\bm{Z}) = \max\{\|\mathbb{E}(\bm{Z}\bm{Z}^\top)\|, \; \|\mathbb{E}(\bm{Z}^\top\bm{Z})\|\}.
\end{equation*}
Then for any $t > 0$,
\begin{equation*}
    \mathbb{P}\left\{\|\bm{Z}\| \geq t \right\} \leq 2d \exp\left(\frac{-t^2/2}{v(\bm{Z}) + Lt/3}\right).
\end{equation*}
\label{lemma:matrix_bernstein}
\end{theorem}

As a result, by applying Theorem~\ref{lemma:matrix_bernstein} on $\bm{X}_i$ defined in \eqref{eq:X_i_definition} we get:
\begin{theorem}
Let $\bm{X}_1, \ldots, \bm{X}_m \in \mathbb{R}^{d \times d}$ be i.i.d. random matrices defined in \eqref{eq:X_i_definition} and $\bm{Z} = \sum \limits_{i = 1}^m \bm{X}_i$, $m = O(n)$. Then for any $t > 0$,
\begin{equation}
    \mathbb{P}\left\{\left\|\bm{Z}\right\|_\mathrm{F} \geq \sqrt{d}t\right\} \leq 2d\exp\left(\frac{-t^2/2}{qm + t/3}\right).
    \label{eq:matrix_bernstein_ineq}
\end{equation}
In other words, 
\begin{equation}
    \frac{1}{\sqrt{d}}\|\bm{Z}\|_\mathrm{F} \leq \sqrt{2qm(c\log n + \log 2d)}\left(\sqrt{1 + \frac{c\log n + \log 2d}{18qm}} + \sqrt{\frac{c\log n + \log 2d}{18qm}}\right)
    \label{eq:matrix_bernstein_ineq2}
\end{equation}
with probability $1 - n^{-c}$ for any $c > 0$.
\label{lemma:large_deviation_random_orthogonal}
\end{theorem}

\begin{proof}
By definition of the operator norm and Frobenius norm $\| \bm{Z} \|  \leq \| \bm{Z} \|_\mathrm{F} \leq \sqrt{d} \| \bm{Z} \|$.
Therefore, it satisfies
\begin{equation*}
    \mathbb{P}\{\left\|\bm{Z}\right\|_\mathrm{F} \geq \sqrt{d}t\} \leq \mathbb{P}\left\{\left\|\bm{Z}\right\|_\mathrm{2} \geq t\right\},
\end{equation*}
which enables us to bound $\left\|\bm{Z}\right\|_\mathrm{F}$ via $\left\|\bm{Z}\right\|_\mathrm{2}$ by Theorem~\ref{lemma:matrix_bernstein}. To this end, 
for $L$ defined in Theorem~\ref{lemma:matrix_bernstein}, we have $\|\bm{X}_i\| \leq 1$ then $L = 1$. For $v(\bm{Z})$ we have $\|\mathbb{E}(\bm{Z}^\top\bm{Z})\| = \left\|\sum_{i = 1}^m\mathbb{E} \left[\bm{X}_i^\top\bm{X}_i \right]\right\| = \left\|\sum_{i = 1}^m q\bm{I}_d\right\| = qm$,
similarly $\|\mathbb{E}(\bm{Z}\bm{Z}^\top)\| = qm$ then it follows $v(\bm{Z}) = qm$. This leads to \eqref{eq:matrix_bernstein_ineq}. Furthermore, by setting the RHS of \eqref{eq:matrix_bernstein_ineq} to be $n^{-c}$ for $c > 0$ and solve it we get \eqref{eq:matrix_bernstein_ineq2}.
\end{proof}

\begin{remark}
In our analysis, we are interested in the case of $m = \rho n, q = \beta \log n/n$ for some $\rho, \beta$ and $n \gg d$. Then \eqref{eq:matrix_bernstein_ineq2} can be written as
\begin{equation}
    \begin{aligned}
        \frac{1}{\sqrt{d}}\|\bm{Z}\|_\mathrm{F} &\leq \sqrt{2c\rho\beta}\left(\sqrt{1 + \frac{c}{18\rho\beta}} + \sqrt{\frac{c}{18\rho\beta}}\right) \log n \approx \sqrt{2c\rho\beta}\log n
    \end{aligned}
    \label{eq:equal_alpha_i1_approx}
\end{equation}
where we use the approximation $\left(\sqrt{1 + \frac{c}{18\rho\beta}} + \sqrt{\frac{c}{18\rho\beta}}\right) \approx 1$ since usually $18\rho\beta \gg c$. 
\end{remark}

In particular, when the transformation is a $2$-dimensional rotation i.e. $d = 2$, we are able to obtain a result sharper than Theorem~\ref{lemma:large_deviation_random_orthogonal} as following: 
\begin{theorem}  Under the setting of Theorem~\ref{lemma:large_deviation_random_orthogonal} with $d = 2$, for a sufficiently large $m$ 
\begin{equation*}
    \frac{1}{\sqrt{d}}\|\bm{Z}\|_\mathrm{F} \leq \sqrt{cqm \log n}(1 + o(1))
\end{equation*}
with probability $1 - n^{-c}$ for any $c > 0$. 
\label{lemma:sum_orth_2}
\end{theorem}

\begin{proof}
To begin with, by definition of $\bm{X}_i$ in \eqref{eq:X_i_definition}, we can rewrite $\bm{X}_i = r_i\bm{R}_i$ where $r_i$ is a Bernoulli random variable such that $\mathbb{P}\{r_i = 1\} = q$. Moreover, when $d = 2$ each rotation matrix $\bm{R}_i$ can be expressed as $\bm{R}_i = 
    \begin{pmatrix}
    \cos \theta_i& -\sin \theta_i\\
    \sin \theta_i &\cos \theta_i
    \end{pmatrix}$
for some $\theta_i \in [0, 2\pi)$. As a result, 
\begin{equation}
    \frac{1}{d}\|\bm{Z}\|^2_\mathrm{F} = \left(\sum \limits_{i = 1}^m r_i\sin \theta_i\right)^2 + \left(\sum \limits_{i = 1}^m r_i\cos \theta_i\right)^2 =: x_m.
    \label{eq:x_m_def}
\end{equation}
We are going to bound $ \frac{1}{\sqrt{d}}\|\bm{Z}\|_\mathrm{F}$ in terms of $x_m$. Our technique is to compute the $k$-th moment of $x_m$ i.e. $\mathbb{E}[x_m^k]$ for some integer $k \leq m$ followed by Markov's inequality. First,
\begin{equation*}
    \mathbb{E}[x_m] = \mathbb{E}\Bigg[\sum \limits_{i = 1}^m \sum \limits_{j = 1}^m r_ir_j\left(\cos \theta_i\cos \theta_j+ \sin \theta_i\sin \theta_j\right)\Bigg] = \sum \limits_{i = 1}^m\sum \limits_{j = 1}^m \mathbb{E}\left[r_ir_j\right]\mathbb{E}\left[\cos(\theta_i - \theta_j)\right].
\end{equation*}
Next, in general $\mathbb{E}[x_m^k]$ is given as
\begin{equation}
    \begin{aligned}
    &\mathbb{E}\left[x^k_m\right] = 
    &\sum \limits_{i_1, \ldots, i_k = 1}^m \sum \limits_{j_1, \ldots, j_k = 1}^m \mathbb{E}\left[r_{i_1}r_{j_1}\cdots r_{i_k}r_{j_k}\right]\mathbb{E}\left[ \cos(\theta_{i_1} - \theta_{j_1})  \cdots \cos(\theta_{i_k} - \theta_{j_k})\right].
    \end{aligned}
    \label{eq:exp_x_n_k}
\end{equation}
Let $s_i \in \{-1, 1\}$ denotes a discrete variable for $i = 1, \ldots, k-1$. Then applying the property $\cos(\alpha_1)\cos(\alpha_2) = \frac{1}{2}\left[\cos(\alpha_1 + \alpha_2) + \cos(\alpha_1 - \alpha_2)\right]$ recursively for $k-1$ times yields
{\scriptsize
\begin{equation*}
    \cos(\theta_{i_1} - \theta_{j_1})  \cdots \cos(\theta_{i_k} - \theta_{j_k}) = \frac{1}{2^{k-1}}\sum \limits_{s_1, \ldots, s_{k-1} \in \{-1, 1\}} \cos\left((\theta_{i_1} \!-\! \theta_{j_1}) \!+\! s_1(\theta_{i_2} - \theta_{j_2})\!+\! \cdots \!+\! s_{k-1}(\theta_{i_k} \!-\! \theta_{j_k})\right).
\end{equation*}
}

\noindent Plugging this into \eqref{eq:exp_x_n_k} gives
{\scriptsize
\begin{align}
\mathbb{E}\left[x^k_m\right] 
&= \left(\frac{1}{2^{k-1}}\right)\sum \limits_{s_1, \ldots, s_{k-1} \in \{-1, 1\}} \sum \limits_{i_1, \ldots, i_k} \sum \limits_{j_1, \ldots, j_k}\mathbb{E}\left[r_{i_1}r_{j_1} \cdots r_{i_k}r_{j_k}\right] \mathbb{E}\left[\cos\left(\theta_{i_1} \!-\! \theta_{j_1} \!+\! \cdots \!+\! \delta_{k-1}(\theta_{i_k} \!-\! \theta_{j_k})\right)\right] \nonumber \\
&\overset{(a)}{=} \sum \limits_{i_1, \ldots, i_k} \sum \limits_{j_1, \ldots, j_k}\mathbb{E}\left[r_{i_1}r_{j_1} \cdots r_{i_k}r_{j_k}\right] \mathbb{E}\left[\cos\left(\theta_{i_1} - \theta_{j_1} + \theta_{i_2} - \theta_{j_2} + \cdots + \theta_{i_k} - \theta_{j_k}\right)\right]. \label{eq:sum_F}
\end{align}}

\noindent Here, ($a$) holds since $
    \sum \limits_{i_1, \ldots, i_k} \sum \limits_{j_1, \ldots, j_k}\mathbb{E}\left[\cos\left(\theta_{i_1} \!-\! \theta_{j_1} \!+\! \cdots \!+\! s_{k-1}(\theta_{i_k} \!-\! \theta_{j_k})\right)\right]$
is identical for different choices of $\{s_i\}_{i = 1}^{k-1}$, and there are $2^{k-1}$ number of choices in total then the factor $1/2^{k-1}$ is cancelled by summing all of them, and we can focus on the case that $s_1 = s_2 = \cdots =s_{k-1} = 1$. To proceed, let us denote $S_i = \{i_1, \ldots, i_k\}$ and $S_j = \{j_1, \ldots, j_k\}$ as two \textit{multisets} (a multiset is a variant of a set that allows repeated elements) that contain all $i$ and $j$ respectively, then
\begin{equation*}
    \mathbb{E}\left[\cos\left(\theta_{i_1} - \theta_{j_1} + \theta_{i_2} - \theta_{j_2} + \cdots + \theta_{i_k} - \theta_{j_k}\right)\right] = 
    \begin{cases}
    1, &\quad S_i = S_j,\\
    0, &\quad \text{otherwise}.
    \end{cases}
\end{equation*}
This holds because the sum of the random angles equals $0$ mod $2\pi$ if $S_i = S_j$, otherwise it vanishes by taking the expectation. 
Therefore, under $S_i = S_j$, let $m_l$ denotes the multiplicity of element $l \in \left\{1, \ldots, m\right\}$ that either $S_i$ or $S_j$ contains, then it satisfies
\begin{equation*}
    \mathbb{E}[r_{i_1}r_{j_1}\cdots r_{i_k} r_{j_k}] \overset{(a)}{=} \mathbb{E}[r_{1}^{2m_1}r_{2}^{2m_2} \cdots r_{m}^{2m_n}] = \prod_{i = 1}^m \mathbb{E}[r_i^{2m_i}] \overset{(b)}{=} q^{\tau(m_1,\ldots m_m)}
\end{equation*}
where $\tau(m_1,\ldots m_m)$ denotes the number of non-zero elements in $\{m_1, \ldots, m_m\}$, ($a$) holds because of $S_i = S_j$, and $(b)$ comes from $\mathbb{E}[r_i^{2m_i}] = q, \; \forall m_i > 0$. Given the above, in \eqref{eq:sum_F} the summation over all $i$ and $j$ can be instead summing over $m_1, \ldots, m_m$. That is,
\begin{equation}
    \begin{aligned}
    \mathbb{E}[x_m^k] = \sum \limits_{\substack{m_1, \ldots, m_m \geq 0 \\ m_1 + \cdots + m_m = k}}\Pi(m_1,\ldots m_m) q^{\tau(m_1,\ldots m_m)}.
    \end{aligned}
    \label{eq:x_k_sum_pi}
\end{equation}
Here, $\Pi(m_1,\ldots m_m)$ denotes the number of combinations of $\{i_1, j_1, \ldots, i_k, j_k\}$ that satisfy $S_i = S_j$ with multiplicity $\{m_1, \ldots, m_m\}$. To determine it, notice that given a fixed $\{m_1, \ldots, m_m\}$ there are ${k \choose {m_1, m_2, \ldots, m_m}}$ number of combinations for either $S_i$ or $S_j$ that satisfy it. Then we have $\Pi(m_1,\ldots m_m) = {k \choose {m_1, m_2, \ldots, m_m}}^2$. Plugging this into \eqref{eq:x_k_sum_pi} gives
\begin{equation*}
    \begin{aligned}
    \mathbb{E}\left[x_m^k\right] =
    \sum \limits_{\substack{m_1, \ldots, m_m \geq 0 \\ m_1 + \cdots + m_m = k}}q^{\tau(m_1,\ldots m_m)} {k \choose {m_1, \ldots, m_m}}^2  
    = \sum \limits_{l = 1}^{k} {n \choose l}q^l \sum \limits_{\substack{m_1, \ldots, m_l \geq 1 \\ m_1+\cdots +m_l = k}}   {k \choose {m_1, \ldots, m_l}}^2,
    \end{aligned}
    \label{eq:moments_sum}
\end{equation*}
where the RHS holds by fixing $\tau(m_1, \ldots, m_m) = l$ for $l = 1,\ldots, k$. Furthermore,
\begin{equation*}
    \begin{aligned}
         \mathbb{E}\left[x_m^k\right] &= {n \choose k}q^k \underbrace{\sum \limits_{\substack{m_1,\ldots, m_k \geq 1 \\ m_1 + \ldots + m_k = k}}{k \choose {m_1, \ldots, m_k}}^2}_{ = (k!)^2 \text{ since $m_1 = \cdots = m_k = 1$}} + \sum \limits_{l = 1}^{k-1} {m \choose l}q^l \sum \limits_{\substack{m_1, \ldots, m_l \geq 1 \\ m_1+\cdots +m_l = k}}   {k \choose {m_1, \ldots, m_l}}^2\\
         &= \left(\frac{m!}{k!(m-k)!}\right)q^k(k!)^2 + O(q^{k-1}m^{k-1}) =  k!q^km^k + O(q^{k-1}m^{k-1})\\
         &\leq c_0k!q^km^k
    \end{aligned}
\end{equation*}
where the last inequality holds for any constant $c_0 > 1$ with a sufficiently large $m$. 

Now we are ready to bound $x_m$ by applying Markov's inequality as
\begin{equation*}
    \begin{aligned}
    \mathbb{P}\left\{ x_m \geq t\right\} &= \mathbb{P}\left\{x_m^k \geq t^k\right\} \leq t^{-k}\mathbb{E}[x_m^k] \leq t^{-k} \cdot c_0k!q^km^k \leq  \bigg(\frac{(c_0e\sqrt{k})^{1/k}k qm}{et}\bigg)^{k}
    \end{aligned}
    \label{eq:moments_bound}
\end{equation*}
where the last inequality uses the fact that $k! \leq e k^{k+\frac{1}{2}} e^{-k}$.  
If we set $k = \lceil c \log n \rceil$ and 
\begin{equation*}
    t = (c_0e\sqrt{k})^{1/k} kqm = e^{\frac{1}{k}\log \left( c_0e\sqrt{k}\right)} kqm = (1 + o(1))cqm\log n,
\end{equation*}
then we have
\begin{equation*}
    \mathbb{P}\left\{ x_n \geq t\right\} \leq n^{-c \left( 1 + \log \frac{t}{(c_0e\sqrt{k})^{1/k} kqm }\right)} = n^{-c}. 
\end{equation*}
By definition of $x_m$, this further leads to $\frac{1}{\sqrt{d}}\|\bm{Z}\|_\mathrm{F} \leq \sqrt{cqm\log n}(1 + o(1))$
with probability $1 - n^{-c}$, which completes the proof.
\end{proof}

\begin{remark}
When $d = 2$, we can interpret $\frac{1}{\sqrt{d}}\|\bm{Z}\|_\mathrm{F}$ from a random walk perspective: Suppose a random walk on a 2-D plane such that in each step, with probability $q$ we take a unit-length step towards a random direction uniformly drawn from $[0, 2\pi)$, and with probability $1 - q$ we stay where we are. As a result, $\frac{1}{\sqrt{d}}\|\bm{Z}\|_\mathrm{F}$ can be interpreted as the distance that we travel within $m$ steps where $r_{i}$ indicates whether we move or not in the $i$-th step, and $\theta_i$ stands for the corresponding direction. In fact, such 2-D random walk with $q = 1$ has been originally studied by Rayleigh in~\cite{rayleigh1880xii, rayleigh1896theory}. Let $p_m(r)$ be the probability distribution of travelling a distance $r$ with $m$ steps, Rayleigh showed that as $m \rightarrow \infty$, 
\begin{equation*}
    p_m(r) \sim \frac{2r}{m}e^{-r^2/m}.
\end{equation*}
Therefore, asymptotically it satisfies $r \leq \sqrt{cm\log n}$ with probability $1 - n^{-c}$ for any $c > 0$. Importantly, this bound agrees with the non-asymptotic result in Theorem~\ref{lemma:sum_orth_2}, which indicates the sharpness of our result.
\end{remark}

\begin{remark}
Again, let us consider $m = \rho n$ and $q = \beta\log n/n$ for some $\rho, \beta$. Then the bound in Theorem~\ref{lemma:sum_orth_2} can be rewritten as 
\begin{equation}
    \frac{1}{\sqrt{d}}\|\bm{Z}\|_\mathrm{F} \leq \sqrt{c\rho\beta n} \log n
    \label{eq:sharp_bound_2}
\end{equation}
with probability $1 - n^{-c}$. As we can see, \eqref{eq:sharp_bound_2} is sharper than the previous result in \eqref{eq:equal_alpha_i1_approx} by only a constant factor 2. However, such tiny difference is significant in our analysis since we will apply the union bound on \eqref{eq:sharp_bound_2} or \eqref{eq:equal_alpha_i1_approx}, therefore even the scaling of the constant $c$ becomes important as well (see e.g. Lemma~\ref{claim:x_i_y_i_equal}). 
\end{remark}

\begin{remark}
One might want to generalize the proof of Theorem~\ref{lemma:sum_orth_2} to the case when $d > 2$. In this case, the difficulty lies in bounding the $k$-th moments of $x_m$ (i.e. $\mathbb{E}[x_m^k]$), where $x_m$ does not have a simple expression as \eqref{eq:x_m_def} when $d = 2$, then we cannot find a way similar to \eqref{eq:exp_x_n_k} that expresses $\mathbb{E}[x_m^k]$ as a sum of $\mathbb{E}\left[ \cos(\theta_{i_1} - \theta_{j_1})  \cdots \cos(\theta_{i_k} - \theta_{j_k})\right]$. To resolve this, using representation theory of the special orthogonal group (e.g.~\cite{collins2006integration, meckes2019random, banica2010orthogonal, bohorquez2020maximizing}) might be helpful and we leave this generalization as a future work.
\end{remark}

\begin{lemma}
Let $m_1 = \rho n$ and $m_2 = (1- \rho)n$ with $0<\rho <1$, and 
$\bm{Z}_1, \ldots, \bm{Z}_{m_1} \in \mathbb{R}^{d \times d}$ be i.i.d. random matrices where each $\bm{Z}_i = \sum \limits_{j = 1}^{m_2} \bm{X}_{j}$ is independently generated as in Theorem~\ref{lemma:large_deviation_random_orthogonal}. Let $\epsilon_i := \frac{\|\bm{Z}_i\|_\mathrm{F}}{m_2\sqrt{d}}$. Then, when $p, q =  O(\log n/n)$
\begin{equation*}
    \sum \limits_{i = 1 }^{m_1} \epsilon_i \leq \frac{2c}{3}\log n + \rho \sqrt{\frac{q n }{1 - \rho}} + \sqrt{\frac{2cq \rho \log n}{1 - \rho}} = \frac{2c}{3}\log n + o(\log n)
\end{equation*}
with probability $1 - n^{-c}$.
\label{lemma:bound_concentration_Z_i}
\end{lemma}

\begin{proof}
By definition, each $\epsilon_i$ is bounded by $1$, then applying Bernstein's inequality~\cite[Theorem 2.8.4]{vershynin2018high} yields 
\begin{equation}
    \mathbb{P}\left\{\sum \limits_{i = 1}^{m_1} \left(\epsilon_i - \mathbb{E}\left[\epsilon_i\right]\right)\geq t\right\} \leq \exp\left(-\frac{t^2/2}{\sigma^2 + Kt/3}\right)
    \label{eq:Bernstein}
\end{equation}
where $K = 1$ is the boundary, and 
\begin{equation*}
    \sigma^2 = \sum \limits_{i = 1}^{m_1} \left(\mathbb{E}[\epsilon_i^2] - \mathbb{E}[\epsilon_i]^2\right) \leq \sum \limits_{i = 1}^{m_1}\mathbb{E}[\epsilon_i^2] = \sum \limits_{i = 1}^{m_1} \frac{\mathbb{E} \left[ \| \bm{Z}_i\|_\mathrm{F}^2 \right]}{m_2^2 d}= \sum \limits_{i = 1}^{m_1} \frac{q}{m_2} = \frac{q \rho }{ 1 - \rho}
\end{equation*}
is the variance of the sum. Then, by letting the RHS of \eqref{eq:Bernstein} equal to $n^{-c}$ for $c > 0$, we obtain
\begin{equation}
    \begin{aligned}
    \sum \limits_{i = 1}^{m_1} \left(\epsilon_i - \mathbb{E}\left[\epsilon_i\right]\right) &\leq \frac{Kc}{3}\log n + \sqrt{\left(\frac{Kc}{3}\log n\right)^2 + 2\sigma^2 c\log n}\\
    &\leq \frac{2Kc}{3}\log n + \sqrt{2\sigma^2 c\log n}
    \end{aligned}
    \label{eq:bound_Z_i_sum}
\end{equation}
with probability $1 - n^{-c}$. Also, by Jensen's inequality the expectation of $\epsilon_i$ is bounded as 
\begin{equation*}
    \mathbb{E}[\epsilon_i] = \frac{1}{m_2\sqrt{d}}\mathbb{E}\left[\sqrt{\T(\bm{Z}_i^\top \bm{Z}_i)}\right] \leq \frac{1}{m_2\sqrt{d}}\sqrt{\mathbb{E}\left[\T(\bm{Z}_i^\top \bm{Z}_i)\right]} = \sqrt{\frac{q}{(1 - \rho) n }}.
\end{equation*}
Plugging this into \eqref{eq:bound_Z_i_sum} completes the proof.
\end{proof}

\subsection{Non-asymptotic operator norm bounds of random (block) matrices}
\label{sec:norm_rand_matrix}
Another key ingredient of our analysis lies in the non-asymptotic operator norm bounds of random (block) matrices, where most of the previous research focuses on random matrices with i.i.d.~entries especially sub-Gaussian entries~\cite{vershynin2010introduction, vershynin2018high, tao2012topics}. First, the following existing result\footnote{Similar results have been also developed in \cite[Theorem 5.2]{lei2015consistency} and \cite[Lemma 2]{lugosi2020concentration}.} provides an upper bound on the spectrum of an Erd{\"o}s–R{\'e}nyi random graph. 

\begin{theorem}[\textup{\cite[Theorem 5]{hajek2016achievinga}}] Let $\bm{E} \in \mathbb{R}^{n \times n}$ denotes the adjacency matrix of a graph generated from the Erd{\"o}s–R{\'e}nyi model $\mathcal{G}(n,p)$, where the entries $\{ E_{ij} : i < j\}$ are independent and $\{0, 1\}$-valued. Assume that $\mathbb{E}[E_{ij}] \leq p$ where $p = \Omega(c_0\log n/n)$. Then for any $c > 0$, there exists $c_1 > 0$ such that, 
\begin{equation*}
    \|\bm{E} - \mathbb{E}[\bm{E}]\| \leq c_1\sqrt{np}
\end{equation*}
with probability $1 - n^{-c}$.
\label{lemma:A_graph}
\end{theorem}

In the following, instead of focusing on a random matrix with i.i.d.~entries, we consider a random block matrix with i.i.d.~random blocks. In particular, let $\bm{S} \in \mathbb{R}^{m_1 d \times m_2 d}$ be an $m_1 \times m_2$ block random matrix where each block $\bm{S}_{ij} \in \mathbb{R}^{d \times d}$ is i.i.d. such that 
\begin{equation}
    \bm{S}_{ij} = 
    \begin{cases}
    \bm{R}_{ij}, &\quad \text{with probability $q$},\\
    \bm{0}, &\quad \text{otherwise},
    \end{cases}
    \label{eq:S_ij_definition}
\end{equation}
where $\bm{R}_{ij}$ is a random rotation matrix uniformly drawn from $\mathrm{SO}{(d)}$. 
Then we obtain the following non-asymptotic bound on $\|\bm{S}\|$.
\begin{theorem}
Let $\bm{S} \in \mathbb{R}^{m_1 d \times m_2 d}$ be an $m_1 \times m_2$ random block matrix with i.i.d block $\bm{S}_{ij}$ defined in \eqref{eq:S_ij_definition}. Let $m_1,m_2 = O(n)$. Then, for any $c > 0$, there exists $c_1, c_2 > 0$ such that
\begin{equation*}
    \|\bm{S}\| \leq c_1 (\sqrt{q m_1} + \sqrt{q m_2}) + c_2 \sqrt{\log n}.  
\end{equation*}
with probability $1 - n^{-c}$.
\label{lemma:spec_bound_S_12}
\end{theorem}

\begin{proof}
Our technique is based on the classic moment method which is broadly used in random matrix theory (see e.g.~\cite{anderson2010introduction, tao2012topics}). We start from bounding $\|\bm{S}\|$ by the trace of $\bm{S}\bm{S}^\top $ as 
\begin{equation*}
    \mathbb{E}\left[\|\bm{S}\|^{2k}\right] = \mathbb{E}\left[\|(\bm{S}\bm{S}^\top)^k\|\right]  \leq \mathbb{E}\left[\T{\left((\bm{S}\bm{S}^\top)^{k}\right)}\right]
\end{equation*}
which holds for any $k \in \mathbb{N}$. Then, $\mathbb{E}\left[\T{\left((\bm{S} \bm{S}^\top)^{k}\right)}\right]$ can be expanded as
\begin{equation}
    \mathbb{E}\left[\T{\left((\bm{S} \bm{S}^\top)^{k}\right)}\right] = \sum \limits_{i_1, \ldots, i_k \in [m_1]} \sum \limits_{j_1, \ldots, j_k \in [m_2]} \mathbb{E}\left[\T{\left(\bm{S}_{i_1j_1}\bm{S}_{i_2j_1}^\top \bm{S}_{i_2j_2} \bm{S}_{i_3j_2}^\top \cdots \bm{S}_{i_kj_k} \bm{S}_{i_1j_k}^\top\right)}\right]
    \label{eq:sst}
\end{equation}
where $[m_1] := \{1, \ldots, m_1\}$ and $[m_2] := \{1, \ldots, m_2\}$. Let $G(U, V)$ be a complete bipartite graph with the set of nodes $U = [m_1]$ and $V = [m_2]$. Then each $\bm{S}_{ij}$ for $i \in U, \; j \in V$ is associated with an edge $(i,j)$ on the graph $G$. Furthermore, the matrix product $\bm{S}_{i_1j_1}\bm{S}_{i_2j_1}^\top \cdots \bm{S}_{i_kj_k} \bm{S}_{i_1j_k}^\top$ can be treated as a \textit{walk} that goes back and forth along $G$ and follows the path 
\begin{equation*}
    i_1 \rightarrow j_1 \rightarrow i_2 \rightarrow j_2 \rightarrow \cdots \rightarrow i_k \rightarrow j_k \rightarrow i_1
\end{equation*}
which starts and ends at $i_1$ as a cycle with $2k$ steps.
With this in mind, let $(i_1, j_1, \ldots, i_k, j_k)$ denote this walk, since $\mathbb{E}[ \bm{S}_{ij}^k] = \bm{0}$ when $k$ is odd, the sum in \eqref{eq:sst} can be restricted to even cycles on the graph $G$ where each distinct edge traversed by the walk should be visited by an even number of times. Therefore, let $W$ denote the set of all such walks, then
\begin{align}
    \mathbb{E}\left[\T{\left((\bm{S} \bm{S}^\top)^{k}\right)}\right] &= \sum \limits_{(i_1, j_1, \ldots, i_k, j_k) \in W} \mathbb{E}\left[\T{\left(\bm{S}_{i_1j_1}\bm{S}_{i_2j_1}^\top \bm{S}_{i_2j_2} \bm{S}_{i_3j_2}^\top \cdots \bm{S}_{i_kj_k} \bm{S}_{i_1j_k}^\top\right)}\right] \nonumber\\
    &\overset{(a)}{=} \sum \limits_{(i_1, j_1, \ldots, i_k, j_k) \in W} \mathbb{E}\left[\T{\left(\bm{R}_{i_1j_1}\bm{R}_{i_2j_1}^\top  \!\cdots\! \bm{R}_{i_kj_k} \bm{R}_{i_1j_k}^\top\right)}\right] \mathbb{E}[r_{i_1j_1}r_{i_2j_1} \!\cdots\! r_{i_kj_k} r_{i_1j_k} ] \nonumber\\
    &\overset{(b)}{\leq} d \sum \limits_{(i_1, j_1, \ldots, i_k, j_k) \in W} \mathbb{E}[r_{i_1j_1}r_{i_2j_1} \cdots r_{i_kj_k} r_{i_1j_k} ],
    \label{eq:trace_sum}
\end{align}
where $(a)$ holds by rewriting $\bm{S}_{ij} = r_{ij}\bm{R}_{ij}$ for some Bernoulli random variable $r_{ij}$ such that $r_{ij} = 1$ with probability $q$ and $r_{ij} = 0$ otherwise, $(b)$ holds because the product $\bm{R}_{i_1j_1}\cdots \bm{R}_{i_1j_k}$ is a rotation matrix, therefore $\|\bm{R}_{i_1j_1}\bm{R}_{i_2j_1}^\top  \cdots \bm{R}_{i_kj_k} \bm{R}_{i_1j_k}^\top\| = 1$ and
\begin{equation}
    \mathbb{E}\left[\T{\left(\bm{R}_{i_1j_1}\bm{R}_{i_2j_1}^\top  \cdots \bm{R}_{i_kj_k} \bm{R}_{i_1j_k}^\top\right)}\right] \leq d. 
    \label{eq:trace_d}
\end{equation}
To proceed, we introduce the i.i.d. symmetric random variables $t_{ij}$ defined as
\begin{equation*}
    t_{ij} = \begin{cases}
    1, & \text{with probability } \frac{q}{2},\\
    -1, &  \text{with probability } \frac{q}{2},\\
    0, & \text{with probability } 1-q. 
    \end{cases}
\end{equation*}
for $i \in [m_1], j \in [m_2]$. As a result, $t_{ij}$ has symmetric distribution such that $\mathbb{E}[t_{ij}]^{2k -1} = 0$ and $\mathbb{E}[t_{ij}]^{2k} = q,\; \forall k \in \mathbb{N}$. Moreover, the summation in \eqref{eq:trace_sum} satisfies
\begin{align}
    \sum \limits_{(i_1, j_1, \ldots, i_k, j_k) \in W} \mathbb{E}[r_{i_1j_1}r_{i_2j_1} \cdots r_{i_kj_k} r_{i_1j_k} ] \nonumber 
    &= \sum \limits_{(i_1, j_1, \ldots, i_k, j_k) \in W} \mathbb{E}[t_{i_1j_1} t_{i_2j_1} \cdots t_{i_kj_k} t_{i_1j_k} ] \nonumber \\ 
    &\overset{(a)}{=} \sum \limits_{i_1,\ldots, i_k \in [m_1]} \sum \limits_{j_1,\ldots, j_k \in [m_2]} \mathbb{E}[t_{i_1j_1} t_{i_2j_1} \cdots t_{i_kj_k} t_{i_1j_k} ] \nonumber\\
    &= \mathbb{E}\left[\T{\left((\bm{T} \bm{T}^\top)^{k}\right)}\right] \nonumber \\
    &\leq \min\{m_1, m_2\} \mathbb{E}\left[\|\bm{T}\|^{2k}\right] \label{eq:bound_S_T}.
\end{align}
Here, $\bm{T} \in \mathbb{R}^{m_1 \times m_2}$ is the random matrix whose $(i,j)$-th entry is $t_{ij}$, the equality $(a)$ holds since
any walk $(i_1, j_1, \ldots, i_k, j_k) \notin W$ that visits some edge by odd times vanishes. Therefore, \eqref{eq:bound_S_T} enables us to bound $\mathbb{E}\left[\T{\left((\bm{S} \bm{S}^\top)^{k}\right)}\right]$ via $\mathbb{E}\left[\|\bm{T}\|^{2k}\right]$. To this end, by computing the following quantities: 
\begin{align*}
\sigma_{k, 1} & := \Bigg(\sum \limits_{i = 1}^{m_1} \Bigg(\sum \limits_{j = 1}^{m_2} \mathbb{E}\left[t_{ij}^2\right]\Bigg)^k \Bigg)^{1/2k} = m_1^{1/2k} \sqrt{m_2 q}, \\
\sigma_{k, 2} &:= \Bigg(\sum \limits_{j = 1}^{m_2} \Bigg(\sum \limits_{i = 1}^{m_1} \mathbb{E}\left[t_{ij}^2\right] \Bigg)^k\Bigg)^{1/2k} = m_2^{1/2k} \sqrt{m_1 q},   \\ 
\sigma_k^* & := \Bigg(\sum \limits_{i = 1}^{m_1} \sum \limits_{j = 1}^{m_2} \|t_{ij} \|_\infty^{2k} \Bigg)^{1/2k} = (m_1 m_2)^{1/2k}.
\end{align*}
Then applying \cite[Theorem 4.9]{latala2018dimension} yields the following bound on $\mathbb{E}\left[\|\bm{T}\|^{2k}\right]$:
\begin{align*}
    \mathbb{E}\left[\|\bm{T}\|^{2k}\right]^{1/2k} & \leq \sigma_{k, 1} + \sigma_{k, 2} + C \sqrt{k} \sigma_{k}^* \leq m_2^{1/2k}\sqrt{m_1q} + m_1^{1/2k}\sqrt{m_2q} + C\sqrt{k}(m_1m_2)^{1/2k} \\
    & \leq n^{1/2k} \left( \sqrt{m_1q} + \sqrt{m_2q} \right) + C\sqrt{k}n^{1/k}
\end{align*}
for some universal constant $C > 0$. Furthermore, by setting $k = \lceil \gamma \log n \rceil$ for some $\gamma > 0$,
\begin{equation*}
     \mathbb{E}\left[\|\bm{T}\|^{2k}\right]^{1/2k} \leq e^{1/2\gamma} \left(\sqrt{m_1q} + \sqrt{m_2q}\right) + C^\prime \sqrt{\log n}
\end{equation*}
for some $C^\prime > 0$. Finally, by putting all the results together and using Markov inequality,
\begin{align*}
     \mathbb{P}\{ \| \bm{S} \| \geq t \} & =  \mathbb{P}\{ \| \bm{S} \|^{2k} \geq t^{2k} \} \leq t^{-2k}     \mathbb{E}\left[\T{\left((\bm{S} \bm{S}^\top)^{k}\right)}\right]  \\
    &\leq d \min\{m_1, m_2 \}\Bigg( \frac {e^{1/2\gamma} (\sqrt{m_1q} + \sqrt{m_2q}) + C^\prime \sqrt{\log n}}{t} \Bigg)^{2\lceil \gamma \log n \rceil}, 
\end{align*}
for any $t > 0$. By setting $\gamma = \frac{c}{2} + 1$, for any $c > 0$, we can identify some $c_1, c_2 > 0$ such that $t = c_1(\sqrt{m_1 q} + \sqrt{m_2q}) + c_2 \sqrt{\log n}$ and $\mathbb{P}\{\|\bm{S}\| \geq t \} \leq n^{-c}$, which completes the proof.
\end{proof}

\begin{remark}
As a special case, when $d = 1$, the matrix $\bm{S}$ defined in \eqref{eq:S_ij_definition} reduces to a random matrix with i.i.d. Bernoulli random variables as its entries, which are also sub-Gaussian. In this case, our bound in Theorem~\ref{lemma:spec_bound_S_12} is equivalent to the previous result in literature~(e.g.,~\cite[Theorem 4.4.5]{vershynin2018high}).
\end{remark}

\begin{remark}
It is worth noting that our result in Theorem~\ref{lemma:spec_bound_S_12} is sharper than either the one by using \textit{$\epsilon$-net argument}~(e.g. \cite[Theorem 4.6.1]{vershynin2018high}) or the one by applying matrix concentration inequalities such as matrix Bernstein~(e.g. \cite[Lemma 5.14]{ling2020near}) and the non-commutative Khintchine~(NCK)~\cite{pisier2003introduction}. If using $\epsilon$-net argument, we will get $\|\bm{S}\| \leq c_1 \sqrt{qnd} + c_2 \sqrt{\log n}$ with high probability for some $c_1, c_2 > 0$, which has an additional factor of $\sqrt{d}$ in the leading term compared to our result. If using matrix Bernstein or NCK, we obtain $\|\bm{S}\| \leq c\sqrt{qn\left(\log nd\right)}$ with high probability for some $c > 0$, which is loose by a factor $\sqrt{\log nd }$ compared to our result. These bounds are not adequate enough in our analysis and thus it is necessary to have a sharp result as in Theorem~\ref{lemma:spec_bound_S_12}.
\end{remark}

\begin{remark}
A recent work~\cite{bandeira2021spectral} studies the spectral norm of a class of random block matrices, called random lifts of matrices, and uses the moment method to obtain the spectral bound of the corresponding random block matrix, which improves the bound provided by NCK inequality. The random block matrix model in~\cite{bandeira2021spectral} is different from ours in the following aspects: (1) based on \cite[Definition 1.2]{bandeira2021spectral}, the matrix $\bm{A}$ before lifting is a deterministic symmetric matrix, while in our model, each rotation matrix is multiplied by a Bernoulli random variable (see the definition of $\bm{S}_{ij}$ in~\eqref{eq:S_ij_definition}), and (2) each sub-block matrix $\Pi_{ij}$ in~\cite[Definition 1.2]{bandeira2021spectral} is assumed to be a symmetric matrix drawn from a centered distribution $\pi$ with spectrum norm at most 1, while we consider rotation matrix drawn according to the Haar measure. As a result, we cannot directly apply the results  in~\cite{bandeira2021spectral} and need to use a different approach to bound the trace of the moments (see \eqref{eq:trace_sum} and the derivation below). 
\end{remark}

\section{Proof of main results}
\label{sec:proof}

This section is devoted to the proofs of theorems in Section~\ref{sec:method}. 
In Sections~\ref{sec:proof_two_equal} and \ref{sec:proof_two_unequal} we
provide detailed proofs for the lemmas and theorems presented in Sections~\ref{sec:sketch_proof_equal_main} and \ref{sec:proof_sketch_uneuqal_main} respectively. In Section~\ref{sec:proof_two_unknown} we consider the case of two clusters with unknown cluster sizes.

\subsection{Two equal-sized clusters with known cluster sizes}
\label{sec:proof_two_equal}
\begin{proof}[Proof of Lemma~\ref{lemma:1_main}]
The optimality is immediately established since \eqref{eq:SDP_equal} is convex then KKT conditions are sufficient for $\bm{M}^*$ being optimal~\cite{boyd2004convex}. For the uniqueness, suppose there exists another optimal solution $\widetilde{\bm{M}} \neq \bm{M}^*$, then the following holds:
\begin{align}
    0 \overset{(a)}{=} 
    \langle \bm{A}, \widetilde{\bm{M}} - \bm{M}^* \rangle &\overset{(b)}{=} -\langle\bm{\Lambda}, \widetilde{\bm{M}} - \bm{M}^*\rangle -  \langle\mathrm{diag}(\{\bm{Z}_i\}_{i = 1}^n), \widetilde{\bm{M}} - \bm{M}^*\rangle + \langle \bm{\Theta} + \bm{\Theta}^\top, \widetilde{\bm{M}} - \bm{M}^*\rangle \nonumber\\
    &\overset{(c)}{=} -\langle\bm{\Lambda}, \widetilde{\bm{M}} - \bm{M}^*\rangle  + \langle \bm{\Theta} + \bm{\Theta}^\top, \widetilde{\bm{M}} - \bm{M}^*\rangle \nonumber\\
    &\overset{(d)}{=} -\langle\bm{\Lambda}, \widetilde{\bm{M}} \rangle  + \langle \bm{\Theta} + \bm{\Theta}^\top, \widetilde{\bm{M}} - \bm{M}^*\rangle \nonumber\\
    &\overset{(e)}{=} -\langle\bm{\Lambda}, \widetilde{\bm{M}} \rangle  + 2\langle \bm{\Theta}, \widetilde{\bm{M}} - \bm{M}^*\rangle,
    \label{eq:relation_tilde_M_M}
\end{align}
where ($a$) holds since both $\widetilde{\bm{M}}$ and $\bm{M}^*$ are optimal and they should share the same primal value, i.e., $\langle \bm{A}, \widetilde{\bm{M}}\rangle = \langle \bm{A}, \bm{M}^* \rangle$; ($b$) follows from \eqref{eq:KKT_stationarity_main}; ($c$) comes from the constraint in \eqref{eq:SDP_equal} that $\bm{M}_{ii}^* = \widetilde{\bm{M}}_{ii} = \bm{I}_d$, for $i = 1, \ldots, n$; ($d$) uses $\langle\bm{\Lambda},\; \bm{M}^*\rangle = 0$ in \eqref{eq:KKT_complementary_main}; and ($e$) holds because both $\widetilde{\bm{M}}$ and $\bm{M}^*$ are symmetric. To proceed, let us rewrite
\begin{equation}
    \begin{aligned}
    \langle \bm{\Theta}, \widetilde{\bm{M}} - \bm{M}^*\rangle &= \langle \bm{\Theta}, \widetilde{\bm{M}}\rangle - \langle \bm{\Theta}, \bm{M}^*\rangle = \sum \limits_{i,j} \langle\bm{\Theta}_{ij}, \widetilde{\bm{M}}_{ij} \rangle - \sum \limits_{i,j} \langle\bm{\Theta}_{ij}, \bm{M}_{ij}^* \rangle.
    \end{aligned}
    \label{eq:theta_M_tilde_M}
\end{equation}
Then by plugging the explicit expression of $\bm{\Theta}_{ij}$ in \eqref{eq:theta_definition_main} into \eqref{eq:theta_M_tilde_M} we obtain
\begin{equation}
    \begin{aligned}
    \sum \limits_{i,j} \langle\bm{\Theta}_{ij}, \bm{M}_{ij}^* \rangle &\overset{(a)}{=} \sum \limits_{i}\sum \limits_{j: \kappa(j) = \kappa(i)}\frac{\mu_i}{\sqrt{d}}\langle \bm{M}_{ij}^*, \bm{M}_{ij}^* \rangle = \sum \limits_{i} \mu_im\sqrt{d},\\
    \sum \limits_{i,j}\langle\bm{\Theta}_{ij}, \widetilde{\bm{M}}_{ij} \rangle &\overset{(b)}{=} \sum \limits_{i}\mu_i\bigg(\sum \limits_{j: \kappa(j) = \kappa(i)} \frac{1}{\sqrt{d}}\langle \bm{M}_{ij}^*, \widetilde{\bm{M}}_{ij}\rangle + \sum \limits_{j: \kappa(j) \neq \kappa(i)} \langle \bm{\alpha}_{ij}, \widetilde{\bm{M}}_{ij}\rangle \bigg) \\
    &\overset{(c)}{\leq} \sum \limits_{i}\mu_i \bigg(\sum \limits_{j: \kappa(j) = \kappa(i)} \frac{1}{\sqrt{d}}\|\bm{M}_{ij}^*\|_\mathrm{F}\|\widetilde{\bm{M}}_{ij}\|_\mathrm{F} + \sum \limits_{j: \kappa(j) \neq \kappa(i)} \|\bm{\alpha}_{ij}\|_\mathrm{F}\|\widetilde{\bm{M}}_{ij}\|_\mathrm{F} \bigg)\\
    &\overset{(d)}{\leq}  \sum \limits_{i}\mu_i\sum \limits_{j}\|\widetilde{\bm{M}}_{ij}\|_\mathrm{F} \overset{(e)}{\leq} \sum \limits_{i}\mu_im\sqrt{d}
    \end{aligned}
    \label{eq:theta_M_M_tilde_bound}
\end{equation}
where both ($a$) and ($b$) hold because $\bm{M}^*_{ij} = \bm{0}, \; \kappa(i) \neq \kappa(j)$; ($c$) follows from Cauchy-Schwartz inequality; ($d$) uses $\|\bm{\alpha}_{ij}\|_\textrm{F} \leq 1$ in \eqref{eq:theta_definition_main}; ($e$) comes from $\sum_{j}\|\widetilde{\bm{M}}_{ij}\|_\mathrm{F} \leq m\sqrt{d}$, for $i = 1, \ldots, n$ in \eqref{eq:SDP_equal}. As a result, combining \eqref{eq:theta_M_tilde_M} and \eqref{eq:theta_M_M_tilde_bound} yields $\langle \bm{\Theta}, \widetilde{\bm{M}} - \bm{M}^*\rangle \leq 0$. Plugging this back into \eqref{eq:relation_tilde_M_M} gives $\langle \bm{\Lambda}, \widetilde{\bm{M}}\rangle \leq 0$. On the other hand, from $\bm{\Lambda} \succeq 0$ in \eqref{eq:KKT_dual_feasibility_main} and $\widetilde{\bm{M}} \succeq 0$ in \eqref{eq:SDP_equal} we have $\langle \bm{\Lambda}, \widetilde{\bm{M}}\rangle \geq 0$. Then it holds that $\langle \bm{\Lambda}, \widetilde{\bm{M}}\rangle = 0$. With this in mind, combining the assumption $\mathcal{N}(\bm{\Lambda}) = \mathcal{R}(\bm{M}^*)$ with $\bm{M}^* = \bm{V}^{(1)}(\bm{V}^{(1)})^{\top} + \bm{V}^{(2)}(\bm{V}^{(2)})^{\top}$ in \eqref{eq:M_ground_truth}, $\widetilde{\bm{M}}$ satisfies
\begin{equation*}
    \widetilde{\bm{M}} = \bm{V}^{(1)}\bm{\Sigma}_1(\bm{V}^{(1)})^{\top} + \bm{V}^{(2)}\bm{\Sigma}_2(\bm{V}^{(2)})^{\top}
\end{equation*}
for some $d \times d$ diagonal matrices $\bm{\Sigma}_1, \bm{\Sigma}_2  \succeq 0$. Furthermore, from the constraint in \eqref{eq:SDP_equal} that $\widetilde{\bm{M}}_{ii} = \bm{I}_d$, we obtain $\bm{\Sigma}_1 = \bm{\Sigma}_2 = \bm{I}_d$, then $\widetilde{\bm{M}} = \bm{V}^{(1)}(\bm{V}^{(1)})^{\top} + \bm{V}^{(2)}(\bm{V}^{(2)})^{\top} = \bm{M}^*$ which contradicts with the assumption $\widetilde{\bm{M}} \neq \bm{M}^*$. Therefore, $\bm{M}^*$ is the unique solution. 
\end{proof}

\begin{proof}[Proof of Lemma~\ref{lemma:guess_dual}]
Recall the assumption $\bm{R}_i = \bm{I}_d, \; i = 1,\ldots, n$.
From \eqref{eq:KKT_stationarity_main} we get $\bm{\Lambda} = -\bm{A} - \mathrm{diag} \left (\{\bm{Z}_i \}_{i = 1}^n \right) + \bm{\Theta}  + \bm{\Theta}^{\top}$, plugging this into \eqref{eq:1_main} yields
\begin{align}
&\sum \limits_{j: \kappa(j) = \kappa(i)} \bm{\Theta}_{ij} + \bm{\Theta}_{ji}^{\top} -\bm{A}_{ij} = \bm{Z}_i,  \quad i = 1, \ldots, n \label{eq:Z_equality} \\
&\sum \limits_{j: \kappa(j) \neq \kappa(i)} \bm{\Theta}_{ij} + \bm{\Theta}_{ji}^{\top} -\bm{A}_{ij} = \bm{0}, \quad i = 1, \ldots, n.
\label{eq:Theta_equality}
\end{align}
To determine $\bm{\Theta}_{ij}$, notice that \eqref{eq:theta_definition_main} can be written as
\begin{equation}
    \begin{cases} \displaystyle
    \bm{\Theta}_{ij} = \mu_i\bm{M}_{ij}^*/\sqrt{d}, &\;  \kappa(i) = \kappa(j),\\[4pt]
    \|\bm{\Theta}_{ij}\|_\mathrm{F} \leq mu_i  &\;  \kappa(i) \neq \kappa(j),
    \end{cases}
    \label{eq:theta_equal}
\end{equation}
where we use the fact $\bm{M}_{ij}^* \neq \bm{0}$ if and only if $\kappa(i) = \kappa(j)$. Then let us consider the case when $\kappa(i) \neq \kappa(j)$. Importantly, our ansatz of $\bm{\Theta}_{ij}$ is of the following form:
\begin{equation}
    \bm{\Theta}_{ij} = \widetilde{\bm{\alpha}}_{i}, \quad  \forall j: \kappa(j) \neq \kappa(i),
    \label{eq:guess_theta}
\end{equation}
for some $\widetilde{\bm{\alpha}}_{i}$. To solve it, plugging \eqref{eq:guess_theta} into \eqref{eq:Theta_equality} yields
\begin{align}
    \widetilde{\bm{\alpha}}_i &= \frac{1}{m}\sum \limits_{j \in C_2}\bm{A}_{ij} - \frac{1}{m}\sum \limits_{j \in C_2}\widetilde{\bm{\alpha}}_j^\top, \quad i \in C_1,\label{eq:alpha_i_alpha_j_1}\\
    \widetilde{\bm{\alpha}}_j &= \frac{1}{m}\sum \limits_{s \in C_1}\bm{A}_{js} - \frac{1}{m}\sum \limits_{s \in C_1}\widetilde{\bm{\alpha}}_s^\top, \quad j \in C_2.
    \label{eq:alpha_i_alpha_j_2}
\end{align}
By combining \eqref{eq:alpha_i_alpha_j_1} and \eqref{eq:alpha_i_alpha_j_2} we get
\begin{align*}
    \widetilde{\bm{\alpha}}_i - \frac{1}{m}\sum \limits_{s \in C_1}\widetilde{\bm{\alpha}}_s = \frac{1}{m}\sum \limits_{j \in C_2}\bm{A}_{ij} - \frac{1}{m^2}\sum \limits_{s \in C_1}\sum \limits_{j \in C_2}\bm{A}_{sj}, \quad i \in C_1, \\
    \widetilde{\bm{\alpha}}_j - \frac{1}{m}\sum \limits_{s \in C_2}\widetilde{\bm{\alpha}}_s = \frac{1}{m}\sum \limits_{i \in C_1}\bm{A}_{ji} - \frac{1}{m^2}\sum \limits_{s \in C_2}\sum \limits_{i \in C_1}\bm{A}_{si}, \quad j \in C_2.
\end{align*}
Then it is natural to guess $\widetilde{\bm{\alpha}}_i$ for $i \in C_1$ and $\widetilde{\bm{\alpha}}_j$ for $j \in C_2$ has the form
\begin{equation*}
    \widetilde{\bm{\alpha}}_{i} = \frac{1}{m}\sum_{s \in C_2} \bm{A}_{is} - \bm{\alpha}^{(1)}, \quad i \in C_1, \quad \widetilde{\bm{\alpha}}_{j} = \frac{1}{m}\sum_{s \in C_2} \bm{A}_{js} - \bm{\alpha}^{(2)}, \quad j \in C_2
\end{equation*}
for some $\bm{\alpha}^{(1)}$ and $\bm{\alpha}^{(2)}$. By further plugging this into \eqref{eq:alpha_i_alpha_j_1} or \eqref{eq:alpha_i_alpha_j_2} we obtain 
\begin{equation*}
    \bm{\alpha}^{(1)} + (\bm{\alpha}^{(2)})^\top = \frac{1}{m^2}\sum_{s_1 \in C_1}\sum_{s_2 \in C_2} \bm{A}_{s_1s_2}.  
\end{equation*}
By symmetry we guess $\bm{\alpha}^{(1)} = (\bm{\alpha}^{(2)})^\top = \frac{1}{2m^2}\sum_{s_1 \in C_1} \sum_{s_2 \in C_2} \bm{A}_{s_1s_2}$, which leads to our guess of $\widetilde{\bm{\alpha}}_i$ given in Lemma~\ref{lemma:guess_dual}.
Now it remains to determine $\mu_i$: From \eqref{eq:theta_equal} and \eqref{eq:guess_theta} we see that when $\kappa(i) \neq \kappa(j)$, $\|\bm{\Theta}_{ij}\|_\mathrm{F} = \mu_i\|\bm{\alpha}_{ij}\|_\mathrm{F} = \|\widetilde{\bm{\alpha}}_{i}\|_\mathrm{F}$. Also, 
due to the constraint in \eqref{eq:theta_definition_main} that $\|\bm{\alpha}_{ij}\|_\mathrm{F} \leq 1$, $\mu_i$ should satisfy $\mu_i \geq \|\widetilde{\bm{\alpha}}_i\|_\textrm{F}$. Then we guess $\mu_i = \|\widetilde{\bm{\alpha}}_{i}\|_\mathrm{F}$.
As we shall see in Lemma~\ref{lemma:simplify_lambda}, this is to make $\mu_i$ as small as possible such that the eigenvalues of $\bm{\Lambda}$ are greater and the condition $\bm{\Lambda} \succeq 0$ benefits from it. 
To complete the proof, plugging $\bm{\Theta}_{ij}$ in \eqref{eq:theta_equal_new} with $\bm{M}_{ij}^* = \bm{I}_d$ into \eqref{eq:Z_equality} yields the form of $\bm{Z}$.
\end{proof}

\begin{proof}[Proof of Lemma~\ref{lemma:simplify_lambda}]
Let us rewrite $\bm{A}$, $\bm{\Lambda}$ and $\bm{\Theta}$ into four blocks as
\begin{equation}
    \bm{A} = 
    \begin{pmatrix}
    \widetilde{\bm{A}}_{11} &\widetilde{\bm{A}}_{12} \\
    \widetilde{\bm{A}}_{12}^{\top} &\widetilde{\bm{A}}_{22} \\
    \end{pmatrix}, \quad
    \bm{\Lambda} = 
    \begin{pmatrix}
    \bm{\Lambda}_{11} &\bm{\Lambda}_{12}\\
    \bm{\Lambda}_{21} &\bm{\Lambda}_{22}
    \end{pmatrix}, \quad
    \bm{\Theta} = 
    \begin{pmatrix}
    \widetilde{\bm{\Theta}}_{11} &\widetilde{\bm{\Theta}}_{12}\\
    \widetilde{\bm{\Theta}}_{21} &\widetilde{\bm{\Theta}}_{22}
    \end{pmatrix}
    \label{eq:2_cluster_A_M}
\end{equation}
where the two diagonal blocks represent $C_1$ and $C_2$ respectively. By assumption the two diagonal blocks of $\bm{M}^*$ in \eqref{eq:M_ground_truth} satisfy $\widetilde{\bm{M}}^*_{1}\! =\! \widetilde{\bm{M}}^*_{2}\! =\! (\bm{1}_m\bm{1}_m^\top) \otimes \bm{I}_d$. Let $\widetilde{\bm{\mu}}_{1} = \mathrm{diag}(\{\mu_{i}\bm{I}_d\}_{i = 1}^m) \in \mathbb{R}^{md \times md}$ and $\widetilde{\bm{\mu}}_{2} = \mathrm{diag}(\{\mu_{i}\bm{I}_d\}_{i = m+1}^n) \in \mathbb{R}^{md \times md}$ be two diagonal block matrices whose diagonal blocks are  $\mu_i\bm{I}_d$ for $i \in C_1$ and $i \in C_2$ respectively.
Then from Lemma~\ref{lemma:guess_dual} each block of $\bm{\Theta}$ can be expressed as
\begin{alignat*}{2}
        &\widetilde{\bm{\Theta}}_{11} = \frac{\widetilde{\bm{\mu}}_1\widetilde{\bm{M}}_{1}^*}{\sqrt{d}}, &&\quad  \widetilde{\bm{\Theta}}_{12}
        = \frac{1}{m}\widetilde{\bm{A}}_{12}\widetilde{\bm{M}}_{2}^* - \frac{1}{2m^2}\widetilde{\bm{M}}_{1}^*\widetilde{\bm{A}}_{12}\widetilde{\bm{M}}_{2}^*, \\
        &\widetilde{\bm{\Theta}}_{21}  = \frac{1}{m}\widetilde{\bm{A}}_{21}\widetilde{\bm{M}}_{1}^* - \frac{1}{2m^2}\widetilde{\bm{M}}_{2}^*\widetilde{\bm{A}}_{21}\widetilde{\bm{M}}_{1}^*,&&\quad
        \widetilde{\bm{\Theta}}_{22} = \frac{\widetilde{\bm{\mu}}_2\widetilde{\bm{M}}_{2}^*}{\sqrt{d}}.
\end{alignat*}
Given the above, from Lemma~\ref{lemma:guess_dual}, each block of $\bm{\Lambda}$ can be expressed as
\begin{equation*}
    \begin{aligned}
    \widetilde{\bm{\Lambda}}_{12} &= -\widetilde{\bm{A}}_{12} + \widetilde{\bm{\Theta}}_{12} + \widetilde{\bm{\Theta}}_{21}^{\top} 
    = -\left(\bm{I}_{md} - \frac{1}{m}\widetilde{\bm{M}}_{1}^*\right)\widetilde{\bm{A}}_{12}\left(\bm{I}_{md} - \frac{1}{m}\widetilde{\bm{M}}_{2}^*\right),\\
    \widetilde{\bm{\Lambda}}_{21} &= -\widetilde{\bm{A}}_{21} + \widetilde{\bm{\Theta}}_{21} + \widetilde{\bm{\Theta}}_{12}^{\top} 
    = -\left(\bm{I}_{md} - \frac{1}{m}\widetilde{\bm{M}}_{2}^*\right)\widetilde{\bm{A}}_{21}\left(\bm{I}_{md} - \frac{1}{m}\widetilde{\bm{M}}_{1}^*\right),\\
    \widetilde{\bm{\Lambda}}_{11} &= -\widetilde{\bm{A}}_{11} - \mathrm{diag}(\{\bm{Z}_i\}_{i = 1}^m) + \frac{1}{\sqrt{d}}\widetilde{\bm{\mu}}_1\widetilde{\bm{M}}_{1}^* + \frac{1}{\sqrt{d}}\widetilde{\bm{M}}_{1}^*\widetilde{\bm{\mu}}_1^{\top}\\
    &= -\widetilde{\bm{A}}_{11} - \mathrm{diag}\bigg(\bigg\{\bigg(\frac{m\mu_i}{\sqrt{d}} +\sum \limits_{s \in C_1}\bigg(\frac{\mu_s}{\sqrt{d}} - r_{is}\bigg)\bigg)\bm{I}_d\bigg\}_{i = 1}^m\bigg) + \frac{1}{\sqrt{d}}\widetilde{\bm{\mu}}_1\widetilde{\bm{M}}_{1}^* + \frac{1}{\sqrt{d}}\widetilde{\bm{M}}_{1}^*\widetilde{\bm{\mu}}_1^{\top}\\
    &= \left(\bm{I}_{md} - \frac{1}{m}\widetilde{\bm{M}}_{1}^*\right)\bigg(-\widetilde{\bm{A}}_{11} - \mathrm{diag}\left(\{\bm{Z}_{i}\}_{i = 1}^m\right) \bigg)\left(\bm{I}_{md} - \frac{1}{m}\widetilde{\bm{M}}_{1}^*\right),\\
    \widetilde{\bm{\Lambda}}_{22} &= \left(\bm{I}_{md} - \frac{1}{m}\widetilde{\bm{M}}_{2}^*\right)\bigg(-\widetilde{\bm{A}}_{22} - \mathrm{diag}\left(\{\bm{Z}_{i}\}_{i = m+1}^n\right) \bigg)\left(\bm{I}_{md} - \frac{1}{m}\widetilde{\bm{M}}_{2}^*\right).
    \end{aligned}
\end{equation*}
Then, $\bm{\Lambda}$ can be simply expressed as
\begin{equation}
    \bm{\Lambda} = (\bm{I}_{nd} - \bm{\Pi})
    \left(-\bm{Z} - \bm{A}\right)
    (\bm{I}_{nd} - \bm{\Pi}).
    \label{eq:lambda_first}
\end{equation}
To proceed, notice that the expectation of $\bm{A}$ is given as
\begin{equation*}
    \mathbb{E}[\bm{A}] = 
    \begin{pmatrix}
        \mathbb{E}[\widetilde{\bm{A}}_{11}] &\mathbb{E}[\widetilde{\bm{A}}_{12}]\\
        \mathbb{E}[\widetilde{\bm{A}}_{21}] &\mathbb{E}[\widetilde{\bm{A}}_{22}]
    \end{pmatrix}
    = 
    \begin{pmatrix}
        p\widetilde{\bm{M}}_{1}^* - p\bm{I}_{md} &\bm{0}\\
        \bm{0} &p\widetilde{\bm{M}}_{2}^* - p\bm{I}_{md}
    \end{pmatrix}
    = 
    \begin{pmatrix}
    p\widetilde{\bm{M}}_{1}^* &\bm{0}\\
    \bm{0} &p\widetilde{\bm{M}}_{2}^*
    \end{pmatrix} 
    - p\bm{I}_{nd},
    \label{eq:exp_A}
\end{equation*}
where $p\bm{I}_{nd}$ comes from the convention $\bm{A}_{ii} = \bm{0}, \; i = 1,\ldots, n$. Then one can see that
\begin{equation}
    \begin{aligned}
    (\bm{I}_{nd} - \bm{\Pi})
    (\mathbb{E}[\bm{A}] + p\bm{I}_{nd})
    (\bm{I}_{nd} - \bm{\Pi}) = \bm{0}.
    \end{aligned}
    \label{eq:project_zero}
\end{equation}
Therefore, by plugging \eqref{eq:project_zero} into \eqref{eq:lambda_first} we obtain $\bm{\Lambda} = (\bm{I}_{nd} - \bm{\Pi})\widetilde{\bm{\Lambda}} (\bm{I}_{nd} - \bm{\Pi})$ with $\widetilde{\bm{\Lambda}}$ defined in Lemma~\ref{lemma:guess_dual}. 
This yields two observations: (1) $\bm{\Lambda} \succeq 0$ is satisfied if $\widetilde{\bm{\Lambda}}\succeq 0$; (2)  $\bm{I}_{nd} - \bm{\Pi}$ is a projection matrix such that $\mathcal{N}(\bm{I}_{nd} - \bm{\Pi}) = \mathcal{R}(\bm{M}^*)$, which implies $\mathcal{N}(\bm{\Lambda}) \supseteq \mathcal{R}(\bm{M}^*)$. When $\widetilde{\bm{\Lambda}} \succ 0$, for all $\bm{x} \notin \mathcal{R}(\bm{M}^*)$,  $\bm{\Lambda} \bm{x} \neq \bm{0}$ since $ \bm{x} \notin \mathcal{N}(\bm{I}_{nd} - \bm{\Pi})$. Combining it with the fact that $\mathcal{N}(\bm{\Lambda}) \supseteq \mathcal{R}(\bm{M}^*)$, we have $\mathcal{N}(\bm{\Lambda}) = \mathcal{R}(\bm{M}^*)$ if  $\widetilde{\bm{\Lambda}} \succ 0$. As a result, $\widetilde{\bm{\Lambda}} \succ 0$ ensures both $\bm{\Lambda} \succeq 0$ and $\mathcal{N}(\bm{\Lambda}) = \mathcal{R}(\bm{M}^*)$ are satisfied.
\end{proof}

\begin{proof}[Proof of Lemma~\ref{claim:x_i_y_i_equal}]
According to Lemma~\ref{lemma:simplify_lambda}, notice that $\bm{Z}$ is a diagonal matrix, then $\lambda_{\text{min}}(p\bm{I}_{nd} - \bm{Z})$ turns out to be the smallest diagonal entry of $p\bm{I}_{nd} - \bm{Z}$. This leads to 
\begin{equation*}
    \lambda_{\text{min}}(p\bm{I}_{nd} - \bm{Z}) = \min_i \left (x_i - y_i - \epsilon_{\kappa(i)} \right ) + p\geq \min_i x_i - \max_i y_i - \max\{\epsilon_1, \epsilon_2 \} + p.
\end{equation*}
with $x_i, y_i, \epsilon_1, \epsilon_2$ defined in Lemma~\ref{claim:x_i_y_i_equal}, and we can bound them separately.

For $x_i$, recall that $r_{is}$ is a Bernoulli random variable with $\mathbb{P}\{r_{is} = 1\} = p$. Then $x_i \sim \textrm{Binom}\left(m, p\right)$ that follows a binomial distribution. Applying the existing tail bound in Lemma~\ref{lemma:Bound_bernoulli} with $\rho = 1/2$ and the union bound yields the result for $\min_i x_i$. 

For $y_i$, by definition 
\begin{align*}
    y_i &= \frac{m\|\widetilde{\bm{\alpha}}_i\|_\mathrm{F}}{\sqrt{d}}  = \frac{m}{\sqrt{d}}\bigg\|\frac{1}{m}\sum \limits_{j: \kappa(j) \neq \kappa(i)}\bm{A}_{ij} - \frac{1}{2m^2}\sum \limits_{j: \kappa(j) \neq \kappa(i)}\sum \limits_{s: C(s)= \kappa(i)}\bm{A}_{sj}\bigg\|_\mathrm{F}\\
    &\leq \underbrace{\frac{1}{\sqrt{d}}\bigg\|\sum \limits_{j: \kappa(j) \neq \kappa(i)}\bm{A}_{ij}\bigg\|_\mathrm{F}}_{=: y_{i1}} + \underbrace{\frac{1}{2m\sqrt{d}}\bigg\|\sum \limits_{j: \kappa(j) \neq \kappa(i)}\sum \limits_{s: C(s)= \kappa(i)}\bm{A}_{sj}\bigg\|_\mathrm{F}}_{=: y_{i2}} = y_{i1} + y_{i2}.
\end{align*}
As a result, both $y_{i1}$ and $y_{i2}$ consist of $\|\sum_{i}\bm{X}_i\|$ where each $\bm{X}_i \in \mathbb{R}^{d \times d}$ is i.i.d. and uniformly distributed in $\SO(d)$. This can be bounded by matrix concentration inequalities such as matrix Bernstein~\cite{tropp2015introductionsup} in general. Therefore, by applying Theorem~\ref{lemma:large_deviation_random_orthogonal} we obtain 
\begin{equation}
    \begin{aligned}
    y_{i1} &\leq \sqrt{2qm(c\log n + \log 2d)}\left(\sqrt{1 + \frac{c\log n + \log 2d}{18qm}} + \sqrt{\frac{c\log n + \log 2d}{18qm}}\right)\\
    &\overset{(a)}{\approx} \sqrt{c\beta}\log n
    \end{aligned}
    \label{eq:bound_alpha_1}
\end{equation}
with probability $1 - n^{-c}$ for any $c > 0$, where $(a)$ comes from \eqref{eq:equal_alpha_i1_approx}.
Similarly, by applying Theorem~\ref{lemma:large_deviation_random_orthogonal} on $y_{i2}$ we get $y_{i2} = o(\log n)$ with high probability. Combining the results above and applying the union bound yields the result for $\max_i y_i$. 

For $\epsilon_1$ and $\epsilon_2$, by symmetry they should have the same statistics, so let us focus on $\epsilon_1$, which satisfies
\begin{equation*}
    \begin{aligned}
    \epsilon_{1} &= \frac{1}{\sqrt{d}}\sum \limits_{s \in C_1} \mu_s = \frac{1}{\sqrt{d}}\sum \limits_{s \in C_1} \|\widetilde{\bm{\alpha}}_s\|_\mathrm{F} \leq \underbrace{\frac{1}{m\sqrt{d}}\sum \limits_{s \in C_1}\Bigg\|\sum \limits_{j \in C_2}\bm{A}_{sj}\Bigg\|_\textrm{F}}_{=: \epsilon_{1}^{(1)}} + \underbrace{\frac{1}{2m\sqrt{d}}\Bigg\|\sum \limits_{j \in C_2}\sum \limits_{s \in C_1}\bm{A}_{sj}\Bigg\|_\mathrm{F}}_{=: \epsilon_{1}^{(2)}} \\
    &= \epsilon_{1}^{(1)} + \epsilon_{1}^{(2)}.
    \end{aligned}
\end{equation*}
Then, by applying Lemma~\ref{lemma:bound_concentration_Z_i} on $\epsilon^{(1)}_{1}$ with $m_1 = m_2 = m$ we get $\epsilon^{(1)}_{1} = (3c/2)\log n + o(\log n)$ 
and $\epsilon_{1}^{(2)} = y_{i2} = o(\log n)$ with probability $1 - n^{-c}$, this leads to the bound on $\epsilon_{1}$ and $\epsilon_2$. By further applying the union bound we get the result on $\max\{\epsilon_1, \epsilon_2\}$. 

Specifically when $d = 2$, a tighter bound on $y_i$ is obtained by applying Theorem~\ref{lemma:sum_orth_2} instead of Theorem~\ref{lemma:large_deviation_random_orthogonal} for bounding $y_{i1}$ and $y_{i2}$, which yields $y_{i1} \leq  \sqrt{c\beta/2}\log n + o(\log n)$ and $y_{i2} = o(\log n)$ with probability $1 - n^{-c}$. The remaining is the same as above.
\end{proof}

\begin{proof}[Proof of Lemma~\ref{lemma:concentration_A}] Let us define 
\begin{equation*}
    \bm{A} = 
    \begin{pmatrix}
    \widetilde{\bm{A}}_{11} &\widetilde{\bm{A}}_{12} \\
    \widetilde{\bm{A}}_{12}^{\top} &\widetilde{\bm{A}}_{22} \\
    \end{pmatrix}, \;
    \mathbb{E}\left[\bm{A}\right] - \bm{A} =:
    \begin{pmatrix}
    \bm{S}_{11} &\bm{S}_{12}\\
    \bm{S}_{12}^{\top} &\bm{S}_{22}
    \end{pmatrix}, \; \bm{S}_{\text{in}} := \begin{pmatrix}
    \bm{S}_{11} &\bm{0}\\
    \bm{0} &\bm{S}_{22}
    \end{pmatrix}, \; \bm{S}_{\text{out}} := \begin{pmatrix}
    \bm{0} &\bm{S}_{12}\\
    \bm{S}_{12}^{\top} &\bm{0}
    \end{pmatrix},
\end{equation*}
where the two diagonal blocks represent the two clusters, then by triangle inequality $\|\mathbb{E}[\bm{A}] - \bm{\bm{A}}\| = \|\bm{S}_{\text{in}} + \bm{S}_{\text{out}}\| \leq \|\bm{S}_{\text{in}}\| + \|\bm{S}_{\text{out}}\|.$
Moreover, notice $\|\bm{S}_{\text{in}}\| = \max\left\{\|\bm{S}_{11}\|, \; \|\bm{S}_{22}\|\right\} \leq \|\bm{S}_{11}\| + \|\bm{S}_{22}\|$ and $\|\bm{S}_{\text{out}}\| = \|\bm{S}_{12}\|$, it holds
\begin{equation*}
    \|\mathbb{E}[\bm{A}] - \bm{\bm{A}}\| \leq \|\bm{S}_{11}\| + \|\bm{S}_{22}\| + \|\bm{S}_{12}\|.
\end{equation*}
As a result, we can bound $\bm{S}_{11}$, $\bm{S}_{22}$ and $\bm{S}_{12}$ them separately. For $\bm{S}_{11}$, by definition $\bm{S}_{11} = \mathbb{E}[\widetilde{\bm{A}}_{11}] - \widetilde{\bm{A}}_{11}$. Then if we define $\bm{E}_1 \in \mathbb{R}^{m_1 \times m_1}$ be the adjacency matrix of the subgraph on $C_1$ which is generated from the Erd{\"o}s–R{\'e}nyi model $\mathcal{G}(m_1, p)$,   $\widetilde{\bm{A}}_{11}$ can be expressed as $\widetilde{\bm{A}}_{11} = \bm{E}_{1} \otimes \bm{I}_d$ under the assumption $\bm{R}_i = \bm{I}_d$ for $i = 1,\ldots, n$, where $\otimes$ denotes the tensor product. Therefore, we have
\begin{equation*}
    \|\bm{S}_{11}\| = \|\mathbb{E}[\widetilde{\bm{A}}_{11}] - \widetilde{\bm{A}}_{11}\| =  \|(\mathbb{E}[\bm{E}_{1}] - \bm{E}_{1})\otimes \bm{I}_d\| = \| \mathbb{E}[\bm{E}_{1}] - \bm{E}_{1}\|.
\end{equation*}
This enables us to study $\|\bm{S}_{11}\|$ via $\| \mathbb{E}[\bm{E}_{1}] - \bm{E}_{1}\|$, which is bounded by Theorem~\ref{lemma:A_graph} as $\|\bm{S}_{11}\| \leq c_1\sqrt{pm_1}$
with probability $1 - n^{-c}$ for some $c_1, c > 0$. $\|\bm{S}_{22}\|$ is bounded in a similar manner and we do not repeat. For $\|\bm{S}_{12}\|$, we apply Theorem~\ref{lemma:spec_bound_S_12} and get $\|\bm{S}_{12}\| \leq c_2(\sqrt{qm_1} + \sqrt{qm_2}) + O(\sqrt{\log n})$ with probability $1 - n^{-c}$ for some $c_2, c > 0$. Combining these completes the proof.
\end{proof}

\subsection{Two unequal-sized clusters with known cluster sizes}
\label{sec:proof_two_unequal}
\begin{proof}[Proof of Lemma~\ref{lemma:uniqueness_unequal_new}]
The proof is very similar to the one for Lemma~\ref{lemma:1_main}, with a tiny difference on evaluating $\langle \bm{\Theta}, \widetilde{\bm{M}} - \bm{M}^*_{ij} \rangle$ in \eqref{eq:theta_M_M_tilde_bound}, where in this case $\mu_i$ is replaced by $\mu_i + \nu$. Therefore we do not repeat.
\end{proof}

\begin{proof}[Proof of Lemma~\ref{lemma:guess_dual_unequal_new}]
First, we obtain $\bm{\Lambda} = -\bm{A} - \text{diag}(\{\bm{Z}_i\}_{i = 1}^n) + \bm{\Theta} + \bm{\Theta}^\top$ from  \eqref{eq:KKT_stationarity_unequal_new}, then plugging this into \eqref{eq:unequal_N_R_new} yields
\begin{alignat}{2}
&\sum_{j: \kappa(j) = \kappa(i)} \bm{\Theta}_{ij} + \bm{\Theta}_{ji}^{\top} -\bm{A}_{ij} = \bm{Z}_i,  &&\quad i = 1, \ldots, n \label{eq:Z_equality_new} \\
&\sum_{j: \kappa(j) \neq \kappa(i)} \bm{\Theta}_{ij} + \bm{\Theta}_{ji}^{\top} -\bm{A}_{ij} = \bm{0}, &&\quad i = 1, \ldots, n.
\label{eq:Theta_equality_new}
\end{alignat}
Next, to determine $\bm{\Theta}_{ij}$,  \eqref{eq:KKT_theta_unequal_new} can be rewritten as
\begin{equation}
    \begin{cases} \displaystyle
    \bm{\Theta}_{ij} = (\mu_i + \nu)\bm{I}_d/\sqrt{d}, &\;  \kappa(i) = \kappa(j),\\[2pt]
    \displaystyle \|\bm{\Theta}_{ij}\|_\mathrm{F} \leq \mu_i + \nu,  &\; \kappa(i) \neq \kappa(j).
    \end{cases}
    \label{eq:theta_equal_new}
\end{equation}
Then one can check that our guess of $\bm{\Theta}_{ij}$ in Lemma~\ref{lemma:guess_dual_unequal_new} indeed satisfies \eqref{eq:Theta_equality_new} and \eqref{eq:theta_equal_new}. 
It remains to determine $\nu$ and $\mu_i$ for $i = 1, \ldots, n$: From \eqref{eq:KKT_mu_i_unequal_new} we get $\mu_i = 0, \; \forall i \in C_2$, since $m_1 > m_2$. Then from \eqref{eq:theta_equal_new}, due to the constraint $\|\bm{\alpha}_{ij}\|_\mathrm{F} \leq 1$, $\nu$ and $\mu_i$ satisfy
\begin{align*}
    \nu + \mu_i \geq \|\bm{\Theta}_{ij}\|_\mathrm{F},  \quad \forall i \in C_1 \quad \text{and} \quad \nu \geq \|\bm{\Theta}_{ij}\|_\mathrm{F}, \quad \forall i \in C_2.
\end{align*}
This leads to our guess of $\nu$ and $\mu_i$ in Lemma~\ref{lemma:guess_dual_unequal_new}.
Plugging this back into \eqref{eq:Z_equality_new} we obtain the expression of $\bm{Z}$.
\end{proof}

\begin{proof}[Proof of Lemma~\ref{lemma:simplify_lambda_unequal}]
We follow the same route in Lemma~\ref{lemma:simplify_lambda} by first rewriting $\bm{A}$, $\bm{\Lambda}$ and $\bm{\Theta}$ into four blocks as \eqref{eq:2_cluster_A_M}. Let $\widetilde{\bm{\mu}}_1 := \text{diag}(\left\{\mu_i \bm{I}_d\right\}_{i = 1}^{m_1}) \in \mathbb{R}^{m_1d \times m_1d}$. Then one can see that the four blocks of $\bm{\Theta}$ satisfy 
\begin{alignat*}{2}
        &\widetilde{\bm{\Theta}}_{11} = \frac{(\widetilde{\bm{\mu}}_1 + \nu\bm{I}_{m_1d})\widetilde{\bm{M}}_{1}^*}{\sqrt{d}}, 
        \quad \quad \quad 
        \widetilde{\bm{\Theta}}_{22} = \frac{ \nu\widetilde{\bm{M}}_{2}^*}{\sqrt{d}},\\
        &\widetilde{\bm{\Theta}}_{12} + \widetilde{\bm{\Theta}}_{21}^\top = \frac{1}{m_1}\widetilde{\bm{M}}_1^*\widetilde{\bm{A}}_{21} + \frac{1}{m_2}\widetilde{\bm{A}}_{12}\widetilde{\bm{M}}_{2}^* - \frac{1}{m_1m_2}\widetilde{\bm{M}}_{1}^*\widetilde{\bm{A}}_{12}\widetilde{\bm{M}}_2^*.
\end{alignat*}
By plugging these into the four blocks of $\bm{\Lambda}$,  one can simplify $\bm{\Lambda}$ as (we leave the detail to readers who are interested)
\begin{equation*}
    \begin{aligned}
    \bm{\Lambda} &= (\bm{I}_{nd} - \bm{\Pi})(-\bm{Z} - \bm{A})(\bm{I}_{nd} - \bm{\Pi}) = (\bm{I}_{nd} - \bm{\Pi})(\mathbb{E}[\bm{A}] - \bm{A} - \bm{Z} + p\bm{I}_{nd} )(\bm{I}_{nd} - \bm{\Pi})
    \end{aligned}
\end{equation*}
where we use \eqref{eq:project_zero} that $\left(\bm{I}_{nd} - \bm{\Pi}\right)(\mathbb{E}[\bm{A}] + p\bm{I}_{nd})\left(\bm{I}_{nd} - \bm{\Pi}\right) = \bm{0}$. 
As a result, by using the same argument as in the proof of Lemma~\ref{lemma:simplify_lambda}, one can see that $\widetilde{\bm{\Lambda}} \succ 0$ ensures both $\bm{\Lambda} \succeq 0$ and $\mathcal{N}(\bm{\Lambda}) = \mathcal{R}(\bm{M}^*)$ are satisfied. 
\end{proof}

\begin{proof}[Proof of Lemma~\ref{lemma:bound_nu_mu}] 
We first define the following quantities and provide their individual analysis.
\begin{alignat*}{2}
    &\sigma_{1} := \max_{j \in C_2}\bigg\|\frac{1}{m_1}\sum_{s_1 \in C_1}\bm{A}_{s_1 j}\bigg\|_\mathrm{F}, \quad &&\sigma_2 := \max_{i \in C_1}\bigg\|\frac{1}{m_2}\sum_{s_2 \in C_1}\bm{A}_{is_2}\bigg\|_\mathrm{F}, \\
    &\sigma := \bigg\|\frac{1}{m_1m_2}\sum_{s_1 \in C_1}\sum_{s_2 \in C_2} \bm{A}_{s_1s_2}\bigg\|_\mathrm{F}, \quad && \sigma^{(i)} := \bigg\|\frac{1}{m_2}\sum_{s_2 \in C_2}\bm{A}_{is_2}\bigg\|_\mathrm{F}. 
\end{alignat*}
By applying Lemma~\ref{lemma:large_deviation_random_orthogonal} with the approximation in \eqref{eq:equal_alpha_i1_approx} and the union bound, one can see that with probability $1 - n^{-c}$,  $\sigma_1$, $\sigma_2$ and $\sigma$ are bounded by 
\begin{equation}
    \begin{aligned}
    \sigma_1 \leq  \sqrt{\frac{2(1+c)\beta d}{\rho}}\bigg(\frac{\log n}{n}\bigg),\quad
    \sigma_2 \leq  \sqrt{\frac{2(1+c)\beta d}{(1-\rho)}}\bigg(\frac{\log n}{n}\bigg),
    \end{aligned}
    \label{eq:bound_sigma_1_2}
\end{equation}
and $\sigma = o(\log n/n)$. Then, the dual variable $\nu$  given in Lemma~\ref{lemma:guess_dual_unequal_new} is bounded as
\begin{equation}
    \begin{aligned}
    \nu &\overset{(a)}{\leq} (1-\gamma) \max_{i \in C_2, j \in C_1}\bigg(\bigg\|\frac{1}{m_1}\sum_{s_1 \in C_1}\bm{A}_{is_1}\bigg\|_\mathrm{F} + \bigg\|\frac{1}{m_2}\sum_{s_2 \in C_2}\bm{A}_{s_2j}\bigg\|_\mathrm{F} + \bigg\|\frac{1}{m_1m_2}\sum_{s_1 \in C_1}\sum_{s_2 \in C_2} \bm{A}_{s_1s_2}\bigg\|_\mathrm{F}\bigg) \\
    &\overset{(b)}{=} (1-\gamma) \left(\sigma_1 + \sigma_2 + \sigma\right) \leq (1-\gamma)\sqrt{2(1+c)\beta d}\bigg(\frac{1}{\sqrt{\rho}} + \frac{1}{\sqrt{1-\rho}}\bigg)\bigg(\frac{\log n}{n}\bigg) + o\bigg(\frac{\log n}{n}\bigg)
    \end{aligned}
    \label{eq:bound_nu}
\end{equation}
with probability $1 - n^{-c}$, 
where $(a)$ holds by triangle inequality and $(b)$ uses the fact $\bm{A}_{ij} = \bm{A}_{ji}^\top$. Similarly, we can bound $\max_{j\in C_2} \|\bm{\Theta}_{ij}\|_\mathrm{F}$ in Lemma~\ref{lemma:guess_dual_unequal_new} for any $i \in C_1$ as
\begin{align*}
    \max_{ j\in C_2} \|\bm{\Theta}_{ij}\|_\mathrm{F} &\leq \gamma\bigg(\max_{j \in C_2}\bigg\|\frac{1}{m_1}\sum_{s_1 \in C_1}\bm{A}_{s_1j}\bigg\|_\mathrm{F} \!+\! \bigg\|\frac{1}{m_2}\sum_{s_2 \in C_2}\bm{A}_{is_2}\bigg\|_\mathrm{F} \!+\!  \bigg\|\frac{1}{m_1m_2}\sum_{s_1 \in C_1}\sum_{s_2 \in C_2}\bm{A}_{s_1s_2}\bigg\|_\mathrm{F}\bigg)\\
    &= \gamma\left(\sigma_1 + \sigma + \sigma^{(i)}\right).
\end{align*}
Therefore, $\mu_i + \nu$ given in Lemma~\ref{lemma:guess_dual_unequal_new} satisfies
\begin{equation}
    \begin{aligned}
    \mu_i + \nu &= \max\left\{\max_{j \in C_2}\|\bm{\Theta}_{ij}\|_\mathrm{F}, \; \nu\right\} \leq \max\bigg\{\gamma\left(\sigma_1 + \sigma + \sigma^{(i)}\right), \; (1-\gamma) \left(\sigma_1 + \sigma_2 + \sigma\right)\bigg\}\\
    &\leq \max\left\{\gamma\sigma_1, \; (1-\gamma)(\sigma_1 + \sigma_2)\right\} + \gamma\left(\sigma + \sigma^{(i)}\right).
    \end{aligned}
    \label{eq:bound_mu_s_nu}
\end{equation}

Now we are ready to bound $\lambda_{\text{min}}(p\bm{I}_{nd} - \bm{Z})$. According to Lemma~\ref{lemma:guess_dual_unequal_new}, by defining
\begin{equation}
    \epsilon_1 := \min_{i \in C_1}\bigg[\sum_{s \in C_1}\left(r_{is} - \frac{\mu_s + \nu}{\sqrt{d}}\right) - \frac{m_1(\mu_i +\nu)}{\sqrt{d}}\bigg], \quad \epsilon_2 := \min_{i \in C_2}\sum_{s \in C_2} r_{is} \!-\! \frac{2m_2\nu}{\sqrt{d}},
    \label{eq:unequal_lambda_epsilon_1_2}
\end{equation}
we have $\lambda_{\text{min}}(p\bm{I}_{nd} - \bm{Z}) = \min\{\epsilon_1, \; \epsilon_2\} + p$, and we can bound $\epsilon_1, \epsilon_2$ separately. To this end, let $x_i := \sum_{s:C(s)=\kappa(i)}r_{is}$ and it follows a Binomial distribution such that $x_i \sim \mathrm{Binom}(m_1, p)$ when $i \in C_1$ and $x_i \sim \mathrm{Binom}(m_2, p)$ when $ i \in C_2$. Then, by applying Lemma~\ref{lemma:Bound_bernoulli} and the union bound we obtain 
\begin{align}
    &\mathbb{P}\left\{\min_{i \in C_1} x_i \leq \tau_1 \rho \log n\right\} = n^{ 1-\rho\left(\alpha - \tau_1 \log(\frac{e\alpha}{\tau_1}) + o(1)\right)},
    \label{eq:bound_x_i_C_1}
    \\
    &\mathbb{P}\left\{\min_{i \in C_2} x_i \leq \tau_2 (1 - \rho) \log n\right\} = n^{ 1-(1- \rho)\left(\alpha - \tau_2 \log(\frac{e\alpha}{\tau_2}) + o(1)\right)}.
    \label{eq:bound_x_i_C_2}
\end{align}
for $\tau_1, \tau_2 \in [0, \alpha)$ which satisfy the conditions in \eqref{eq:cond_tau_unequal}.
For $\epsilon_1$, we have
\begin{equation*}
\begin{aligned}
\epsilon_1 &\geq \underbrace{\min_{i \in C_1}\sum_{s \in C_1}r_{is}}_{=: \epsilon_{11}} - \underbrace{\sum_{s\in C_1} \frac{\mu_s + \nu}{\sqrt{d}}}_{=: \epsilon_{12}} - \underbrace{\frac{m_1}{\sqrt{d}}\max_{i \in C_1}(\mu_i + \nu)}_{=: \epsilon_{13}} = \epsilon_{11} - \epsilon_{12} - \epsilon_{13}.
\end{aligned}
\end{equation*}
Then we can bound $\epsilon_{11}, \epsilon_{12}$ and $\epsilon_{13}$ separately. $\epsilon_{11}$ is bounded by \eqref{eq:bound_x_i_C_1}. For $\epsilon_{12}$, by definition,
\begin{align*}
    \epsilon_{12} &\overset{(a)}{\leq} \frac{1}{\sqrt{d}}\sum_{i \in C_1} \max\left\{\gamma\sigma_1, \; (1-\gamma)(\sigma_1 + \sigma_2)\right\} + \gamma\left(\sigma + \sigma^{(i)}\right) \\
    &= \frac{m_1}{\sqrt{d}}\left(\max\left\{\gamma\sigma_1, \; (1-\gamma)(\sigma_1 + \sigma_2)\right\} + \gamma\sigma\right) + \frac{1}{\sqrt{d}} \sum_{i \in C_1} \sigma^{(i)}\\
    &\overset{(b)}{\leq} \max\bigg\{\frac{\gamma}{\sqrt{\rho}},\; (1-\gamma)\bigg(\frac{1}{\sqrt{\rho}} + \frac{1}{\sqrt{1-\rho}}\bigg)\bigg\}\cdot \rho\sqrt{2(1+c)\beta}\log n + \frac{3\gamma c}{2}\log n + o(\log n)
\end{align*}
with probability $1 - n^{-c}$, where $(a)$ uses the bound of $\mu_s + \nu$ in \eqref{eq:bound_mu_s_nu}, $(b)$ comes from \eqref{eq:bound_sigma_1_2} for bounding $\sigma_1$ and $\sigma_2$, and $\frac{1}{\sqrt{d}} \sum_{i \in C_1} \sigma^{(i)}$ is bounded by Lemma~\ref{lemma:bound_concentration_Z_i}. Similarly, $\epsilon_{13}$ can be bounded as
\begin{align*}
    \epsilon_{13} &= \frac{m_1}{\sqrt{d}} \max_{i \in C_1}\left( \max\left\{\gamma\sigma_1, \; (1-\gamma)(\sigma_1 + \sigma_2)\right\} + \gamma\left(\sigma + \sigma^{(i)}\right)\right)\\
    &\overset{(a)}{=} \frac{m_1}{\sqrt{d}}\left( \max\left\{\gamma(\sigma_1+ \sigma_2), \; (1-\gamma)(\sigma_1 + \sigma_2)\right\} + \gamma\sigma\right) \overset{(b)}{=} \frac{m_1}{\sqrt{d}}\left(\gamma(\sigma_1 + \sigma_2) + \gamma\sigma\right)\\
    &\leq \gamma \rho \bigg(\frac{1}{\sqrt{\rho}} + \frac{1}{\sqrt{1-\rho}}\bigg)\sqrt{2(1+c)\beta}\log n + o(\log n)
\end{align*}
with probability $1 - n^{-c}$, where $(a)$ comes from $\max_{i \in C_1}\sigma^{(i)} = \sigma_1$ and $(b)$ holds since $\gamma \geq 1/2$ by assumption. Combining these individual results and applying the union bound yields
\begin{equation*}
    \epsilon_1 \geq \tau_1 \rho \log n - \rho \max\bigg\{\frac{2\gamma}{\sqrt{\rho}} + \frac{\gamma}{\sqrt{1- \rho}}, \frac{1}{\sqrt{\rho}} + \frac{1}{\sqrt{1-\rho}}\bigg\}\sqrt{2(1+c)\beta}\log n \!-\! \frac{3\gamma c}{2} \log n\!-\! o(\log n)
\end{equation*}
with probability $ 1 - n^{-\Omega(1)}$.
By further taking $c \rightarrow 0$ to be close to zero we get 
\begin{align*}
    \epsilon_1 &\geq \tau_1 \rho \log n - \rho \max\bigg\{\frac{2\gamma}{\sqrt{\rho}} + \frac{\gamma}{\sqrt{1- \rho}}, \frac{1}{\sqrt{\rho}} + \frac{1}{\sqrt{1-\rho}}\bigg\}\sqrt{\ell_d\beta}\log n - o(\log n),
\end{align*}
Similarly, by combining \eqref{eq:bound_x_i_C_2} with the bound of $\nu$ in \eqref{eq:bound_nu} and let $c \rightarrow 0$, we get the bound for $\epsilon_2$ as 
\begin{equation*}
    \epsilon_2 \geq \tau_2(1-\rho)\log n - 2(1-\gamma)(1-\rho)\bigg(\frac{1}{\sqrt{\rho}} + \frac{1}{\sqrt{1 - \rho}}\bigg)\sqrt{\ell_d\beta} \log n - o(\log n)
\end{equation*}
with probability $1 - n^{-\Omega(1)}$. Plugging these into $\lambda_{\text{min}}(p\bm{I}_{nd} - \bm{Z}) = \min\{\epsilon_1, \; \epsilon_2\} + p$ completes the proof.

Specifically when $d = 2$, $\sigma_1, \sigma_2$ and $\sigma$ can be bounded by Theorem~\ref{lemma:sum_orth_2} instead of Theorem~\ref{lemma:large_deviation_random_orthogonal}. The remaining proof is same as above, therefore we do not repeat.
\end{proof}

\subsection{Two clusters with unknown cluster sizes}
\label{sec:proof_two_unknown}
Now we consider the SDP proposed in \eqref{eq:SDP_unknown}. Note in this case, we use the normalized ground truth $\bar{\bm{M}}^*$ defined in \eqref{eq:M_ground_truth_unknown}. The proof basically follows Section~\ref{sec:sketch_proof_equal_main}, and we will omit repetitive analysis. 
\paragraph{Step 1: Derive KKT conditions, uniqueness and optimality of $\bm{M}^*$.}
The KKT conditions of \eqref{eq:SDP_unknown} when $\bm{M} = \bar{\bm{M}}^*$ are given as
\begin{alignat}{2}
    &\bullet \; \textit{Stationarity:}  &&-\bm{A} - \bm{\Lambda} - \bm{I}_n \otimes \bm{Z} + \bm{\Theta} + \bm{\Theta}^{\top}  = \bm{0}, \label{eq:KKT_stationarity_unknown} \\
    & &&\bm{\Theta} \!=\! [\bm{\Theta}_{ij}]_{i, j = 1}^n, \quad
    \begin{cases}
    \displaystyle \bm{\Theta}_{ij} \!=\!  \mu_i\bar{\bm{M}}_{ij}^*/\|\bar{\bm{M}}_{ij}^*\|_{\mathrm{F}}, & \bar{\bm{M}}_{ij}^* \!\neq\! \bm{0},\\
    \|\bm{\Theta}_{ij}\|_\mathrm{F} \leq \mu_i,  & \bar{\bm{M}}_{ij}^* \!=\! \bm{0}.
    \end{cases} \label{eq:theta_definition_unknown}\\
    &\bullet \; \textit{Comp. slackness:} \quad  &&\left\langle \bm{\Lambda}, \bar{\bm{M}}^* \right\rangle = 0, \quad \bigg\langle \mu_i, \sqrt{d} - \sum \limits_{j = 1}^n \|\bar{\bm{M}}^*_{ij}\|_{\mathrm{F}}\bigg\rangle = 0, \quad i = 1, \ldots, n. \label{eq:KKT_complementary_unknown}\\
    &\bullet \; \textit{Dual feasibility:}  &&\bm{\Lambda} \succeq 0, \quad \mu_i \geq 0, \quad i = 1,\ldots, n.
    \label{eq:KKT_dual_feasibility_unknown}
\end{alignat}
Again, the following result establishes the uniqueness and optimality of $\bar{\bm{M}}^*$:
\begin{lemma}
Given $\bar{\bm{M}}^*$ defined in \eqref{eq:ground_truth_def_normalize}, suppose there exist dual variables $\bm{\Lambda}$, $\bm{Z}$, and $\bm{\Theta}$ that satisfy the KKT conditions \eqref{eq:KKT_stationarity_unknown} - \eqref{eq:KKT_dual_feasibility_unknown} as well as
\begin{equation}
    \mathcal{N}(\bm{\Lambda}) = \mathcal{R}(\bar{\bm{M}}^*).
    \label{eq:unknown_N_R}
\end{equation}
Then $\bm{M} = \bar{\bm{M}}^*$ is the optimal and unique solution to \eqref{eq:SDP_unknown}. 
\label{lemma:uniqueness_unknown}
\end{lemma}

\paragraph{Step 2: Construct dual variables.}
Again, we follow the assumption in Section~\ref{sec:sketch_proof_equal_main} that $\bm{R}_i = \bm{I}_d$ for each $i$ and $\bm{A}_{ij} = r_{ij}\bm{I}_d$ for any $j$ with $ \kappa(j) = \kappa(i)$. Then we have $\bar{\bm{M}}_{ij}^* = \bm{I}_d/m_1$ for $i, j \in C_1$ and $\bar{\bm{M}}_{ij}^* = \bm{I}_d/m_2$ for $i, j \in C_2$. Also, let us define 
\begin{equation*}
    \sigma_i := \sum \limits_{s: C(s)=\kappa(i)} r_{is}, \quad \sigma^{(1)} := \frac{1}{m_1}\sum \limits_{s_1 \in C_1}\sum \limits_{s_2\in C_1} r_{s_1s_2}, \quad \sigma^{(2)} := \frac{1}{m_2}\sum \limits_{s_1 \in C_2}\sum \limits_{s_2\in C_2} r_{s_1s_2}.
\end{equation*}
We make the following guess of dual variables, in particular, we adopt the assumption of $\bm{\Theta}$ in Lemma~\ref{lemma:guess_dual_unequal_new}.

\begin{lemma}
With a constant $\gamma \in [0, 1]$, the dual variables with the following forms satisfy the KKT conditions \eqref{eq:KKT_stationarity_unknown} - \eqref{eq:KKT_complementary_unknown} and $\mu_i \geq 0, \; i = 1,\ldots, n$ in \eqref{eq:KKT_dual_feasibility_unknown}.
\begin{equation*}
    \begin{aligned}
    \widetilde{\bm{\alpha}}_{ij} &= \frac{1}{m_1}\sum \limits_{s_1 \in C_1}\bm{A}_{s_1j} + \frac{1}{m_2}\sum \limits_{s_2 \in C_2}\bm{A}_{is_2} - \frac{1}{m_1m_2}\sum \limits_{s_1 \in C_1}\sum \limits_{s_2 \in C_2}\bm{A}_{s_1s_2}, \quad i \in C_1, \; j \in C_2,\\
    \bar{\mu}_1 &= 2\max_{i \in C_1}\bigg\{\frac{m_1 \gamma}{\sqrt{d}}\max_{j \in C_2}\|\widetilde{\bm{\alpha}}_{ij}\|_{\mathrm{F}} \!-\! \sigma_i\bigg\} \!-\! \sigma^{(2)},\;\;
    \bar{\mu}_2 = 2\max_{i \in C_2}\bigg\{\frac{m_2 (1\!-\!\gamma)}{\sqrt{d}}\max_{j \in C_2}\|\widetilde{\bm{\alpha}}_{ji}\|_{\mathrm{F}} \!-\! \sigma_i\bigg\} \!-\! \sigma^{(1)},\\
    \mu_i &=
    \begin{cases}
    \displaystyle \frac{\sqrt{d}}{m_1}\bigg(\sigma_i + \frac{\sigma^{(2)}}{2}+ \frac{\max\left\{ \bar{\mu}_1, \bar{\mu}_2\right\}}{2}\bigg), &\; i \in C_1,\\
    \displaystyle \frac{\sqrt{d}}{m_2}\bigg(\sigma_i + \frac{\sigma^{(1)}}{2}  + \frac{\max\left\{ \bar{\mu}_1, \bar{\mu}_2\right\}}{2}\bigg), &\; i \in C_2.
    \end{cases} \quad 
    \bm{\Theta}_{ij} = 
    \begin{cases}
    \mu_i\bm{I}_d/\sqrt{d}, &  \kappa(i)= \kappa(j), \\
    \gamma \widetilde{\bm{\alpha}}_{ij}, & i \in C_1, j \in C_2,\\
    (1-\gamma) \widetilde{\bm{\alpha}}_{ji}^\top, & i \in C_2, j \in C_1,
    \end{cases}\\ 
    \bm{Z} &= \left(\sigma^{(1)} + \sigma^{(2)} + \max\left\{ \bar{\mu}_1, \bar{\mu}_2\right\}\right)\bm{I}_d, \quad \quad  \bm{\Lambda} = -\bm{A} - \bm{I}_n \otimes \bm{Z} + \bm{\Theta} + \bm{\Theta}^{\top}.
    \end{aligned}
\end{equation*}
\label{lemma:guess_dual_unknown}
\end{lemma}
\begin{lemma}
Given the dual variables in Lemma~\ref{lemma:guess_dual_unknown}, $\bm{\Lambda}$ can be expressed as
\begin{equation*}
    \bm{\Lambda} = (\bm{I}_{nd} - \bar{\bm{M}}^*)(\mathbb{E}[\bm{A}] - \bm{A} + (p - \epsilon)\bm{I}_{nd})(\bm{I}_{nd} - \bar{\bm{M}}^*), \quad \epsilon := \sigma^{(1)} + \sigma^{(2)} + \max\left\{ \bar{\mu}_1, \bar{\mu}_2\right\}
\end{equation*}
Then, $\bm{\Lambda} \succeq 0$ and $\mathcal{N}(\bm{\Lambda}) = \mathcal{R}(\bar{\bm{M}}^*)$ are satisfied if $\widetilde{\bm{\Lambda}} \succ 0$.
\label{lemma:simplify_lambda_unknown}
\end{lemma}

\paragraph{Step 3: Find the condition for $\widetilde{\bm{\Lambda}} \succ 0$.} 
Again, by using the same argument as in Section~\ref{sec:sketch_proof_equal_main} we get, for exact recovery it suffices to ensure $\lambda_{\text{min}}((p-\epsilon) \bm{I}_{nd}) > \|\mathbb{E}[\bm{A}] - \bm{\bm{A}}\|$. To this end, we have the following bound for $\lambda_{\text{min}}((p-\epsilon) \bm{I}_{nd})$.

\begin{lemma}
Let $p = \alpha\log n/n, q = \beta \log n/n$ and $\delta := \left(\frac{1}{\sqrt{\rho}} + \frac{1}{\sqrt{1-\rho}}\right)\sqrt{\ell_d\beta}$, where $\ell_d = 2$ for $d > 2$ and $\ell_d = 1$ for $d = 2$. Then $\lambda_{\text{min}}((p-\epsilon) \bm{I}_{nd})$ satisfies 
\begin{equation*}
    \lambda_{\text{min}}((p-\epsilon) \bm{I}_{nd}) \geq \min\{\rho(2\tau_1 - \alpha - 2\gamma \delta), \; (1-\rho)(2\tau_2 - \alpha - 2(1-\gamma)\delta\} \log n + o(\log n)
\end{equation*}
with probability $1 - n^{-\Omega(1)}$ for $\tau_1, \tau_2 \in [0, \alpha)$ such that 
\begin{align}
    1-\rho\bigg(\alpha - \tau_1 \log\bigg(\frac{e\alpha}{\tau_1}\bigg)\bigg) < 0, \quad 
    1-(1- \rho)\bigg(\alpha - \tau_2 \log\bigg(\frac{e\alpha}{\tau_2}\bigg)\bigg) < 0.
    \label{eq:cond_tau_unknown} 
\end{align}
\label{lemma:bound_epsilon_1_2}
\end{lemma}

\begin{proof}[Proof of Theorem~\ref{the:2_cluster_unknown_d_2}]
From Lemma~\ref{lemma:bound_epsilon_1_2} we have $\lambda_{\text{min}}((p-\epsilon)\bm{I}_{nd}) = O(\log n)$ with high probability. Also from Lemma~\ref{lemma:concentration_A} we get $\|\mathbb{E}[\bm{\bm{A}}] - \bm{A}\| = o(\sqrt{\log n})$ with high probability. This implies as $n$ is large,  $\lambda_{\text{min}}(p\bm{I}_{nd} - \bm{Z}) > \|\mathbb{E}[\bm{\bm{A}}] - \bm{A}\|$ satisfies as long as $\lambda_{\text{min}}(p\bm{I}_{nd} - \bm{Z}) = \Omega(\log n)$, which is equivalent to 
\begin{equation}
    \min\{\rho(2\tau_1 - \alpha - 2\gamma \delta), \; (1-\rho)(2\tau_2 - \alpha - 2(1-\gamma)\delta\} > 0.
    \label{eq:bound_unknown_old}
\end{equation}
This further reduces to $\tau_1 > \alpha/2 + \gamma \delta$ and $\tau_2 > \alpha/2 + (1-\gamma)\delta$. It remains to check \eqref{eq:cond_tau_unknown} holds for some $\tau_1, \tau_2 \in [0, \alpha)$. To this end, by using the same argument as in the proof of Theorem~\ref{the:2_cluster_unequal_d_2}, one can see that \eqref{eq:cond_tau_unknown} holds as long as $\tau_1 < \tau_1^*,\; \tau_2 < \tau_2^*$ with $\tau_1^*, \tau_2^*$ defined in \eqref{eq:tau_1_2_star} (Notice that $\tau_1^*, \tau_2^*$ exist only if $\alpha > \max\left\{1/\rho, \; 1/(1-\rho)\right\}$). Therefore, putting all these together the condition for exact recovery becomes 
\begin{equation}
\alpha > \max\left\{\frac{1}{\rho}, \; \frac{1}{1-\rho}\right\}, \quad 
\frac{\alpha}{2} + \gamma \delta < \tau_1 < \tau_1^* \quad \text{and} \quad \frac{\alpha}{2} + (1-\gamma)\delta < \tau_2 < \tau_2^*.
\label{eq:range_tau_12_unknown}
\end{equation}
Furthermore, since $\gamma \in [0,1]$ is not specified, it suffices to ensure \eqref{eq:range_tau_12_unknown} holds for some $\gamma \in [0,1]$ such that the range of $\tau_1$ and $\tau_2$ is valid. This leads to the conditions in \eqref{eq:unknown_cond_theorem}.
\end{proof}

\subsubsection{Proof of the lemmas}
\label{sec:proof_claims_unknown}
\begin{proof}[Proof of Lemma~\ref{lemma:uniqueness_unknown}]
The optimality of $\bar{\bm{M}}^*$ is immediately obtained since \eqref{eq:SDP_unknown} is convex then KKT conditions are sufficient for $\bar{\bm{M}}^*$ being optimal~\cite{boyd2004convex}. For uniqueness, suppose $\widetilde{\bm{M}} \neq \bar{\bm{M}}^*$ is another optimal solution. Similar to \eqref{eq:relation_tilde_M_M} in Section~\ref{sec:proof_two_equal}, we have 
\begin{align}
    0 =
    \langle \bm{A}, \widetilde{\bm{M}} - \bar{\bm{M}}^* \rangle &= -\langle\bm{\Lambda}, \widetilde{\bm{M}} - \bar{\bm{M}}^*\rangle -  \langle\bm{I}_n \otimes \bm{Z}, \widetilde{\bm{M}} - \bar{\bm{M}}^*\rangle + \langle \bm{\Theta} + \bm{\Theta}^\top, \widetilde{\bm{M}} - \bar{\bm{M}}^*\rangle \nonumber\\
    &\overset{(a)}{=} -\langle\bm{\Lambda}, \widetilde{\bm{M}} - \bar{\bm{M}}^*\rangle  + \langle \bm{\Theta} + \bm{\Theta}^\top, \widetilde{\bm{M}} - \bar{\bm{M}}^*\rangle \nonumber\\
    &= -\langle\bm{\Lambda}, \widetilde{\bm{M}} \rangle  + 2\langle \bm{\Theta}, \widetilde{\bm{M}} - \bar{\bm{M}}^*\rangle
    \label{eq:relation_tilde_M_M_unknown}
\end{align}
where $(a)$ comes from $\sum_{i = 1}^n\bm{M}_{ii} = 2\bm{I}_d$ in \eqref{eq:SDP_unknown}. To proceed, let us rewrite 
\begin{equation}
    \begin{aligned}
    \langle \bm{\Theta}, \widetilde{\bm{M}} - \bar{\bm{M}}^*\rangle &= \langle \bm{\Theta}, \widetilde{\bm{M}}\rangle - \langle \bm{\Theta}, \bar{\bm{M}}^*\rangle = \sum \limits_{i,j} \langle\bm{\Theta}_{ij}, \widetilde{\bm{M}}_{ij} \rangle - \sum \limits_{i,j} \langle\bm{\Theta}_{ij}, \bar{\bm{M}}_{ij}^* \rangle.
    \end{aligned}
    \label{eq:theta_M_tilde_M_unknown}
\end{equation}
Furthermore, by plugging \eqref{eq:theta_definition_unknown} into \eqref{eq:theta_M_tilde_M_unknown} we obtain
\begin{equation}
    \begin{aligned}
    \sum \limits_{i,j} \langle\bm{\Theta}_{ij}, \bar{\bm{M}}_{ij}^* \rangle &= \sum \limits_{i}\sum \limits_{j: \kappa(j)=\kappa(i)}\frac{\mu_i}{\|\bar{\bm{M}}_{ij}^*\|_{\mathrm{F}}}\langle \bar{\bm{M}}_{ij}^*, \bar{\bm{M}}_{ij}^* \rangle = \sum \limits_{i} \mu_i\sqrt{d},\\
    \sum \limits_{i,j}\langle\bm{\Theta}_{ij}, \widetilde{\bm{M}}_{ij} \rangle &= \sum \limits_{i}\mu_i\bigg(\sum \limits_{j: \kappa(j)= \kappa(i)} \frac{1}{\|\bar{\bm{M}}_{ij}^*\|_{\mathrm{F}}}\langle \bar{\bm{M}}_{ij}^*, \widetilde{\bm{M}}_{ij}\rangle + \sum \limits_{j: \kappa(j) \neq \kappa(i)} \langle \bm{\alpha}_{ij}, \widetilde{\bm{M}}_{ij}\rangle \bigg) \\
    &\leq \sum \limits_{i}\mu_i \bigg(\sum \limits_{j: \kappa(j)=\kappa(i)} \frac{1}{\|\bar{\bm{M}}_{ij}^*\|_{\mathrm{F}}}\|\bar{\bm{M}}_{ij}^*\|_\mathrm{F}\|\widetilde{\bm{M}}_{ij}\|_\mathrm{F} + \sum \limits_{j: \kappa(j)\neq \kappa(i)} \|\bm{\alpha}_{ij}\|_\mathrm{F}\|\widetilde{\bm{M}}_{ij}\|_\mathrm{F} \bigg)\\
    &\leq \sum \limits_{i}\mu_i\sum \limits_{j}\|\widetilde{\bm{M}}_{ij}\|_\mathrm{F} \overset{(a)}{\leq} \sum \limits_{i}\mu_i\sqrt{d},
    \end{aligned}
    \label{eq:theta_M_M_tilde_bound_unknown}
\end{equation}
where $(a)$ holds since $\sum_{j}\|\widetilde{\bm{M}}_{ij}\|_\mathrm{F} \leq \sqrt{d}$ in \eqref{eq:SDP_unknown}. Then combining \eqref{eq:theta_M_tilde_M_unknown} and \eqref{eq:theta_M_M_tilde_bound_unknown} results in $\langle \bm{\Theta}, \widetilde{\bm{M}} - \bm{M}^*\rangle \leq 0$, and plugging this back into \eqref{eq:relation_tilde_M_M_unknown} gives $\langle \bm{\Lambda}, \widetilde{\bm{M}}\rangle \leq 0$. On the other hand, from $\bm{\Lambda} \succeq 0$ in \eqref{eq:KKT_dual_feasibility_unknown} and $\widetilde{\bm{M}} \succeq 0$ in \eqref{eq:SDP_unknown} we have $\langle\bm{\Lambda}, \widetilde{\bm{M}} \rangle \geq 0$. Therefore we conclude $\langle\bm{\Lambda}, \widetilde{\bm{M}} \rangle = 0$ and $\langle \bm{\Theta}, \widetilde{\bm{M}} - \bm{M}^*\rangle = 0$. To proceed, by $\mathcal{N}(\bm{\Lambda}) = \mathcal{R}(\bar{\bm{M}}^*)$ and the definition of $\bar{\bm{M}}^*$ in \eqref{eq:M_ground_truth_unknown}, $\widetilde{\bm{M}}$ has the following form:
\begin{equation}
    \widetilde{\bm{M}} = \bm{V}^{(1)}\bm{\Sigma}_1(\bm{V}^{(1)})^{\top} + \bm{V}^{(2)}\bm{\Sigma}_2(\bm{V}^{(2)})^{\top},
    \label{eq:another_M}
\end{equation}
for some diagonal $\bm{\Sigma}_1, \bm{\Sigma}_2  \succeq 0$. Now it remains to show  $\bm{\Sigma}_1 = \bm{I}_d/m_1, \; \bm{\Sigma}_2 = \bm{I}_d/m_2$. To this end, from $\langle \bm{\Theta}, \widetilde{\bm{M}} - \bm{M}^*\rangle = 0$ we see $(a)$ in \eqref{eq:theta_M_M_tilde_bound_unknown} holds with equality, i.e. $\sum_{j = 1}^n \|\widetilde{\bm{M}}_{ij}\|_{\mathrm{F}} = \sqrt{d}, \; \forall i$. Then combining this with \eqref{eq:another_M} yields the following for $i \in C_1$:
\begin{equation*}
    \sum \limits_{j = 1}^n \|\widetilde{\bm{M}}_{ij}\|_{\mathrm{F}} \overset{(a)}{=} \sum \limits_{j = 1}^n \|\bm{R}_i \bm{\Sigma}_1 \bm{R}_j^{\top}\|_{\mathrm{F}} \overset{(b)}{=} m_1\|\bm{\Sigma}_1\|_{\mathrm{F}} = \sqrt{d}, \; \forall i \in C_1 \quad \Rightarrow \quad \|\bm{\Sigma}_1\|_{\mathrm{F}} = \sqrt{d}/m_1,
\end{equation*}
where both $(a)$ and $(b)$ follow from \eqref{eq:another_M}.
Similarly for $i \in C_2$ we can obtain $\|\bm{\Sigma}_2\|_{\mathrm{F}} = \sqrt{d}/m_2$. Next, note for each diagonal block of $\widetilde{\bm{M}}$ it satisfies
\begin{equation*}
    \begin{aligned}
    \bigg\|\sum \limits_{i = 1}^n \widetilde{\bm{M}}_{ii}\bigg\|_{\mathrm{F}} &\overset{(a)}{=} \bigg\|\sum \limits_{ i \in C_1}\bm{R}_i \bm{\Sigma}_1 \bm{R}_i^{\top} \!+\! \sum \limits_{ i \in C_2}\bm{R}_i \bm{\Sigma}_2 \bm{R}_i^{\top}\bigg\|_{\mathrm{F}} \!\overset{(b)}{\leq}\! \sum \limits_{i \in C_1}\left\|\bm{R}_i \bm{\Sigma}_1 \bm{R}_i^{\top}\right\|_{\mathrm{F}} \!+\! \sum \limits_{i \in C_2}\left\|\bm{R}_i \bm{\Sigma}_2 \bm{R}_i^{\top}\right\|_{\mathrm{F}} \\
    &\overset{(c)}{=} 2\sqrt{d} \overset{(d)}{=} \bigg\|\sum \limits_{i = 1}^n \widetilde{\bm{M}}_{ii}\bigg\|_{\mathrm{F}},
    \end{aligned}
    \label{eq:widetilde_M_ii}
\end{equation*}
where $(a)$ follows from \eqref{eq:another_M}; $(b)$ comes from triangle inequality; $(c)$ holds because $\|\bm{\Sigma}_1\|_{\mathrm{F}} = \sqrt{d}/m_1$ and $\|\bm{\Sigma}_1\|_{\mathrm{F}} = \sqrt{d}/m_2$; $(d)$ follows from $\sum_{i = 1}^n\bm{M}_{ii} = 2\bm{I}_d$ in \eqref{eq:SDP_unknown}. This implies $(b)$ holds with equality, and further 
$\bm{R}_i \bm{\Sigma}_1 \bm{R}_i^{\top} = \bm{I}_d/m_1, \; \forall i \in C_1$ and $\bm{R}_i \bm{\Sigma}_2 \bm{R}_i^{\top} = \bm{I}_d/m_2, \; \forall i \in C_2$. It follows $\bm{\Sigma}_1 = \bm{I}_d/m_1$ and  $\bm{\Sigma}_2 = \bm{I}_d/m_2$, which indicates $\widetilde{\bm{M}} = \bar{\bm{M}}^*$.
\end{proof}

\begin{proof}[Proof of Lemma~\ref{lemma:guess_dual_unknown}]
First, from \eqref{eq:KKT_stationarity_unknown} we get $\bm{\Lambda} = -\bm{A} - \bm{I}_n \otimes \bm{Z} + \bm{\Theta} + \bm{\Theta}^{\top}$, plugging this into \eqref{eq:unknown_N_R} yields
\begin{alignat}{2}
    &\sum \limits_{s: C(s)=\kappa(i)}\left(\frac{\mu_i \bm{I}_d}{\sqrt{d}} + \frac{\mu_s \bm{I}_d}{\sqrt{d}} - \bm{A}_{is}\right) = \bm{Z}, \label{eq:Z_theta_unknown} \quad &&i = 1,\ldots, n,\\
    &\sum \limits_{s: C(s) \neq \kappa(i)}\left(\bm{\Theta}_{is} + \bm{\Theta}_{si}^{\top} - \bm{A}_{is}\right) = \bm{0},  \quad && i = 1,\ldots, n. \label{eq:theta_unknown}
\end{alignat}
Then, by summing \eqref{eq:Z_theta_unknown} over $i \in C_1$ and $i \in C_2$ separately we get the expression of $\bm{Z}$ as
\begin{equation}
    \begin{aligned}
    \bm{Z} = \bigg(\frac{2}{\sqrt{d}}\sum \limits_{s \in C_1} \mu_s - \sigma^{(1)}\bigg)\bm{I}_d = \bigg(\frac{2}{\sqrt{d}}\sum \limits_{s \in C_2} \mu_s - \sigma^{(2)}\bigg)\bm{I}_d,
    \end{aligned}
    \label{eq:Z_sum}
\end{equation}
Next, plugging \eqref{eq:Z_sum} back into \eqref{eq:Z_theta_unknown} yields expressions of $\mu_i$ as
\begin{equation}
    \begin{aligned}
    \mu_i &= \frac{\sqrt{d}}{m_1}\sigma_i + \underbrace{\bigg(\frac{1}{m_1}\sum \limits_{s \in C_1}\mu_s 
    - \frac{\sqrt{d}}{m_1}\sigma^{(1)} \bigg)}_{=: \bar{\mu}_1} = \frac{\sqrt{d}}{m_1}\sigma_i + \bar{\mu}_{1}, \quad  i \in C_1, \\
    \mu_i &= \frac{\sqrt{d}}{m_2}\sigma_i + \underbrace{\bigg(\frac{1}{m_2}\sum \limits_{s \in C_2}\mu_s - \frac{\sqrt{d}}{m_2}\sigma^{(2)} \bigg)}_{=: \bar{\mu}_2} = \frac{\sqrt{d}}{m_2}\sigma_i + \bar{\mu}_{2}, \quad i \in C_2.
    \end{aligned}
    \label{eq:mu_i_def}
\end{equation}
Then plugging \eqref{eq:mu_i_def} into \eqref{eq:Z_sum} yields
\begin{equation}
    \bm{Z} = \bigg(\sigma^{(1)} + \frac{2m_1}{\sqrt{d}}\bar{\mu}_{1}\bigg)\bm{I}_d =  \bigg(\sigma^{(2)} + \frac{2m_2}{\sqrt{d}}\bar{\mu}_{2}\bigg)\bm{I}_d.
    \label{eq:Z_definition_unknown}
\end{equation}
Importantly, to make the LHS and RHS above identical, $\bar{\mu}_{1}$ and $\bar{\mu}_{2}$ satisfy
\begin{equation}
    \bar{\mu}_{1} = \frac{\sqrt{d}}{2m_1}(\sigma^{(2)} + \bar{\mu}), \quad \bar{\mu}_{2} = \frac{\sqrt{d}}{2m_2}(\sigma^{(1)} + \bar{\mu})
    \label{eq:mu_C1_C2_definition_unknown}
\end{equation}
for some $\bar{\mu}$. Plugging \eqref{eq:mu_C1_C2_definition_unknown} into \eqref{eq:mu_i_def} and \eqref{eq:Z_sum} yields the expressions of $\bm{Z}$ and $\mu_i$ in terms of $\bar{\mu}$, and we restrict $\gamma \in [0,1]$ which ensures $\mu_i \geq 0$. To determine $\bar{\mu}$, we adopt the guess of $\bm{\Theta}_{ij}$ and $\widetilde{\bm{\alpha}}_{ij}$ in Lemma~\ref{lemma:guess_dual_unequal_new}. From \eqref{eq:theta_definition_unknown} we have $\|\bm{\Theta}_{ij}\|_{\mathrm{F}} \leq \mu_i$, then for each $\mu_i$ it should satisfy $\mu_i \geq \|\bm{\Theta}_{ij}\|_{\mathrm{F}}$ for any $j$ with $\kappa(j)\neq \kappa(i)$. By further plugging the definition of $\mu_i$ we get $\bar{\mu} \geq \max\left\{ \bar{\mu}_1, \bar{\mu}_2\right\}$ with $\bar{\mu}_1, \bar{\mu}_2$ defined in Lemma~\ref{lemma:guess_dual_unknown}, and we set $\bar{\mu} = \max\left\{ \bar{\mu}_{1}, \bar{\mu}_{2}\right\}$. This completes our guess of dual variables.
\end{proof}

\begin{proof}[Proof of Lemma~\ref{lemma:simplify_lambda_unknown}]
We follow the same route as in Lemma~\ref{lemma:simplify_lambda} by first rewriting $\bm{A}$, $\bm{\Lambda}$ and $\bm{\Theta}$ into four blocks as \eqref{eq:2_cluster_A_M}. Also, let $\widetilde{\bm{\mu}}_{1} = \text{diag}(\{\mu_{i}\bm{I}_d\}_{i = 1}^{m_1}) \in \mathbb{R}^{m_1d \times m_1d}$ and $\widetilde{\bm{\mu}}_{2} = \text{diag}(\{\mu_{i}\bm{I}_d\}_{i = m_1+1}^n) \in \mathbb{R}^{m_2d \times m_2d}$. Then the four blocks of $\bm{\Theta}$ satisfy
\begin{align*}
        &\widetilde{\bm{\Theta}}_{11} = \frac{\widetilde{\bm{\mu}}_1\widetilde{\bm{M}}_{1}^*}{\sqrt{d}}, \quad \quad \quad 
        \widetilde{\bm{\Theta}}_{22} = \frac{\widetilde{\bm{\mu}}_2\widetilde{\bm{M}}_{2}^*}{\sqrt{d}},\\
        &\widetilde{\bm{\Theta}}_{12} + \widetilde{\bm{\Theta}}_{21}^\top = \frac{1}{m_1}\widetilde{\bm{M}}_1^*\widetilde{\bm{A}}_{21} + \frac{1}{m_2}\widetilde{\bm{A}}_{12}\widetilde{\bm{M}}_{2}^* - \frac{1}{m_1m_2}\widetilde{\bm{M}}_{1}^*\widetilde{\bm{A}}_{12}\widetilde{\bm{M}}_2^*.
\end{align*}
The remaining proof is very similar to the one of Lemma~\ref{lemma:simplify_lambda}, and we do not repeat.
\end{proof}

\begin{proof}
According to Lemma~\ref{lemma:simplify_lambda_unknown}, $\lambda_{\text{min}}((p-\epsilon)\bm{I}_{nd})$ satisfies
\begin{equation*}
    \begin{aligned}
    &\lambda_{\text{min}}((p-\epsilon) \bm{I}_{nd}) = p - \epsilon = p - \max\{\bar{\mu}_1, \; \bar{\mu}_2\} -  \sigma^{(1)} - \sigma^{(2)} \\
        &= p \!+\! \min\bigg\{\underbrace{2\min_{i \in C_1}\bigg\{\sigma_i \!-\! \frac{m_1 \gamma}{\sqrt{d}}\max_{j \in C_2}\|\widetilde{\bm{\alpha}}_{ij}\|_{\mathrm{F}}\bigg\} \!-\! \sigma^{(1)}}_{=: \epsilon_1}, \; \underbrace{2\min_{i \in C_2}\bigg\{\sigma_i \!-\! \frac{m_2 (1\!-\!\gamma)}{\sqrt{d}}\max_{j \in C_1}\|\widetilde{\bm{\alpha}}_{ji}\|_{\mathrm{F}}\bigg\} \!-\! \sigma^{(2)}}_{=:\epsilon_2}\bigg\}\\
        &= p + \min\{\epsilon_1, \epsilon_2\}.
    \end{aligned}
    \label{eq:epsilon_1_2_unknown}
\end{equation*}
Then we bound $\epsilon_1$ and $\epsilon_2$ separately as 
\begin{equation*}
    \begin{aligned}
    \epsilon_1 &\geq 2\min_{i \in C_1}\sigma_i - \frac{2m_1\gamma}{\sqrt{d}}\max_{i \in C_1, j \in C_2}\|\widetilde{\bm{\alpha}}_{ij}\|_\mathrm{F} - \sigma^{(1)}, \\
    \epsilon_2 &\geq 2\min_{i \in C_2}\sigma_i - \frac{2m_2(1-\gamma)}{\sqrt{d}}\max_{i \in C_2, j \in C_1}\|\widetilde{\bm{\alpha}}_{ji}\|_\mathrm{F} - \sigma^{(2)}.
    \end{aligned}
\end{equation*}
For $\min_{i \in C_1}\sigma_i$ and $\min_{i \in C_2}\sigma_i$, we can use the bounds derived in \eqref{eq:bound_x_i_C_1} and \eqref{eq:bound_x_i_C_2} respectively. For $\max_{i \in C_1, j \in C_2}\|\widetilde{\bm{\alpha}}_{ij}\|_\mathrm{F}$, by using the result in \eqref{eq:bound_nu} we obtain
\begin{equation*}
    \max_{i \in C_1, j \in C_2}\|\widetilde{\bm{\alpha}}_{ij}\|_\mathrm{F} \leq \sqrt{2(1+c)\beta}\bigg(\frac{1}{\sqrt{\rho}} + \frac{1}{\sqrt{1-\rho}}\bigg)\bigg(\frac{\log n}{n}\bigg) + o\bigg(\frac{\log n}{n}\bigg)
\end{equation*}
with probability $1 - n^{-c}$. For $\sigma^{(1)}$, since $r_{ij} = r_{ji}$ for $i, j \in C_1$, it satisfies 
\begin{equation*}
    \sum \limits_{s_1 \in C_1} \sum \limits_{s_2 \in C_1} r_{s_1s_2} = 2\sum \limits_{s_1, \;  s_2 \in C_1,\; s_1 < s_2} r_{s_1s_2},
\end{equation*}
and $\sum \limits_{s_1, s_2 \in C_1,\; s_1 < s_2} r_{s_1s_2}$ follows a binomial distribution $\mathrm{Binom}\left(m_1(m_1-1)/2, \alpha \log n/n\right)$. By applying Hoeffding's inequality~\cite{boucheron2013concentration} we obtain
\begin{equation*}
    \begin{aligned}
    &\mathbb{P}\Bigg\{\sum \limits_{s_1, s_2 \in C_1,\; s_1 < s_2} r_{s_1s_2} - \frac{ (m_1-1)\alpha\rho\log n}{2} \geq t \Bigg\} \leq \exp\left(-\frac{4t^2}{m_1^2}\right), 
    \end{aligned}
\end{equation*}
for any $t > 0$. This further indicates $\sigma^{(1)} \leq \alpha\rho \log n + o(\log n)$ with high probability. Similarly we can obtain  $\sigma^{(2)} \leq \alpha(1-\rho)\log n + o(\log n)$ with high probability. By putting all the results together and let $c \rightarrow 0$ to be close to zero, we get the bounds for $\epsilon_1, \epsilon_2$ and further $\lambda_{\text{min}}((p-\epsilon)\bm{I}_{nd})$. 
\end{proof}

\bibliographystyle{plain}
\bibliography{refs}

\end{document}